\documentclass{article}

\usepackage[preprint]{neurips_2025}

\usepackage[utf8]{inputenc} 
\usepackage[T1]{fontenc}    
\usepackage{hyperref}       
\usepackage{url}            
\usepackage{booktabs}       
\usepackage{amsfonts}       
\usepackage{nicefrac}       
\usepackage{microtype}      
\usepackage{xcolor}         
\usepackage{graphicx}
\usepackage{subfigure}

\usepackage{algorithm}
\usepackage{algorithmic}
\usepackage{amsmath,amssymb,amsfonts,amsthm,mathtools,bm}
\usepackage{enumitem}
\usepackage{multirow}
\usepackage{array}
\usepackage{color}
\usepackage{xcolor}
\usepackage{wrapfig}
\usepackage{multicol}
\usepackage{caption}
\usepackage{diagbox}
\usepackage{pifont}
\usepackage{pdfpages}
\usepackage{kotex}
\usepackage{excludeonly}

\usepackage[capitalize,noabbrev]{cleveref}

\usepackage{listings}
\usepackage{xcolor}

\theoremstyle{plain}
\newtheorem{theorem}{Theorem}[section]
\newtheorem{proposition}[theorem]{Proposition}
\newtheorem{lemma}[theorem]{Lemma}

\theoremstyle{definition}
\newtheorem{definition}[theorem]{Definition}
\newtheorem{assumption}[theorem]{Assumption}
\theoremstyle{remark}

\usepackage[toc,page,header]{appendix}
\usepackage{minitoc}
\usepackage{silence}
\DeclareMathOperator*{\diag}{diag}

\definecolor{customred}{RGB}{192, 0, 0}

\newcommand{\loss}{\mathcal{L}}

\newcommand{\grad}{\widehat{\mathbf{g}}}

\newcommand{\batch}{\mathcal{B}}

\newcommand\matnorm[2]{|\!|\!|#1|\!|\!|_{#2}}

\makeatletter
\newcommand*\rel@kern[1]{\kern#1\dimexpr\macc@kerna}
\newcommand*\widebar[1]{%
  \begingroup
  \def\mathaccent##1##2{%
    \rel@kern{0.8}%
    \overline{\rel@kern{-0.8}\macc@nucleus\rel@kern{0.2}}%
    \rel@kern{-0.2}%
  }%
  \macc@depth\@ne
  \let\math@bgroup\@empty \let\math@egroup\macc@set@skewchar
  \mathsurround\z@ \frozen@everymath{\mathgroup\macc@group\relax}%
  \macc@set@skewchar\relax
  \let\mathaccentV\macc@nested@a
  \macc@nested@a\relax111{#1}%
  \endgroup
}
\makeatother

\definecolor{codegreen}{rgb}{0,0.6,0}
\definecolor{codegray}{rgb}{0.5,0.5,0.5}
\definecolor{codepurple}{rgb}{0.58,0,0.82}
\definecolor{backcolour}{rgb}{0.95,0.95,0.92}

\lstdefinestyle{mystyle}{
    backgroundcolor=\color{backcolour},   
    commentstyle=\color{codegreen},
    keywordstyle=\color{magenta},
    numberstyle=\tiny\color{codegray},
    stringstyle=\color{codepurple},
    basicstyle=\ttfamily\footnotesize,
    breakatwhitespace=false,         
    breaklines=true,                 
    captionpos=b,                    
    keepspaces=true,                 
    numbers=left,                    
    numbersep=5pt,                  
    showspaces=false,                
    showstringspaces=false,
    showtabs=false,                  
    tabsize=2,
    xleftmargin=0.05\textwidth,
    xrightmargin=0.05\textwidth
}

\lstnewenvironment{python}[1][]
{
    
    \lstset{style=mystyle,#1}
}{}

\newcommand{\Tr}{\operatorname{Tr}}

\title{Elucidating Subspace Perturbation in Zeroth-Order Optimization: Theory and Practice at Scale}

\author{
  Sihwan Park\textsuperscript{1$*$} \quad Jihun Yun\textsuperscript{1}\thanks{Equal Contribution} \quad
  \bf{Sungyub Kim}\textsuperscript{1} \quad
  \bf{Souvik Kundu}\textsuperscript{2} \quad  
   \bf{Eunho Yang}\textsuperscript{1,3}\thanks{Corresponding Author} \\
\, 
\textsuperscript{1}KAIST  \quad
\textsuperscript{2}Intel Labs  \quad \textsuperscript{3}AITRICS \\
\, 
\texttt{\{sihwan.park, arcprime, eunhoy\}@kaist.ac.kr}\\
\texttt{  sungyub.kim@mli.kaist.ac.kr, souvikk.kundu@intel.com}
}

\begin{document}
\doparttoc 
\faketableofcontents 

\maketitle

\begin{abstract}
    Zeroth-order (ZO) optimization has emerged as a promising alternative to gradient-based backpropagation methods, particularly for black-box optimization and large language model (LLM) fine-tuning. However, ZO methods often suffer from slow convergence due to high-variance stochastic gradient estimators. While subspace perturbations, such as sparsity and low-rank constraints, have been explored to mitigate this issue, their effectiveness remains poorly understood. In this work, we develop a \emph{unified theoretical framework} that analyzes both the convergence and generalization properties of ZO optimization under subspace perturbations. We show that high dimensionality is the primary bottleneck and introduce the notion of \textit{subspace alignment} to explain how the subspace perturbations reduce gradient noise and accelerate convergence. Our analysis further shows that a broad class of subspace perturbations exhibits a similar convergence rate, motivating us to prioritize practical considerations in real-world algorithm design. Building on these insights, we propose an efficient ZO method using block coordinate descent (MeZO-BCD), which perturbs and updates only a subset of parameters at each step. Extensive experiments show that MeZO-BCD significantly accelerates optimization, achieving up to $\mathbf{\times2.77}$ speedup in wall-clock time over MeZO on OPT-13B, while maintaining comparable iteration complexity and fine-tuning performance.
\end{abstract}

\section{Introduction}\label{sec:intro}

Zeroth-order (ZO) optimization~\citep{shamir2013complexity}, sometimes referred to as gradient-free or derivative-free optimization, has emerged as a compelling alternative to traditional gradient-based methods in deep learning. Unlike first- or higher-order approaches that rely on explicit gradient or Hessian information, ZO methods update model parameters solely based on the function evaluations. This property makes them particularly useful in settings where gradients are inaccessible or expensive to compute, such as black-box optimization~\citep{cai2021zeroth,zhang2022how}, adversarial attacks~\citep{chen2017zoo,kurakin2017adversarial,papernot2017practical}, hyperparameter tuning~\citep{li2021zeroth}, neural architecture search~\citep{wang2022zarts}, and efficient machine learning~\citep{zhang2024foresight}.

Beyond these traditional applications, recently ZO optimization has gained attention for fine-tuning large language models (LLMs)~\citep{malladi2023fine,guo2024zeroth,gautam2024variance,liu2024sparse, zo-bench, chen2024enhancing, yu2024subzero, park2024mezoadam}, primarily due to its significantly lower memory requirements compared to gradient-based optimizations. A pioneering study, MeZO~\citep{malladi2023fine}, demonstrated that zeroth-order stochastic gradient descent (SGD)~\citep{rob1951sgd} can fine-tune language models using memory and computational resources comparable to inference, making it a viable alternative for resource-constrained settings.

However, despite its practicality, ZO optimization suffers from notoriously slow convergence and poor generalization, particularly in large-scale fine-tuning scenarios. Unlike first-order methods that benefit from exact gradient updates, ZO approaches rely on stochastic gradient estimates obtained via random perturbations, leading to inherently high variance and inefficient updates. While recent studies have attempted to address this issue by introducing subspace perturbations, such as sparse perturbations~\citep{liu2024sparse,guo2024zeroth} and low-rank perturbations~\citep{chen2024enhancing,yu2024subzero}, the underlying principles that govern their effectiveness remain under-explored. Specifically, no unified framework exists to explain why these structures improve ZO efficiency or which structural properties are critical. Furthermore, the generalization properties of ZO optimization also remain elusive, as existing theoretical studies primarily focus on convergence while neglecting its impact on model generalization. In light of these observations, it is natural to pose the following question:
\begin{center}
    \emph{``What makes subspace perturbation effective, and how can this guide better algorithms?''}
\end{center}

In this paper, we answer this question by providing a \emph{unified framework} to analyze both the convergence and generalization properties of ZO optimization. This analysis shows that diverse perturbation strategies exhibit similar convergence behavior despite their structural differences, motivating algorithm design driven by practical efficiency rather than subspace structure. Building on this insight, we develop a practical ZO method for large-scale optimization. Our main contributions can be summarized as follows:

\begin{itemize}[leftmargin=4mm]
    \item \textbf{Unified theoretical framework for subspace perturbation.} We introduce a unified analysis of convergence and generalization in ZO optimization under subspace perturbations, based on the notions of \emph{subspace alignment} and \emph{dimensionality}. Importantly, our analysis remains valid under finite perturbation magnitude, avoiding the restrictive infinitesimal smoothing assumption commonly adopted in prior works, hence enabling more realistic modeling.
    
    \item \textbf{Theoretical criteria for effective subspace perturbation.} We show that sparse, low-rank, and other common structures share the same expected subspace alignment, implying equivalent average-case convergence. This enables the design of ZO algorithms that prioritize practical considerations, such as computational and memory efficiency, when choosing subspace perturbation.
    
    \item \textbf{MeZO-BCD: a practical and scalable ZO algorithm.} Based on our theory, we propose MeZO-BCD, a block coordinate descent method for ZO that perturbs and updates only one block per step. This substantially reduces memory transfer overhead and achieves up to $\mathbf{\times2.77}$ wall-clock speedup over MeZO on OPT-13B, while maintaining convergence and fine-tuning performance. The approach also enables promising future extensions such as adaptive block selection and integration with stateful optimizers.
\end{itemize}

\section{Preliminaries}

In this section, we summarize the notations and briefly introduce the fundamental zeroth-order optimization algorithm.

\textbf{Notations.} In this paper, we denote vectors by boldface lowercase letters (e.g. $\mathbf{x}, \mathbf{y}$) and matrices by boldface uppercase letters (e.g. $\mathbf{A}, \mathbf{B}$). We let $\mathcal{L}(\bm\theta)$ denote the loss function with respect to $\bm\theta = (\theta_1, \cdots, \theta_d) \in \mathbb{R}^d$. Further, for any $\mu > 0$, we define $\loss_\mu(\bm\theta) = \mathbb{E}_{\mathbf{u}}\big[\loss(\bm\theta + \mu \mathbf{u})\big]$, where $\mathbf{u} \sim \mathcal{N}(\bm 0, \mathbf{I}_d)$. For a matrix $\mathbf{A}$, we let $\sigma_i(\mathbf{A})$/$\lambda_i(\mathbf{A})$ denote $i$-th singular value/eigenvalue respectively, 
and $\sigma_{\text{max}}(\mathbf{A})$/$\lambda_{\text{max}}(\mathbf{A})$ represent the maximum singular value/eigenvalue of $\mathbf{A}$. Also, we use $\matnorm{\mathbf{A}}{p}$ to denote the matrix $p$-norm. Finally, we define two quantities, stable rank and intrinsic dimension of matrix: $\mathrm{srank}(\mathbf{A}) \coloneqq \frac{\sum_i \sigma_i(\mathbf{A})^2}{\sigma_{\text{max}}(\mathbf{A})^2}$ and $\mathrm{intdim}(\mathbf{A}) \coloneqq \frac{\mathrm{Tr}(\mathbf{A})}{\matnorm{\mathbf{A}}{\mathrm{op}}}$~\citep{ipsen2024stable}. 

\textbf{Simultaneous Perturbation Stochastic Approximation (SPSA).} We consider the following zeroth-order gradient estimator, known to be SPSA, as
\begin{align}\label{eqn:spsa}
    \mathbf{g}_t & = \frac{\loss(\bm\theta_t + \mu \mathbf{u}_t; \mathcal{B}_t) - \loss(\bm\theta_t - \mu \mathbf{u}_t; \mathcal{B}_t)}{2\mu} \mathbf{u}_t
\end{align}
where $\mu$ is the smoothing parameter, $\mathbf{u}_t \sim \mathcal{N}(\bm 0, \mathbf{I}_d)$, and $\mathcal{B}_t$ denotes minibatch selected at time $t$. Importantly, the SPSA estimator is an unbiased estimator of $\nabla \loss_\mu(\bm\theta)$.

\textbf{Zeroth-Order SGD (ZO-SGD).} The zeroth-order SGD updates the model parameter $\bm \theta_t$ via ZO gradient estimate as $\bm \theta_{t+1} = \bm \theta_t - \alpha_t \mathbf{g}_t$, where $\alpha_t$ is the step-size at time $t$.

\section{Theoretical Foundations on Subspace Perturbation}\label{sec:theory_intuition}

\subsection{Subspace ZO-SGD and Standard Assumptions}
We begin by formalizing the setup of zeroth-order optimization with structured perturbations. In particular, we describe the subspace ZO-SGD algorithm and introduce standard assumptions that underlie our convergence and generalization analysis.
\begin{enumerate}[itemsep=0em,leftmargin=0.3cm, itemindent=0.65cm,label=\textbf{(C-$\bf{\arabic*}$)},ref=\text{(C-$\arabic*$)},start=1]
    \item The loss function $\loss_i = \loss(\cdot; \mathbf{z}_i)$ for all $i \in [n]$ is $L$-smooth, i.e.,
    \begin{align*}
        \|\nabla\loss_i(\bm \theta) - \nabla\loss_i(\bm \theta')\| \le L\|\bm \theta - \bm \theta'\|
    \end{align*}
    for all $\bm \theta$ and $\bm \theta' \in \mathbb{R}^d$. \label{con:conv_smooth} 
    \item The stochastic gradient is unbiased and has a bounded variance. Further, we assume that the true gradient is bounded, i.e., $\mathbb{E}\left[\nabla\loss_i(\bm \theta)\right] = \nabla\loss(\bm \theta)$, $\mathbb{E}\big[\|\nabla\loss(\bm \theta) - \nabla\loss_i(\bm \theta)\|^2\big] \le \sigma^2$, and $\|\nabla\loss(\bm \theta)\| \le G$ for all $i \in [n]$. \label{con:conv_gradient}
\end{enumerate}
The condition \ref{con:conv_smooth} is standard in non-convex optimization analysis. In condition \ref{con:conv_gradient}, the unbiasedness of stochastic gradients and bounded variance are fundamental in stochastic optimization literature~\citep{ghadimi2013stochastic}. The bounded gradient condition in \ref{con:conv_gradient} is frequently used in convergence analysis in the context of adaptive gradient methods~\citep{KingBa15,j.2018on,chen2018on,yun2022adablock,ahn2024adam}.

Under these assumptions, we consider a ZO-SGD with subspace perturbation:
\begin{align}\label{eqn:struct-perturb}
    \widehat{\mathbf{g}}_t & = \frac{\mathcal{L}(\bm\theta_t + \mu \mathbf{M}_t \mathbf{u}_t; \mathcal{B}_t) - \mathcal{L}(\bm\theta_t - \mu \mathbf{M}_t \mathbf{u}_t; \mathcal{B}_t)}{2\mu} \mathbf{M}_t \mathbf{u}_t,
\end{align}
where $\mu$ is called a smoothing parameter. For simplicity, we assume that $\mathbf{M}_t$ is positive semi-definite and its maximum singular value is $\sigma_{\text{max}}(\mathbf{M}_t) = 1$ since it can be controlled by $\mu$. Also, we refer to the zeroth-order SGD with subspace perturbation \eqref{eqn:struct-perturb} as \textbf{\emph{subspace ZO-SGD}}, i.e., the model parameter is updated via $\bm\theta_{t+1} = \bm\theta_t - \alpha_t \widehat{\mathbf{g}}_t$. 

\subsection{Convergence Analysis of Subspace ZO-SGD}\label{subsec:zo_noise}

From this section, we provide the theoretical understanding of subspace ZO-SGD. Toward this, we start from convergence analysis. In recent studies, there have been many attempts to leverage the zeroth-order optimization in fine-tuning large language models due to its memory-efficiency~\citep{malladi2023fine}. To address \emph{both from-scratch training and fine-tuning scenarios}, we propose the following assumption.

\begin{assumption}[Local Intrinsic Dimension]\label{assumption:intdim}
    Under \ref{con:conv_smooth} and \ref{con:conv_gradient}, there exists a matrix $\mathbf{H}(\bm\theta_t) \preceq L\!\cdot\! \mathbf{I}_d$ such that:
    \begin{enumerate}
        \item For all $\bm \theta$ such that $\|\bm\theta - \bm\theta_t\| \le 2\alpha sG + 2\mu \sqrt{d}$, we have $\nabla^2 \loss(\bm\theta) \preceq \mathbf{H}(\bm\theta_t)$.
        \item $\mathrm{intdim}\big(\mathbf{H}(\bm\theta_t)\big) \le r$ for some constant $r \le d$.
    \end{enumerate}
\end{assumption}

\textbf{Remark on Assumption \ref{assumption:intdim}.}~ This assumption extends the one originally introduced in MeZO~\citep{malladi2023fine}, which assumes a global $\mathcal{O}(d)$-neighborhood around $\bm\theta_t$ and takes the limit $\mu \to 0$. In contrast, our formulation allows for a strictly positive $\mu$ and only requires a local $\mathcal{O}(s)$-ball since $\mu \sqrt{d}$ is negligible under $\mu = \mathcal{O}(1/\sqrt{dT})$ (refer to Theorem \ref{thm:improved_convergence_zosgd}). This makes our assumption significantly more practical for large-scale models where $d \gg s$, and substantially relaxes the regularity conditions on the loss landscape. When $\mathbf{H}(\bm\theta_t) = \mathbf{I}_d$ (i.e., $r = d$), our assumption reduces to the standard $L$-smoothness condition~\ref{con:conv_smooth}. Appendix~\ref{app:assumption} provides further discussion.

By introducing Assumption~\ref{assumption:intdim}, the subspace perturbation matrix $\mathbf{M}_t$ and the local Hessian upper bound inevitably interact with each other. To quantify this interaction, we introduce the following key quantity for our analysis.
\begin{definition}[Subspace Alignment]\label{def:subspace_alignment}
    For the matrix $\mathbf{M}_t$ in \eqref{eqn:struct-perturb} and $\mathbf{H}(\bm\theta_t)$ in Assumption \ref{assumption:intdim}, we define the \emph{subspace alignment} at time $t$ as 
    \begin{align}\label{eqn:subspace_alignment}
        \rho_t = \frac{\mathrm{Tr}(\mathbf{M}_t^\top \mathbf{H}(\bm\theta_t)\mathbf{M}_t)}{\lambda_{\text{max}}(\mathbf{H}(\bm\theta_t))}
    \end{align}
    Further, we define the \emph{mean subspace alignment} and \emph{maximum subspace alignment} as $\widebar{\rho} = (1/T)\sum_{t=1}^{T}\rho_t$ and $\rho_{\text{max}} = \max_{t\in[T]} \rho_t$ respectively.
\end{definition}

\textbf{Remark on Subspace Alignment.}~ Intuitively, $\mathrm{Tr}(\mathbf{M}_t^\top \mathbf{H}(\bm\theta_t) \mathbf{M}_t)$ captures ``how much the primary directions'' of $\mathbf{M}_t$ and $\mathbf{H}(\bm\theta_t)$ align each other, since both matrices are assumed to be PSD and thus act as ``weights'' on overlapping subspaces. Note that $0\le \rho_t \le \min\{r,s\}$ always holds.

\begin{theorem}[Dimension-Free Rate of Subspace ZO-SGD]\label{thm:improved_convergence_zosgd}
    Suppose that \ref{con:conv_smooth}, \ref{con:conv_gradient}, and Assumption \ref{assumption:intdim} hold with $\mathrm{srank}(\mathbf{M}_t) \le s$. Under the following parameter settings $\mu = \mathcal{O}\Big(\frac{1}{L\sqrt{\widebar{\rho} dT}}\Big)$ and $\alpha \le \frac{1}{36L\rho_{\textnormal{max}} + 336Ls\exp(-s/2)}$
    with $\widebar{\rho}$ and $\rho_{\textnormal{max}}$ defined in Def. \ref{def:subspace_alignment}, the subspace ZO-SGD satisfies
    \begin{align*}
        \mathbb{E}\big[\|\nabla\loss(\bm\theta_t)\|^2\big] \le \mathcal{O}\Big(\frac{r^2}{\widebar{\rho} T}+\frac{s^2}{dT} + \underbrace{\frac{\Delta}{\alpha T} + \frac{\sigma^2}{B}}_{\text{Standard SGD Rate}}\Big)
    \end{align*}
    where
    $\Delta \coloneqq \loss_\mu(\bm\theta_1) - \loss_\mu(\bm\theta_T)$. 
\end{theorem}

\textbf{Novelty.}~ Unlike previous approaches~\citep{rando2023optimal,chen2024enhancing} that require $\mathbf{M}_t^\top\mathbf{M}_t = \mathbf{I}_d$ to ensure the unbiasedness of zeroth-order gradient estimator, our Theorem~\ref{thm:improved_convergence_zosgd} only assumes that $\mathbf{M}_t$ is \emph{positive semidefinite}, which introduces a slight bias while still guaranteeing a valid descent direction. Hence, our analysis can encompass a much broader range of subspace perturbation matrix $\mathbf{M}_t$. Additionally, we are the first to highlight the importance of subspace perturbation by introducing the concept of subspace alignment for all $\mu > 0$, which sets our analysis totally apart from~\citep{malladi2023fine}, which focuses only on $\mu \rightarrow 0$. 

\textbf{Observation.}~ Theorem~\ref{thm:improved_convergence_zosgd} establishes a convergence of subspace ZO-SGD governed by a trade-off between the parameter $s$ and the subspace alignment $\widebar{\rho}$. As $s$ increases, $\widebar{\rho}$ typically grows, which improves $\frac{r^2}{\widebar{\rho}T}$ term yet make the term $\frac{s^2}{dT}$ worse. Consequently, small subspaces (e.g., $s = \mathcal{O}(\sqrt{rd})$ are limited by suboptimal $\widebar{\rho}$, whereas larger subspaces (e.g., $s = \omega(r\sqrt{d})$) forfeit low-dimensional advantages and may revert to the standard $\mathcal{O}(\frac{d}{T})$ rate. Hence, a judicious choice of the intermediate parameter $s$ is crucial for maintaining a sufficiently large $\widebar{\rho}$ while controlling the dimensional overhead $\mathcal{O}(\frac{s^2}{dT})$. In the ideal regime where $\widebar{\rho} = \mathcal{O}(r)$, our analysis enjoys a convergence rate of $\mathcal{O}(\frac{r}{T})$, offering a marked improvement over standard ZO-SGD. 

The convergence rate in Theorem~\ref{thm:improved_convergence_zosgd} depends critically on the subspace alignment $\rho_t$, which is governed by the structure of the perturbation matrix $\mathbf{M}$. This raises a central question:
\begin{center}
    \emph{``What constitutes an effective choice of $\mathbf{M}$ for improving convergence via subspace alignment?''}
\end{center}
The following proposition provides an intriguing preliminary insight.

\begin{proposition}[Expected Subspace Alignment $\rho$]\label{prop:expected_rho}
    Let $\mathbf{H}\in\mathbb{R}^{d\times d}$ be a fixed symmetric positive semi-definite matrix with $\lambda_{\max}(\mathbf{H})>0$.
    Let $\mathbf{M}\in\mathbb{R}^{d\times d}$ be a random matrix such that (i) $\mathbf{M}$ is (almost surely) an orthogonal projection matrix (i.e., $\mathbf{M}^\top=\mathbf{M}, \mathbf{M}^2=\mathbf{M}$ with probability $1$), (ii) $\mathbb{E}[\mathbf{M}] = \frac{s}{d}\mathbf{I}_d$ holds for given $s>0$.
    Then, the expectation of $\rho$ over distribution of $\mathbf{M}$ is:
    \begin{equation}
        \mathbb{E}[\rho]=\frac{s\operatorname{Tr}(\mathbf{H})}{d\lambda_{\max}(\mathbf{H})}.
    \end{equation}
\end{proposition}

\textbf{Remark.} Proposition~\ref{prop:expected_rho} shows that for any random matrix $\mathbf{M}$ satisfying these broad, non-specialized conditions, $\mathbb{E}[\rho]$ remains unchanged. Such conditions accommodate a wide variety of random matrix constructions--ranging from low-rank projections and sparse perturbations to more intricate ensembles--implying that most standard perturbation schemes share the same expected $\rho$ and, thus, exhibit comparable average convergence. This naturally raises a follow-up question: \emph{are all such choices of $\mathbf{M}$ effectively interchangeable, or do notable differences arise in practice?}

To advance our investigation, we first present three canonical instantiations of $\mathbf{M}$ that fulfill the criteria specified in Proposition~\ref{prop:expected_rho}.

\begin{definition}[Three Instantiations of $\mathbf{M}$]\label{def:instantiation}
    Let $s\in\mathbb{R}$ such that $0<s\leq d$. We define three instantiations of random matrix $\mathbf{M}\in\mathbb{R}^{d\times d}$ satisfying the conditions presented in Proposition~\ref{prop:expected_rho}:
    \begin{enumerate}[itemsep=0em,leftmargin=0.7cm]
        \item \textbf{Low-rank Projection}: $\mathbf{M}_\mathsf{LR} = \mathbf{U}\mathbf{U}^\top$, where $\mathbf{U} \in \mathbb{R}^{d\times s}$ is a random matrix with $s$ orthonormal columns. $\mathbf{U}$ is assumed to span an $s$-dimensional subspace chosen uniformly at random from the Grassmannian manifold $\mathbf{Gr}(s,\mathbb{R}^{d})$.
        \item \textbf{Sparse Perturbation}: $\mathbf{M}_\mathsf{SP} = \diag(\mathbf{m})$, where $\mathbf{m} = (m_1,\ldots,m_d)$ and each $m_i \sim \mathrm{Bernoulli}(p)$ with $p=\frac{s}{d}$.
        \item \textbf{Block Sparse Perturbation}: $\mathbf{M}_\mathsf{BSP} = \diag(\mathbf{m})$, but partition the parameters into $N$ blocks $\{B_1,\ldots, B_N\}$ such that $|B_i|=s$ for all $i=1,\ldots,N$. For a mask $\mathbf{m}=(m_1,\ldots, m_d)$, we set $m_i=1$ if $i\in B_j$ where $j\sim\mathcal{U}\{1,N\}$.
    \end{enumerate}
\end{definition}

Although these instantiations share the same expected $\rho$, their distributional characteristics can differ significantly, as demonstrated by the following result.

\begin{proposition}[Upper Tail Probabilities of $\rho$]\label{prop:prob_rho}
Let $\mathbf{H}\in\mathbb{R}^{d\times d}$ be as in Proposition~\ref{prop:expected_rho} and $\mathrm{intdim}(\mathbf{H})\leq r$ for some fixed $r\leq d$. Let $\widehat\rho$ be any threshold such that $\widehat{\rho}>\mathbb{E}[\rho]=\frac{s\mathrm{Tr}(\mathbf{H})}{d\lambda_{\max}(\mathbf{H})}$ and $\Delta=\widehat{\rho}-\mathbb{E}[\rho]>0$. Then, we have the following upper tail probabilities:
\begin{enumerate}[itemsep=0em,leftmargin=0.7cm]
    \item \textbf{Low-Rank Projection}: $\mathbb{P}(\rho\geq\widehat{\rho})\leq 2\exp\big(-c_1\frac{d\Delta^2}{s}\big)$.
    \item \textbf{Sparse Perturbation}: $\mathbb{P}(\rho\geq\widehat{\rho})\leq c_2\frac{rs}{d\Delta^2}$.
    \item \textbf{Block Sparse Perturbation}: $\mathbb{P}(\rho\geq\widehat{\rho})=\frac{1}{N}\big|\big\{j\in [N]:\sum_{i\in B_j}\mathbf{H}_{ii}\geq \lambda_{\max}\widehat{\rho}\big\}\big|=\frac{k_{\#}(\widehat{\rho})s}{d}$.
\end{enumerate}
where $c_1,c_2>0$ are constants and $k_{\#}(\widehat{\rho})$ counts ``good'' (high $\rho$) blocks.
\end{proposition}

\textbf{Observations.}~ Proposition~\ref{prop:prob_rho} highlights a clear distinction: while these methods share the same expected $\rho$, the probability distribution of $\rho$ differs considerably. In the low-rank projection and sparse perturbation cases, the likelihood of $\rho$ deviating from its mean decays exponentially and quadratically with the size of the deviation (i.e., denoted by $\Delta$), indicating a tight concentration around its expectation. In contrast, for block-sparse perturbation, although the mean remains unchanged, the probability becomes at least $1/N$ if there is any such good block.

These observations carry substantial implications. When the subspace defined by $\mathbf{M}$ is chosen purely at random, all three methods exhibit similar convergence in average. However, their distinct $\rho$ distributions become relevant once even modest adaptivity is introduced. In block sparse perturbation, for example, selectively targeting more favorable blocks (or generally identifying subspaces that yield higher $\rho$) may produce performance gains not evident from an expectation-based perspective alone.

\subsection{Generalization Error Bound of Subspace ZO-SGD}\label{subsec:smooth_error}

All the aforementioned analyses are concerned with the convergence of the zeroth-order SGD. In this section, we aim to provide the generalization error bound of zeroth-order SGD. In pursuit of this, we use the uniform stability framework~\citep{bousquet2002stability,hardt2016train,lei2023stability} of the randomized optimization algorithm. 

For analysis, we summarize the notations used in this section. We denote $\mathcal{A}$ by a randomized optimization algorithm such as SGD and $ S\in\mathcal {Z}^n$ by the training dataset where $\mathcal{Z}^n$ represents the collection of datasets with the size $n$ drawn from the data distribution $\mathcal{D}$. The quantity $\mathcal{A}(S)$ denotes the trained parameter using the algorithm $\mathcal{A}$ on the dataset $S$. 

Now, we start with the definition of the generalization error.
\begin{definition}[Generalization Error]
    The generalization error $\epsilon_{\textnormal{gen}}$ is defined by the gap between the population risk $R(\bm \theta) = \mathbb{E}_{\mathbf{z} \sim \mathcal{D}}\left[\loss(\bm \theta; \mathbf{z})\right]$ and the empirical risk $R_S(\bm \theta) = (1/n)\sum_{i=1}^{n} \loss(\bm \theta; \mathbf{z}_i)$ evaluated on the dataset $S = \{\mathbf{z}_1, \mathbf{z}_2, \cdots, \mathbf{z}_n\}$ as $\epsilon_{\textnormal{gen}} = \mathbb{E}_{\mathbf{z} \sim \mathcal{D}}\left[R(\mathcal{A}(S)) - R_S(\mathcal{A}(S))\right]$.
\end{definition}
Owing to previous studies~\citep{bousquet2002stability,hardt2016train} on the generalization, $\epsilon_{\textnormal{gen}}$ is closely related to the uniform stability of the optimization algorithm. 
\begin{definition}[Uniform Stability~\citep{bousquet2002stability,hardt2016train}]
    The randomized algorithm $\mathcal{A}$ is said to be $\epsilon_{\textnormal{stab}}$-uniformly stable if for all neighboring datasets $S, S' \in \mathcal{Z}^n$ such that $S$ and $S'$ differ in only one sample, we have  $\mathbb{E}_{\mathcal{A},S}\left[\loss(\mathcal{A}(S); \mathbf{z}) - \loss(\mathcal{A}(S'); \mathbf{z})\right] \le \epsilon_{\textnormal{stab}}$.
\end{definition}

The next lemma provides the connections between the generalization error and the uniform stability.
\begin{lemma}[Theorem 2.2 in~\citep{hardt2016train}]\label{lemma:gen_stab}
    If $\mathcal{A}$ is $\epsilon_{\textnormal{stab}}$-uniformly stable, then the generalization error is bounded by $|\epsilon_{\textnormal{gen}}| \le \epsilon_{\textnormal{stab}}$.
\end{lemma}

Hence, bounding the generalization error reduces to proving the uniform stability of subspace ZO-SGD under suitable assumptions.

\begin{theorem}[Generalization of Subspace ZO-SGD]\label{thm:generalization_zosgd}
    Suppose that \ref{con:conv_smooth}, \ref{con:conv_gradient}, and $\mathrm{srank}(\mathbf{M}_t) \le s$ hold. Further, suppose that Assumption \ref{assumption:intdim} holds for every $\loss_i$ and $\sup_{\bm\theta, \mathbf{z}} |\loss(\bm\theta; \mathbf{z})| \le 1$. Under the following parameter settings $\alpha_t \le \frac{C}{t}$ and $\mu \lesssim \frac{1}{nL(\rho_{\text{max}} + s\exp(-s/2))}$ for some fixed constant $C > 0$, the generalization error bound of subspace ZO-SGD satisfies $\epsilon_{\textnormal{gen}} \le \mathcal{O}\Big(\frac{T^{1-\frac{1}{1+q}}}{n}\Big)$ where $q = CL\sqrt{s}$ does not rely on $d$.
\end{theorem}

\textbf{On the Constant $q$.}~ A recent study~\citep{liu2024general} applies the uniform stability to ZO algorithms, thereby deriving generalization error bounds for various ZO gradient estimators. However, in that study, the parameter $q = \mathcal{O}(\sqrt{d})$ grows polynomially w.r.t. $d$, which ultimately yields a generalization error bound in the order of $\mathcal{O}(T/n)$ since $\frac{1}{1+q} \approx 0$. Moreover, the previous analysis requires $\mu = \mathcal{O}(1/nLd)$ which is tiny in extremely high-dimensional spaces (ex. fine-tuning LLMs). In contrast, under Theorem \ref{thm:generalization_zosgd}, which exploits the notion of stable rank with local intrinsic dimension, $q$ depends on $s$ while significantly alleviating the order of $\mu$ as $\mu = \mathcal{O}(1/nL\rho_{\text{max}})$. Consequently, we achieve much tighter bound than $\mathcal{O}(T/n)$, especially under $s \ll d$ and Assumption \ref{assumption:intdim}.

\subsection{Empirical Validation}\label{subsec:empirical_validation}

In this section, we empirically substantiate the key theoretical insights from preceding sections using the randomized quadratic minimization problem, $\min_{\mathbf{\theta}\in\mathbb{R}^d}\frac{1}{2}\mathbf{\theta}^\top\mathbf{H}\mathbf{\theta}$ with a synthetically generated positive semi-definite Hessian $\mathbf{H}$. This tractable testbed, commonly used in prior theoretical validation studies~\citep{sun2021worst,zhang2024adam}, allows for direct verification of our claims. Details are provided in Appendix~\ref{app:exp_detail_empirical}.

\begin{figure}[t]
    \centering
    \subfigure{\includegraphics[width=0.325\linewidth]{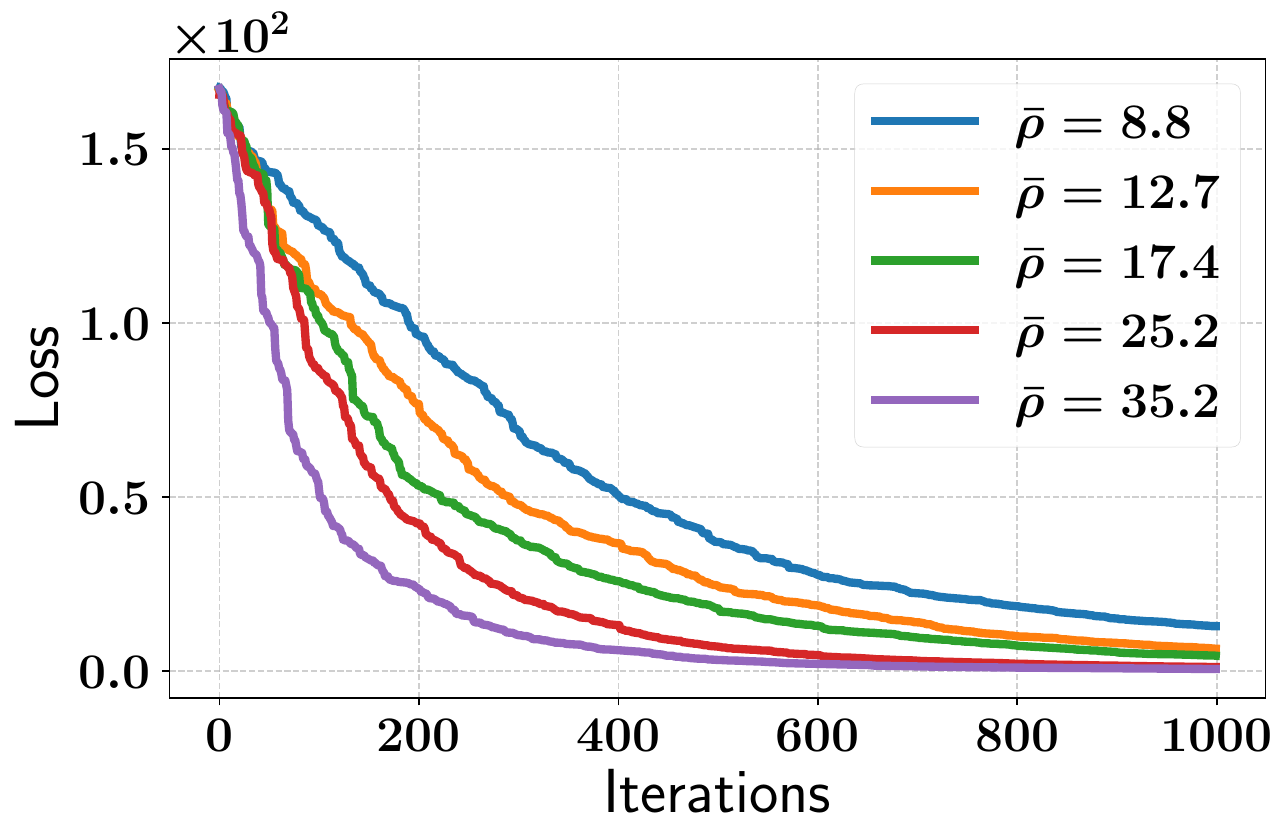}}
    \subfigure{\includegraphics[width=0.325\linewidth]{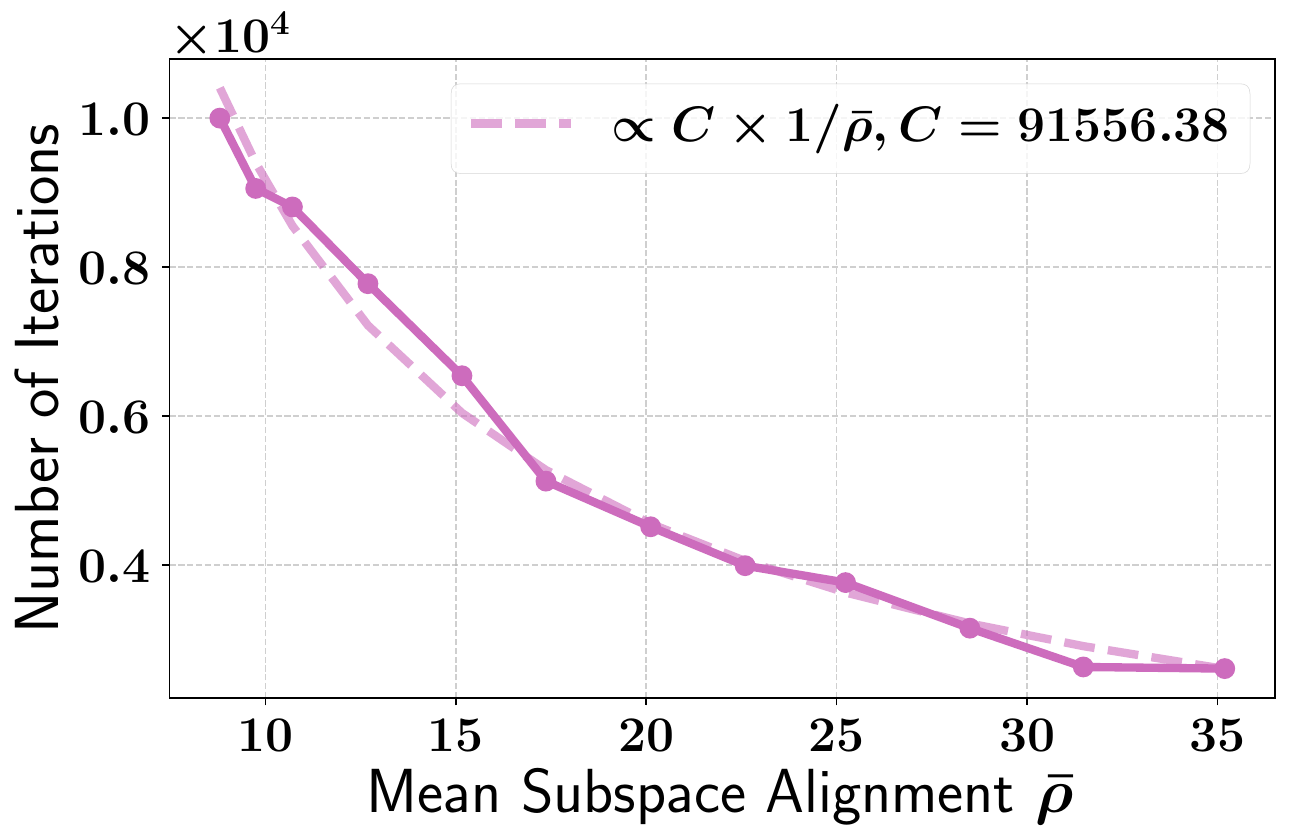}}
    \subfigure{\includegraphics[width=0.32\linewidth]{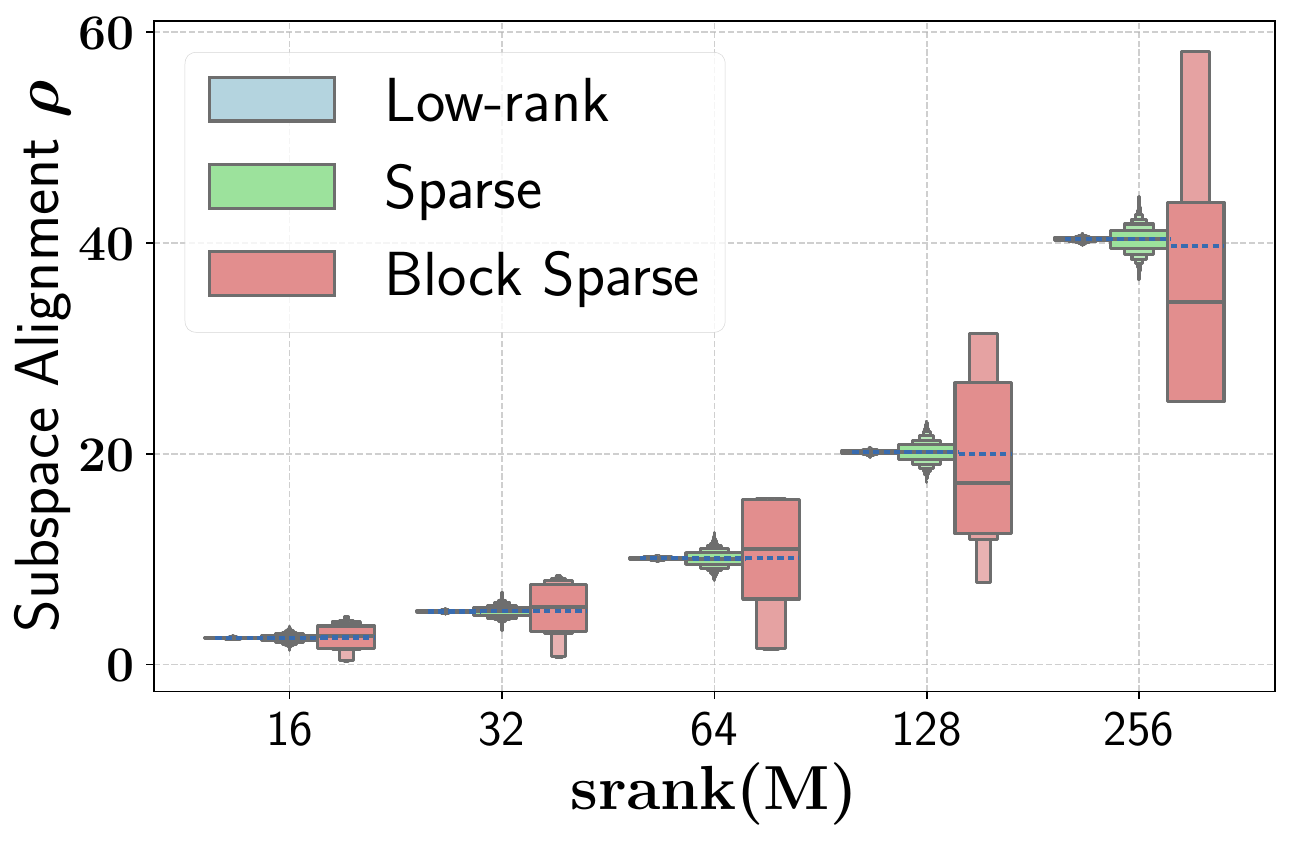}}
    \vspace{-5pt}
    \caption{Empirical validation of theoretical insights on randomized quadratic minimization. \textbf{\emph{Left}}: Higher $\bar\rho$ values yield faster convergence, as predicted ($s$ is fixed).  \textbf{\textit{Middle}}: Iterations required to reach a fixed loss target decrease proportionally with $1/\bar\rho$. \textbf{\textit{Right}}: Distribution of $\rho$ under three $\mathbf{M}$ types shows matching means but different concentration patterns. Horizontal blue dashed lines indicate the $\bar{\rho}$ for each method. Supplementary results are provided in Appendix~\ref{app:supp_empirical_study}.}
    \label{fig:curves_rho}
    \vspace{-15pt}
\end{figure}

\textbf{Effect of mean subspace alignment $\mathbf{\bar\rho}$.}~ Our theoretical framework (Theorem~\ref{thm:improved_convergence_zosgd}) posits a critical role for mean subspace alignment $\bar\rho$ in determining convergence speed. We first validate this by modulating $\bar\rho$ in a controlled experiment using a rank-$64$ dense Hessian with $d=256$ and a low-rank projection with $s=64$ as $\mathbf{M}$. Figure~\ref{fig:curves_rho} (Left) confirms that higher $\bar\rho$ values indeed lead to faster convergence. Consistently, (Middle) demonstrates that the number of iterations required to reach a target loss (set to be the slowest one's final loss) is inversely proportional to $\bar\rho$, directly corroborating our theory that the convergence rate scales with $1/\bar\rho$.

\textbf{Distribution of subspace alignment $\mathbf{\rho}$.}~ Next, we investigate the practical demonstrations of Propositions~\ref{prop:expected_rho} and~\ref{prop:prob_rho}, which concern the expected value and distributional properties of $\rho$ for different instantiations of $\mathbf{M}$. We consider three canonical instantiations provided in Definition~\ref{def:instantiation}. For this, we use a block-diagonal Hessian ($d=1024$, $16$ blocks, each rank $64$) with heterogeneous eigenspectra across blocks, a structure often used to approximate neural network Hessians~\citep{collobert2004large,zhang2024transformers,zhang2024adam}.

Figure~\ref{fig:curves_rho} (right) presents the empirical distributions of $\rho$ from 1,000 trials, revealing two key observations consistent with our theory. \textbf{First}, for any fixed $\mathrm{srank}(\mathbf{M})$, the mean subspace alignment $\widebar{\rho}$ (blue dashed lines) is nearly identical across all three methods, confirming Proposition~\ref{prop:expected_rho} and implying comparable average convergence under purely random selection (see Appendix F.2). \textbf{Second}, although their means coincide, the distributions of $\rho$ differ markedly: \emph{low-rank projection and sparse perturbation exhibit tight concentration around the mean, whereas block sparse perturbation shows a broader spread} owing to heterogeneity in Hessian blocks (a \emph{strong} block yields large $\rho$, and a \emph{weak} block small $\rho$). This aligns with Proposition~\ref{prop:prob_rho}’s assertion regarding differing distributional behaviors. Finally, $\rho$ generally increases with $\mathrm{srank}(\mathbf{M})$ in all methods, as intuitively expected.

\section{Revisiting Block Coordinate Descent for Zeroth-order Optimization}\label{sec:revisiting}

Our theoretical analysis showed that many subspace perturbation schemes, despite structural differences, achieve similar convergence in expectation. This intriguing finding prompts a pivotal question regarding method selection: \emph{If theory offers no clear winner, what factors matter most in practice, particularly for large-scale optimization?} In such scenarios, pragmatic considerations (including memory usage and execution time) emerge as the decisive differentiators. Guided by this imperative for practical efficacy, we revisit the Block Coordinate Descent (BCD) paradigm in this section, an approach rooted in the block sparse perturbation.

\subsection{Why Block Coordinate Descent?}\label{subsec:why_bcd}

\textbf{Memory transfer is the key bottleneck in large-scale ZO.}~ Zeroth-order optimization, by virtue of its forward-only nature, fundamentally shares the memory-bound bottleneck commonly observed in LLM inference~\citep{kwon2023efficient, kim2023squeezellm, gholami2024ai}. For large models, the time spent loading parameters to memory can far exceed actual compute time, making memory bandwidth the dominant factor in overall latency. This challenge is even exacerbated in zeroth-order methods: a single iteration requires two forward evaluations, two parameter perturbations, and one parameter update, each involving the entire parameter. As a result, the total parameter loading per step is effectively equivalent to $5\times$ the inference, a substantial overhead, especially for large-scale models.

\textbf{BCD can reduce the memory transfer.}~ BCD offers a practical solution to this issue by restricting parameter perturbations and updates to a single block at each iteration. While the forward passes still involve the full model, only the selected block is loaded and modified during the perturbation and update phases. If we partition the parameters into $N$ equally sized blocks, the total parameter load per iteration is reduced to $2d + 3d/N$ (two full-model forwards plus block-local perturb and update), compared to the $5d$ required by other ZO methods. This translates into a tangible reduction in memory transfer, leading to faster wall-clock training and thus improved scalability.

Importantly, BCD retains the same convergence guarantees as other qualifying subspace methods~\citep{liu2024sparse, chen2024enhancing}. With its added practical efficiency, it stands out as a particularly attractive choice when theoretical distinctions are absent.

\subsection{MeZO-BCD: Translating Block Sparse Perturbation into a Scalable ZO Method}\label{subsec:mezo-bcd}

Building on the practical advantages of block-wise perturbation discussed above, we now present \emph{MeZO-BCD}, a concrete implementation designed to optimize large-scale LLMs under the zeroth-order paradigm. While the theoretical formulation in Section~\ref{sec:theory_intuition} abstracts perturbation as a projection onto a random subspace, MeZO-BCD operationalizes this by applying perturbation and updates to only one block per iteration.

\textbf{Overall Procedure.}~ Let $\bm\theta_t = [\bm\theta_t^{(1)}, \dots, \bm\theta_t^{(N)}] \in \mathbb{R}^d$ denote model parameters partitioned into $N$ disjoint blocks. Each block typically corresponds to a functional unit of LLMs, such as a transformer decoder layer (comprising attention, FFN, normalization), embedding layer, or output head~\citep{zhang2022opt, touvron2023llama}.

At each iteration $t$, a block index $j_t$ is chosen, and a local perturbation $\mathbf{u}_t \sim \mathcal{N}(\bm{0}, \mathbf{I}_{|\bm\theta_t^{(j_t)}|})$ is sampled using a fixed seed $s_t$. The block-local SPSA estimator is:
\[
    \widehat{\mathbf{g}}_t = \frac{\mathcal{L}(\bm\theta_t + \mu \cdot \mathbf{e}_{B_{j_t}}(\mathbf{u}_t); \mathcal{B}_t) - \mathcal{L}(\bm\theta_t - \mu \cdot \mathbf{e}_{B_{j_t}}(\mathbf{u}_t); \mathcal{B}_t)}{2\mu}\mathbf{u}_t,
\]
where $\mathbf{e}_{B_{j_t}}(\mathbf{u}_t)$ embeds the block-local perturbation into the full parameter space. In practice, no actual zero-padded vector is instantiated as only the relevant block is accessed and modified.

Using the same seed $s_t$, we regenerate $\mathbf{u}_t$ for consistency during the update:
\[
    \bm\theta_{t+1}^{(j_t)} = \bm\theta_t^{(j_t)} - \eta_t \cdot \widehat{\mathbf{g}}_t, \quad \bm\theta_{t+1}^{(k)} = \bm\theta_t^{(k)} \text{ for } k \neq j_t.
\]
This localized update significantly reduces memory transfer overhead, as discussed above, while preserving theoretical convergence behavior.

\textbf{Block selection strategy.}~ While the theory assumes i.i.d. uniform sampling of blocks, such sampling can lead to imbalanced short-term coverage in practice. To address this while preserving the marginal distribution $\mathbb{P}(j_t = j) = 1/N$, we adopt a \emph{cyclic random permutation} strategy, where all blocks are visited exactly once per cycle in a randomly shuffled order. This avoids starvation and maintains the theoretical convergence guarantees. Further details on MeZO-BCD are discussed in Appendix~\ref{app:mezo-bcd}.

\subsection{Experiments on LLM fine-tuning}\label{subsec:llm_exp}

\begin{table*}[t]
    \centering
    \caption{Comparison of peak memory usage and iteration time across methods. Peak memory (GB) is reported on different datasets using OPT-13B, and iteration time (seconds) is measured on SST-2 across various model sizes, with speedup relative to MeZO shown in parentheses. All results are averaged over 100 training steps with a batch size of 16, using a single A100 80GB PCIe.}
    \label{tab:time_mem_comparison}
    \setlength{\tabcolsep}{5pt}
    \resizebox{0.9\textwidth}{!}{
        \begin{tabular}{l|ccc|ccc}
        \toprule
        \multirow{2}{*}{Method} & \multicolumn{3}{c|}{\textbf{Iteration Time (s)}} & \multicolumn{3}{c}{\textbf{Peak Memory (GB)}} \\
        \cmidrule{2-7}
         & OPT-125M & OPT-1.3B & OPT-13B & SST-2 & BoolQ & MultiRC \\
        \midrule
        MeZO & $0.0621~(\times1.00)$ & $0.134~(\times1.00)$ & $1.04~(\times1.00)$ & $25.9$ & $33.6$ & $43.8$ \\
        SparseMeZO & $0.0961~(\times0.66)$ & $0.203~(\times0.66)$ & $1.60~(\times0.65)$ & $26.1$ & $33.6$ & $43.8$ \\
        LoZO & $0.0804~(\times0.77)$ & $0.146~(\times0.92)$ & $0.966~(\times1.08)$ & $25.5$ & $33.6$ & $43.8$ \\
        MeZO-BCD & $\mathbf{0.0339~(\times1.83)}$ & $\mathbf{0.0637~(\times2.10)}$ & $\mathbf{0.375~(\times2.77)}$ & $25.3$ & $33.6$ & $43.8$ \\
        \bottomrule
        \end{tabular}
        \vspace{-5pt}
    }
\end{table*}

\begin{figure}[t]
    \centering
    \includegraphics[width=0.4\linewidth]{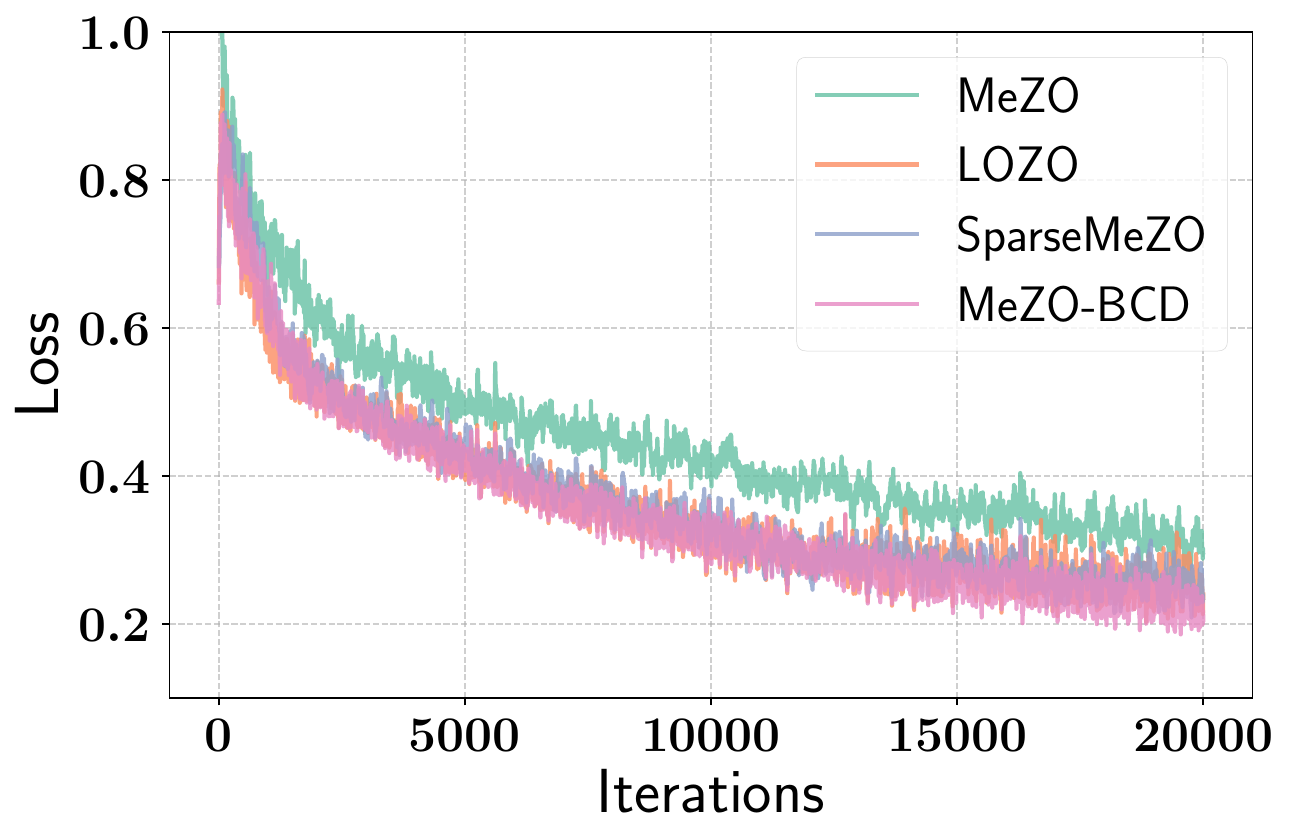}
    \hspace{0.1cm} 
    \includegraphics[width=0.4\linewidth]{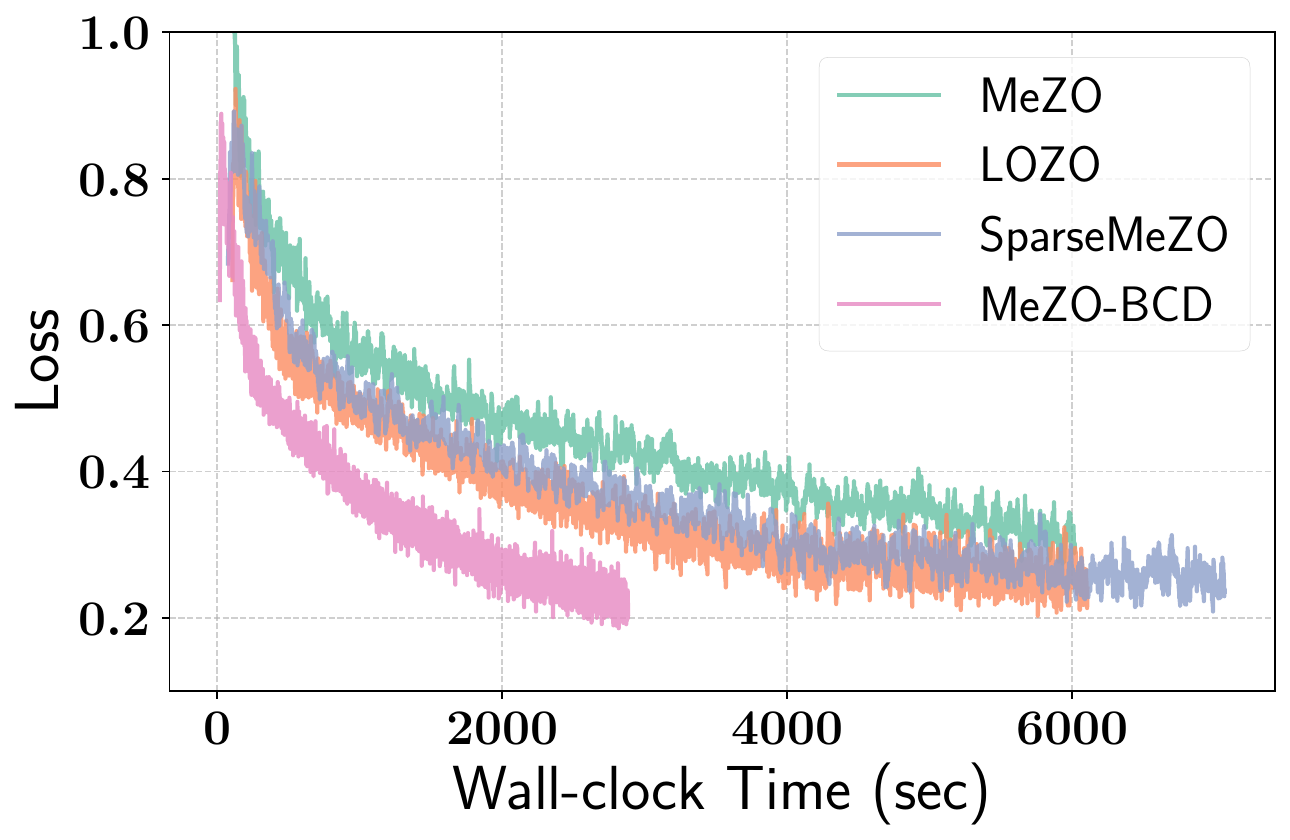}
    \vskip -5pt
    \caption{
        Comparison of training loss curves in terms of iterations (\textit{left}) and wall-clock time (\textit{middle}) on OPT-1.3B~\citep{zhang2022opt} fine-tuning with SST-2~\citep{sst2} for different methods.
    }
    \label{fig:opt-1.3b}
    \vspace{-20pt}
\end{figure}

To empirically demonstrate the practical advantages of MeZO-BCD, particularly its enhanced efficiency and competitive performance, we conduct extensive fine-tuning experiments on LLMs. Key results are presented in Figure~\ref{fig:opt-1.3b} and Table~\ref{tab:time_mem_comparison},~\ref{tab:opt-13b}. Detailed experimental setups, including datasets, hyperparameters and model specifics are provided in Appendix~\ref{app:exp_detail_opt}.

\textbf{MeZO-BCD exhibits superior wall-clock time efficiency.}~ The most significant empirical validation of our approach is MeZO-BCD's superior wall-clock time efficiency, as shown in Figure~\ref{fig:opt-1.3b} (Right) and Table~\ref{tab:time_mem_comparison}. Across different model scales, MeZO-BCD consistently has less iteration time compared to other ZO methods like MeZO, LoZO, and SparseMeZO. Notably, as the model size increases, the acceleration gain becomes higher due to the increased parameter loading overhead. This directly substantiates our central argument from Section~\ref{subsec:why_bcd} that MeZO-BCD's design significantly reduces memory loading overhead, translating into tangible training speedups.

\textbf{MeZO-BCD incurs no additional memory compared to MeZO.}~ MeZO-BCD operates under the same forward-only regime as MeZO but perturbs and updates only a single block at each step. As a result, it avoids any increase in memory footprint. This is empirically confirmed in Table~\ref{tab:time_mem_comparison}, where the peak memory remains effectively identical across different methods\protect\footnotemark.

\footnotetext{Minor differences in peak memory across methods (e.g., on SST-2) are due to system-level factors such as CUDA memory fragmentation or transient buffer allocation, rather than algorithmic differences.}

\begin{table*}[t]
\centering
\caption{Performance of fine-tuning OPT-13B on 1000 examples across various tasks. ICL denotes in-context learning, while FT denotes full fine-tuning via first-order Adam~\citep{KingBa15}. The best results except FT are highlighted in \textbf{bold}. Results for baselines are taken from LoZO~\citep{chen2024enhancing}.}\label{tab:opt-13b}
\resizebox{1.0\linewidth}{!}{
    \begin{tabular}{lccccccccccc}
    \toprule
     Task &\textbf{SST-2} & \textbf{RTE} & \textbf{CB} & \textbf{BoolQ} & \textbf{WSC} & \textbf{WIC}	& \textbf{MultiRC} & \textbf{COPA} & \textbf{ReCoRD} & \textbf{SQuAD} & \textbf{DROP} \\
    Task type & \multicolumn{7}{c}{----------------------------- classification -----------------------------} & \multicolumn{2}{c}{-- multiple choice --} & \multicolumn{2}{c}{--- generation ---}\\
    \midrule
    Zero-shot & 
    $58.8$ & $59.6$ & $46.4$ & $59.0$ & $38.5$ & $55.0$ & $46.9$ & $80.0$ & $81.0$ & $46.2$ & $14.6$ \\
    ICL & 
    $87.0$ & $62.1$ & $57.1$ & $66.9$ & $39.4$ & $50.5$ & $53.1$ & $87.0$ & $\mathbf{82.3}$ & $75.9$ & $29.5$ \\
    \midrule
    MeZO & 
    $91.3$ & $68.2$ & $66.1$ & $68.1$ & $61.5$ & $59.4$ & $59.4$ & $88.0$ & $81.3$ & $81.8$ & $\mathbf{31.3}$ \\
    LoZO & 
    $91.7$ & $\mathbf{70.4}$ & $69.6$ & $71.9$ & $\mathbf{63.5}$ & $60.8$ & $\mathbf{63.0}$ & $\mathbf{89.0}$ & $81.3$ & $\mathbf{84.9}$ & $30.7$ \\
    MeZO-BCD & 
    $\mathbf{93.0}$ & $70.0$ & $\mathbf{71.4}$ & $\mathbf{72.4}$ & $\mathbf{63.5}$ & $\mathbf{61.9}$ & $62.8$ & $\mathbf{89.0}$ & $\mathbf{81.6}$ & $83.7$ & $31.0$ \\
    \midrule
    FT & 
    $91.8$ & $70.9$ & $84.1$ & $76.9$ & $63.5$ & $70.1$ & $71.1$ & $79.0$ & $74.1$ & $84.9$ & $31.3$ \\
    \bottomrule
    \end{tabular}}
\end{table*}

\textbf{MeZO-BCD shows competitive convergence and fine-tuning performance.}~ Notably, this efficiency gain does not compromise iteration complexity or fine-tuning performance. As shown by the training loss curves in Figure~\ref{fig:opt-1.3b} (left), MeZO-BCD converges on par with baseline ZO methods. Furthermore, Table~\ref{tab:opt-13b} demonstrates MeZO-BCD's strong fine-tuning performance on a diverse suite of downstream tasks with an OPT-13B~\citep{zhang2022opt} model, where it frequently outperforms other ZO methods~\citep{malladi2023fine,chen2024enhancing}\protect\footnotemark. Additional experimental results are provided in Appendix~\ref{app:stat_sig}.

\footnotetext{Since SparseMeZO's official implementation is unavailable, and our implementation cannot fully reproduce its reported performance in~\citep{liu2024sparse}, we exclude this baseline to maintain a fair comparison.}

\section{Discussion}
\subsection{MeZO-BCD as a Scalable Foundation for Zeroth-Order Optimization}

\paragraph{Adaptive Block Selection.}
While our theoretical analysis (Propositions~\ref{prop:expected_rho} and~\ref{prop:prob_rho}) shows that various subspace perturbation schemes share similar expected convergence rates, their subspace alignment $\rho$ can differ substantially in distribution. This highlights a natural opportunity for improvement via adaptive selection of subspaces (specifically, blocks in the case of MeZO-BCD). Compared to low-rank or sparse perturbation, where adaptivity requires searching over continuous subspaces or combinatorial coordinate subsets, MeZO-BCD operates on a small set of architecturally defined, semantically meaningful blocks. This makes adaptive scheduling more tractable and interpretable in practice. Moreover, since blocks often differ in curvature or gradient scale across layers in LLMs~\citep{zhang2024transformers,zhang2024adam}, MeZO-BCD provides a well-suited interface for learning dynamics-aware scheduling. We leave the empirical realization of this idea as promising future work.

\paragraph{Integration with Stateful Optimizers.}
Stateful optimizers such as Adam~\citep{KingBa15} have been instrumental in first-order optimization, but their extension to zeroth-order methods has shown limited benefits. Previous works~\citep{malladi2023fine,zo-bench} found marginal performance improvements and substantial memory overhead due to per-parameter optimizer state storage. MeZO-BCD offers a new path forward: by perturbing and updating only one block at a time, optimizer states can be stored and updated locally for the active block, without incurring full-model memory costs. This enables practical and scalable use of stateful optimizers in ZO settings, potentially opening a new line of research into adaptive gradient-free methods that leverage momentum, variance correction, or block-wise learning rates in a resource-efficient way. Preliminary investigations on the future directions are provided in Appendix~\ref{app:future_directions}.

\subsection{Limitations}

Despite its scalability and extensibility, MeZO-BCD has several limitations. First, while our theoretical framework provides general insights into the role of subspace perturbations and dimensionality, it does not offer task-specific guidance on optimal perturbation structures or block sizes. Second, although we highlight promising extensions including adaptive block selection and integration with stateful optimizers, this work focuses on establishing the core algorithm and validating its foundational benefits. Empirical investigation of these directions remains an important avenue for future work.

\section{Conclusion}

In this work, we introduced a unified theoretical framework for understanding zeroth-order optimization under subspace perturbations, connecting convergence and generalization through the notion of subspace alignment. Our analysis provides new insights into how different perturbation strategies behave and what distinguishes them beyond that. Building on the theoretical insights, we proposed MeZO-BCD, a block coordinate descent method for zeroth-order optimization that achieves substantial practical speedups. Our experiments on LLM fine-tuning confirm its efficiency and competitiveness, and its structured design opens promising directions for future research, including adaptive block selection and integration with stateful optimizers.

\bibliographystyle{abbrvnat_lastname_first}
\bibliography{reference}




\clearpage
\appendix
\addcontentsline{toc}{section}{Appendix}
\part{\LARGE Supplementary Materials}

\parttoc

\section*{Broader Impact Statement}

This work improves the efficiency and scalability of zeroth-order optimization for fine-tuning large language models, offering a practical alternative to gradient-based approaches in resource-constrained settings. By lowering computational and memory demands, it may help reduce the environmental footprint of large-scale AI training and broaden access to advanced machine learning technologies. However, expanding access to powerful fine-tuning tools also raises societal concerns, including the risk of misuse or unintended consequences from customized language models. We advocate for responsible usage, with attention to transparency, fairness, and ongoing evaluation of social impact as such methods are deployed in practice.

\clearpage
\section{Related Works}
\paragraph{Convergence and Generalization in Zeroth-order optimization.}
Theoretical analysis of zeroth-order (ZO) optimization has been an active research area, with extensive work on its convergence properties across various settings. Early studies established fundamental complexity results for both first-order and zeroth-order methods using Gaussian smoothing~\citep{ghadimi2013stochastic}. Later, optimal convergence bounds for zeroth-order stochastic gradient descent (SGD) were derived for convex objectives~\citep{duchi2015optimal} and extended to Gaussian-smoothing-based approaches~\citep{nesterov2017random}. In nonconvex settings, variance-reduced ZO optimization methods such as ZO-SVRG and ZO-SPIDER were introduced to improve convergence rates and function query efficiency~\citep{ji2019improved}. More recent works have analyzed ZO optimization under finite-difference, linear interpolation, and various smoothing schemes~\citep{berahas2022theoretical}, while others have investigated its behavior in non-smooth regimes, providing sample complexity bounds~\citep{zhang2020complexity, davis2022gradient} and demonstrating that non-smooth ZO optimization is not necessarily more difficult than the smooth case~\citep{rando2023an, kornowski2024algorithm}.

Beyond convergence, the generalization properties of ZO optimization remain underexplored. While existing theoretical studies primarily focus on convergence rates and complexity bounds, fewer works provide a rigorous understanding of how zeroth-order methods generalize in high-dimensional settings. Recent efforts have applied stability-based analyses to first-order optimization, revealing the role of noise in controlling generalization error~\citep{hardt2016train, lei2023stability}. However, extending such frameworks to ZO optimization poses unique challenges due to the inherently higher variance of zeroth-order gradient estimates. Some studies have begun addressing this issue, examining stability bounds for ZO methods~\citep{liu2024general}, but a comprehensive theoretical framework connecting structured perturbations to both convergence and generalization is still lacking.

\paragraph{Fine-tuning Large Language Models with Zeroth-order Optimization.}
Zeroth-order (ZO) optimization, a long-standing technique in conventional machine learning~\citep{spall1992multivariate, ghadimi2013stochastic}, was first explored for large language model (LLM) fine-tuning by MeZO~\citep{malladi2023fine}. By generating perturbations dynamically using random seeds, MeZO eliminates the need to store large perturbation matrices, significantly reducing memory overhead. This allows ZO optimization to operate with memory requirements comparable to inference. While MeZO provides theoretical guarantees, suggesting that scaling the learning rate with respect to the problem dimension ensures dimension-free convergence, this approach inherently results in slow convergence.

Several approaches have been proposed to improve the efficiency of ZO optimization. MeZO-SVRG~\citep{gautam2024variance} integrates the SVRG algorithm to reduce gradient variance, leading to faster convergence and improved fine-tuning performance. SparseMeZO~\citep{liu2024sparse} enhances computational efficiency by updating only a subset of parameters, while Fisher-informed sparsity~\citep{guo2024zeroth} prioritizes updates based on Fisher information. A benchmarking study~\citep{zo-bench} evaluates various ZO methods, including SGD, SignSGD, and Adam. More recently, LOZO~\citep{chen2024enhancing} introduced low-rank perturbations, leveraging the low intrinsic dimensionality of gradients in LLMs to improve convergence and fine-tuning performance.

Building on these works, this work develops a unified theoretical framework that explains how structured perturbations improve both convergence and generalization in ZO optimization. Through this perspective, block coordinate descent (BCD) is identified as an effective structured perturbation strategy, and its advantages are validated through theoretical analysis and empirical evaluation.
\clearpage
\section{Proof for Convergence}
In this section, we provide the theoretical analysis of Section \ref{sec:theory_intuition} in detail.

\subsection{Discussion on the Use of Assumption~\ref{assumption:intdim}}\label{app:assumption}
Assumption~\ref{assumption:intdim} is well-justified in the context of large-scale language model optimization. A closely related, but even stronger, assumption was already adopted in MeZO~\citep{malladi2023fine}, where it was supported both theoretically and empirically. Specifically, MeZO justified this condition by noting that the Hessian of the loss for deep neural networks (especially during fine-tuning) often exhibits a remarkably low local effective rank~\citep{papyan2018full,papyan2020traces,ghorbani2019investigation,wu2020dissecting, sagun2017empirical, zhang2024adam, zhang2024transformers}, and that even for billion-parameter models, zeroth-order optimization remains effective, as also evidenced by the rapid convergence observed in practice. These empirical phenomena are further supported by numerous studies demonstrating that the Hessian spectrum is highly concentrated near zero, with only a handful of large outliers providing an effective upper bound on the local rank.

Our formulation is strictly less restrictive than MeZO’s, as it only requires the assumption to hold within a local $\mathcal{O}(s)$-ball and allows for a strictly positive smoothing parameter $\mu$. Given that the original assumption was already validated in prior work, adopting this weaker version in our analysis is not only natural but also better aligned with the practical realities of large-scale model fine-tuning.

\subsection{Preliminaries}
Recall that our zeroth-order gradient estimator and the subspace ZO-SGD.
\begin{align*}
    \grad_t & = \frac{\loss_{\mathcal{B}_t}(\bm\theta_t + \mu \mathbf{M}_t \mathbf{u}_t) - \loss_{\batch_t}(\bm\theta_t - \mu \mathbf{M}_t\mathbf{u}_t)}{2\mu} \mathbf{M}_t\mathbf{u}_t \\
    \bm\theta_{t+1} & = \bm\theta_t - \alpha \grad_t 
\end{align*}
where $\loss_{\batch_t}(\bm\theta_t) \coloneqq \frac{1}{|\batch_t|}\sum\limits_{i\in\batch_t} \loss(\bm\theta_t; \mathbf{z}_i)$ is the loss function evaluated over the minibatch sample $\batch_t$ at time $t$.

\paragraph{Gaussian Smoothing.} The zeroth-order gradient estimate $\grad_t$ is unbiased estimator of the following smoothed loss.
\begin{align*}
    \loss_\mu(\bm\theta) & \coloneqq \mathbb{E}_{\mathbf{u} \sim \mathcal{N}(\bm 0, \mathbf{I}_d)}\big[\loss(\bm\theta + \mu \mathbf{u})\big] \\
    \loss_{\mu, \mathbf{M}_t}(\bm\theta) & \coloneqq \mathbb{E}_{\mathbf{u} \sim \mathcal{N}(\bm 0, \mathbf{I}_d)}\big[\loss(\bm\theta + \mu \mathbf{M}_t\mathbf{u})\big] 
\end{align*}
Note that it is clear that $\loss_\mu$ approaches to the original loss $\loss$ as $\mu \rightarrow 0$. This indicates that the estimation error between $\loss_\mu$ and $\loss$ or the distance of gradients between $\nabla \loss_\mu$ and $\nabla \loss$ would depend on the smoothing parameter $\mu$, thus the key point is ``how to handle this mismatch''. 

\begin{lemma}[Alignment of Smoothed Gradients]
    The following inner product could be lower-bounded as follows.
    \begin{align*}
        \big\langle \nabla \loss_\mu(\bm\theta), \nabla \loss_{\mu, \mathbf{M}_t}(\bm\theta) \big\rangle & \ge \|\nabla\loss_\mu(\bm\theta)\|^2 - \|\nabla\loss_\mu(\bm\theta)\| \mu L \sqrt{2(d+s)}
    \end{align*}
\end{lemma}

\begin{proof}
    By the definition, we have
    \begin{align*}
        \|\nabla\loss_{\mu, \mathbf{M}_t}(\bm\theta) - \nabla\loss_\mu(\bm\theta)\| & = \Big\|\mathbb{E}_{\mathbf{u}}\big[\nabla \loss(\bm\theta + \mu \mathbf{M}_t\mathbf{u}) - \mathbb{E}_{\mathbf{u}}\big[\nabla \loss(\bm\theta + \mu \mathbf{u})\big]\Big\| \\
        & = \Big\|\mathbb{E}_{\mathbf{u}}\big[\nabla \loss(\bm\theta + \mu \mathbf{M}_t\mathbf{u}) - \nabla \loss(\bm\theta + \mu \mathbf{u})\big]\Big\| \\
        & \le \mathbb{E}_{\mathbf{u}}\big[\big\|\nabla \loss(\bm\theta + \mu \mathbf{M}_t\mathbf{u}) - \nabla \loss(\bm\theta + \mu \mathbf{u})\big\|\big] \\
        & \le \mu L\mathbb{E}_\mathbf{u}\big[\|(\mathbf{M}_t - \mathbf{I}_d)\mathbf{u}\|\big] \\
        & \le \mu L \sqrt{\mathrm{Tr}\Big((\mathbf{M}_t - \mathbf{I}_d)^\top (\mathbf{M}_t - \mathbf{I}_d)\Big)}
    \end{align*}
    We let
    \begin{align*}
        R \coloneqq \mathrm{Tr}\Big((\mathbf{M}_t - \mathbf{I}_d)^\top (\mathbf{M}_t - \mathbf{I}_d)\Big)
    \end{align*}
    for simplicity. For the positive semi-definite matrix $\mathbf{M}_t$ with $\mathrm{srank}(\mathbf{M}_t) \le s$, it follows that
    \begin{align*}
        R & = \mathrm{Tr}\Big((\mathbf{M}_t - \mathbf{I}_d)^\top (\mathbf{M}_t - \mathbf{I}_d)\Big) \\
        & = \|\mathbf{M}_t - \mathbf{I}_d\|_F^2 \\
        & \le (\|\mathbf{M}_t\|_F + \|\mathbf{I}_d\|)^2 \\
        & \le 2\|\mathbf{M}_t\|_F^2 + 2\|\mathbf{I}_d\|_F^2 \\
        & = 2(d+s)
    \end{align*}
    Under the maximal difference $\mu L \sqrt{2(d+s)}$ between $\nabla \loss_{\mu, \mathbf{M}_t}(\bm\theta)$ and $\nabla\loss_\mu(\bm\theta)$, it is easy to derive that the inner product is lower-bounded as
    \begin{align*}
        \big\langle \nabla \loss_\mu(\bm\theta), \nabla\loss_{\mu, \mathbf{M}_t}(\bm\theta) \big\rangle & \ge  
        \|\nabla\loss_\mu(\bm\theta)\|^2 - \|\nabla\loss_\mu(\bm\theta)\| \mu L \sqrt{2(d+s)}
    \end{align*}
\end{proof}

\subsection{Technical Lemmas for Matrix Calculus}

\begin{lemma}[Trace of Matrix Product]\label{lemma:trace_matmul}
    For positive semi-definite matrices $\mathbf{A}, \mathbf{B} \in \mathbb{R}^{d \times d}$, the following holds.
    \begin{align*}
        \mathrm{Tr}(\mathbf{AB}) \le \mathrm{Tr}(\mathbf{A})\mathrm{Tr}(\mathbf{B})
    \end{align*}
\end{lemma}
\begin{proof}
    Let $\mathbf{v}_i$ be an orthonormal basis for $\mathbf{B}$ and $\{\lambda_i\}_{i=1}^{d}$ be corresponding eigenvalues. Then, we have
    \begin{align*}
        \mathrm{Tr}(\mathbf{AB}) = \sum\limits_{i=1}^{d} \langle \mathbf{AB}\mathbf{v}_i, \mathbf{v}_i \rangle = \sum\limits_{i=1}^{d} \lambda_i \langle \mathbf{A}\mathbf{v}_i, \mathbf{v}_i \rangle = \max\limits_{i \in[d]} \lambda_i \mathrm{Tr}(\mathbf{A}) \le \mathrm{Tr}(\mathbf{A})\mathrm{Tr}(\mathbf{B})
    \end{align*}
\end{proof}

\begin{lemma}[Lemma 1 in~\citep{nesterov2017random}]\label{lemma:gaussian_norm}
    Let $\mathbf{u} \sim \mathcal{N}(\bm 0, \mathbf{I}_d)$. Then, the expectation of the moments satisfies
    \begin{align*}
        \mathbb{E}\big[\|\mathbf{u}\|] & = 1 \\
        \mathbb{E}\big[\|\mathbf{u}\|^2] & = d \\
        \mathbb{E}\big[\|\mathbf{u}\|^n] & \le (d+n)^{n/2}
    \end{align*}
    for $n \ge 2$. 
\end{lemma}

\begin{lemma}[$n$-th Moments of Quadratic Forms for $n = 1,2,3,$ and $4$~\citep{magnus1978moments}]\label{lemma:nth_moments}
    Let $\mathbf{u} \sim \mathcal{N}(\bm 0, \mathbf{I}_d)$ and $\mathbf{A}$ be a positive semi-definite matrix. Then, the expectation of the folllowings are
    \begin{align*}
        \mathbb{E}\Big[\mathbf{u}^\mathsf{T} \mathbf{A} \mathbf{u}\Big] & = \mathrm{Tr}(\mathbf{A}) \\
        \mathbb{E}\Big[\big(\mathbf{u}^\mathsf{T} \mathbf{A} \mathbf{u}\big)^2\Big] & = (\mathrm{Tr}(\mathbf{A}))^2 + 2\mathrm{Tr}(\mathbf{A}^2) \le 3\mathrm{Tr}(\mathbf{A})^2 \\
        \mathbb{E}\Big[(\mathbf{u}^\top \mathbf{A} \mathbf{u})^3\Big] & = \mathrm{Tr}(A)^3 + 6\mathrm{Tr}(\mathbf{A}) \mathrm{Tr}(\mathbf{A}^2) + 8\mathrm{Tr}(\mathbf{A}^3) \le 15\mathrm{Tr}(\mathbf{A})^3 \\
        \mathbb{E}\Big[\big(\mathbf{u}^\mathsf{T} \mathbf{A} \mathbf{u}\big)^4\Big] & = (\mathrm{Tr}(\mathbf{A}))^4 + 32\mathrm{Tr}(\mathbf{A})\mathrm{Tr}(\mathbf{A}^3) + 12\big(\mathrm{Tr}(\mathbf{A}^2)\big)^2 + 12\big(\mathrm{Tr}(\mathbf{A})\big)^2\mathrm{Tr}(\mathbf{A}^2) + 48\mathrm{Tr}(\mathbf{A}^4) \\
        & \le 105\mathrm{Tr}(\mathbf{A})^4
    \end{align*}
    where the inequalities come from Lemma \ref{lemma:trace_matmul}.
\end{lemma}

\begin{lemma}\label{lemma:gaussian_inner_prod}
    Let $\mathbf{u} \sim \mathcal{N}(\bm 0, \mathbf{I}_d)$ and $\mathbf{Q} \in \mathbb{R}^{d \times d}$ be an arbitrary matrix. For any vector $\mathbf{a} = (a_1, a_2, \cdots, a_d) \in \mathbb{R}^d$, it should hold that
    \begin{align*}
        \mathbb{E}\big[\|\langle \mathbf{a}, \mathbf{Qu} \rangle \mathbf{Qu}\|^2] & = \mathrm{Tr}(\mathbf{Q}\mathbf{Q}^\mathsf{T}) \|\mathbf{Q}^\mathsf{T} \mathbf{a}\|^2 + 2\mathbf{a}^\mathsf{T} (\mathbf{Q}\mathbf{Q}^\mathsf{T})^2 \mathbf{a} 
    \end{align*}
\end{lemma}

\begin{proof}
    Let $\mathbf{X} = \mathbf{Qu}$ and $\bm \Sigma = \mathbf{Q}\mathbf{Q}^\mathsf{T}$. Since $\mathbf{u} \sim \mathcal{N}(\bm 0, \mathbf{I}_d)$, we have $\mathbf{X} \sim \mathcal{N}(\bm 0, \bm \Sigma)$. Note that $\mathbf{X}$ is $d$-dimensional random vector with $\mathbf{X} = (X_1, X_2, \cdots, X_d)$. Also, the quantity to be expected could be re-written as
    \begin{align*}
        \|\langle \mathbf{a}, \mathbf{Qu}\rangle \mathbf{Qu}\|^2 & = (\mathbf{a}^\mathsf{T} \mathbf{X})^2 \|\mathbf{X}\|^2 \\
        & = \left(\sum\limits_{i,j} a_i a_j X_i X_j\right)\left(\sum\limits_{k} X_k^2\right) \\
        & = \sum\limits_{i,j,k} a_i a_j X_i X_j X_k^2 
    \end{align*}
    Note that all we need to derive is $\mathbb{E}[X_i X_j X_k^2]$. Since $\mathbf{X}$ follows $\mathcal{N}(\bm 0, \bm \Sigma)$ distribution, it should hold that $X_i$, $X_j$, and $X_k$ follow $\mathcal{N}(0, \Sigma_{ii})$, $\mathcal{N}(0, \Sigma_{jj})$, and $\mathcal{N}(0, \Sigma_{kk})$, respectively. 
    Thanks to Isserlis' theorem~\citep{isserlis1918formula} or also known as Wick's probability theorem~\citep{wick1950evaluation}, we have
    \begin{align*}
        \mathbb{E}\big[X_i X_j X_k^2\big] & = \mathbb{E}\big[X_i X_j X_k X_k\big] \\
        & = \mathbb{E}\big[X_i X_j\big] \mathbb{E}\big[X_k X_k\big] + \mathbb{E}\big[X_i X_k\big] \mathbb{E}\big[X_j X_k\big] + \mathbb{E}\big[X_i X_k\big] \mathbb{E}\big[X_j X_k\big] \\
        & = \mathbb{E}\big[X_i X_j\big] \mathbb{E}\big[X_k^2\big] + 2\mathbb{E}\big[X_i X_k\big] \mathbb{E}\big[X_j X_k\big] \\
        & = \Sigma_{ij}\Sigma_{kk} + 2\Sigma_{ik}\Sigma_{jk} 
    \end{align*}
    due to the fact $\mathbf{X} \sim \mathcal{N}(\bm 0, \bm \Sigma)$ and $\bm \Sigma = [\Sigma_{mn}]$ for $1 \le m, n \le d$. Arranging all items, we have
    \begin{align*}
        \mathbb{E}\big[\|\langle \mathbf{a}, \mathbf{Qu}\rangle \mathbf{Qu}\|^2\big] & = \sum\limits_{i,j,k} a_i a_j \mathbb{E}\big[X_i X_j X_k^2\big] \\
        & = \sum\limits_{i,j,k} a_i a_j \big(\Sigma_{ij} \Sigma_{kk} + 2\Sigma_{ik} \Sigma_{jk}\big) \\
        & = \underbrace{\sum\limits_{i,j,k} a_i a_j \Sigma_{ij} \Sigma_{kk}}_{T_1} + \underbrace{2\sum\limits_{i,j,k} a_i a_j \Sigma_{ik} \Sigma_{jk}}_{T_2}
    \end{align*}
    We first compute $T_1$. Note that $\Sigma_{kk}$ in $T_1$ only depends on the index $k$, thus we have
    \begin{align*}
        T_1 & = \sum\limits_{i,j,k} a_i a_j \Sigma_{ij} \Sigma_{kk} \\
        & = \left(\sum\limits_{k=1}^{d} \Sigma_{kk}\right) \left(\sum\limits_{i,j} a_i a_j \Sigma_{ij} \right) = \mathrm{Tr}(\bm\Sigma) \mathbf{a}^\mathsf{T} \bm\Sigma \mathbf{a} \\
        & = \mathrm{Tr}\big(\mathbf{Q}\mathbf{Q}^\mathsf{T}\big) \mathbf{a}^\mathsf{T} \mathbf{Q} \mathbf{Q}^\mathsf{T} \mathbf{a} = \mathrm{Tr}\big(\mathbf{Q}\mathbf{Q}^\mathsf{T}\big) \big\|\mathbf{Q}^\mathsf{T} \mathbf{a}\big\|^2
    \end{align*}
    Now, we compute $T_2$. Note that for a fixed $k \in [d]$, $\sum_{i=1}^{d} a_i \Sigma_{ik} = [\bm\Sigma \mathbf{a}]_k$ holds by the definition of matrix multiplication. Hence, we have
    \begin{align*}
        T_2 & = 2\sum\limits_{i,j,k} a_i a_j \Sigma_{ik} \Sigma_{jk} \\
        & = 2\sum\limits_{k=1}^{d} \sum\limits_{i,j} a_i \Sigma_{ik} a_j \Sigma_{jk} = 2\sum\limits_{k=1}^{d} \left(\sum\limits_{i=1}^{d} a_i \Sigma_{ik}\right) \left(\sum\limits_{j=1}^{d} a_j \Sigma_{jk}\right) \\
        & = 2\sum\limits_{k=1}^{d} [\bm \Sigma \mathbf{a}]_k^2 = 2\|\bm \Sigma \mathbf{a}\|^2 = 2\mathbf{a}^\mathsf{T} \bm \Sigma^2 \mathbf{a} \\
        & = 2 \mathbf{a}^\mathsf{T} \big(\mathbf{Q}\mathbf{Q}^\mathsf{T}\big)^2 \mathbf{a} 
    \end{align*}
    Combining $T_1$ and $T_2$, we finally obtain
    \begin{align*}
        \mathbb{E}\big[\|\langle \mathbf{a}, \mathbf{Qu} \rangle \mathbf{Qu}\|^2] & = \mathrm{Tr}(\mathbf{Q}\mathbf{Q}^\mathsf{T}) \|\mathbf{Q}^\mathsf{T} \mathbf{a}\|^2 + 2\mathbf{a}^\mathsf{T} (\mathbf{Q}\mathbf{Q}^\mathsf{T})^2 \mathbf{a}
    \end{align*}
\end{proof}

\subsection{Technical Lemmas for Gaussian Smoothed Loss}

For completeness of the paper, we include several lemmas already known in zeroth-order optimization literature.
\begin{lemma}
    Let $\loss_{\mu}(\bm\theta) = \mathbb{E}\big[\loss(\bm\theta + \mu \mathbf{u})\big]$ for some positive $\mu > 0$. Then, the following holds.
    \begin{align*}
        \|\nabla\loss_{\mu}(\bm\theta_t) - \nabla\loss(\bm\theta_t)\| & \le \frac{\mu L (d+3)^{3/2}}{2} \\
        \|\nabla\loss(\bm\theta_t)\|^2 & \le \frac{\mu^2 L^2 (d+3)^3}{2} + 2\|\nabla\loss_{\mu}(\bm\theta_t)\|^2 
    \end{align*}
\end{lemma}

\begin{proof}
    Since $\loss$ is $L$-smooth, we have for $\mathbf{u} \sim \mathcal{N}(\bm 0, \mathbf{I}_d)$
    \begin{align*}
        \loss(\bm\theta_t + \mu \mathbf{u}) & \le \loss(\bm\theta_t) + \mu \nabla\loss(\bm\theta_t)^\mathsf{T} \mathbf{u} + \frac{\mu^2 L}{2} \|\mathbf{u}\|^2.
    \end{align*}
    Thus, the distance between $\nabla\loss_{\mu}(\bm\theta_t)$ and $\nabla\loss(\bm\theta_t)$ can be bounded as
    \begin{align*}
        \|\nabla\loss_{\mu}(\bm\theta_t) - \nabla \loss(\bm\theta_t)\| & = \frac{1}{Z}\left\|\int_{\mathbf{u}} \left(\frac{\loss(\bm\theta_t + \mu \mathbf{u}) - \loss(\bm\theta_t)}{\mu} - \nabla\loss(\bm\theta_t)\right)\mathbf{u}e^{-\frac{1}{2}\|\mathbf{u}\|^2}\mathrm{d}\mathbf{u}\right\| \\
        & = \frac{1}{Z}\left\|\int_{\mathbf{u}} \left(\frac{\loss(\bm\theta_t + \mu \mathbf{u}) - \loss(\bm\theta_t) - \mu\langle \nabla\loss(\bm\theta_t), \mathbf{u}\rangle}{\mu}\right)\mathbf{u}e^{-\frac{1}{2}\|\mathbf{u}\|^2}\mathrm{d}\mathbf{u}\right\| \\
        & \le \frac{1}{Z} \int_{\mathbf{u}}\left|\frac{\loss(\bm\theta_t + \mu \mathbf{u}) - \loss(\bm\theta_t) - \mu\langle \nabla\loss(\bm\theta_t), \mathbf{u}\rangle}{\mu}\right| \|\mathbf{u}\| e^{-\frac{1}{2}\|\mathbf{u}\|^2} \mathrm{d}\mathbf{u} \\
        & \le \frac{1}{Z} \int_{\mathbf{u}} \frac{\mu L}{2} \|\mathbf{u}\|^3 e^{-\frac{1}{2}\|\mathbf{u}\|^2} \mathrm{d}\mathbf{u} \\
        & \le \frac{\mu L (d+3)^{3/2}}{2}
    \end{align*}
    where the second inequality comes from the fact $\mathbb{E}[\langle \mathbf{a}, \mathbf{u}\rangle \mathbf{u}] = \mathbf{a}$ for any vector $\mathbf{a} \in \mathbb{R}^d$. By triangle inequality and Young's inequality, we have
    \begin{align*}
        \|\nabla\loss(\bm\theta_t)\|^2 & \le 2\|\nabla\loss_{\mu}(\bm\theta_t) - \nabla \loss(\bm\theta_t)\|^2 + \|\nabla\loss_{\mu}(\bm\theta_t)\|^2 \\
        & \le \frac{\mu^2 L^2 (d+3)^3}{2} + 2\|\nabla\loss_{\mu}(\bm\theta_t)\|^2
    \end{align*}
\end{proof}

\begin{lemma}[Expected Norm of $\grad_t$]\label{lemma:expected_grad}
    Suppose that $\mathrm{srank}(\mathbf{M}_t) \le s$. Then, the expected norm of structured zeroth-order gradient satisfies  
    \begin{align*}
        \mathbb{E}\big[\|\grad_t\|\big]^2 \le \mathbb{E}\big[\|\grad_t\|^2\big] & \le \frac{15\mu^2 L^2 s^3}{2} + 4(s + 2) \left(\|\nabla\loss(\bm\theta_t)\|^2 + \frac{\sigma^2}{B}\right)
    \end{align*}
\end{lemma}

\begin{proof}
    By the definition of $\grad_t$ and $L$-smoothness, we have
    \begin{align*}
        \|\grad_t\|^2 & = \left\|\frac{\loss_{\batch_t}(\bm\theta_t + \mu \mathbf{M}_t\mathbf{u}_t) - \loss_{\batch_t}(\bm\theta_t - \mu \mathbf{M}_t\mathbf{u}_t)}{2\mu} \mathbf{M}_t\mathbf{u}_t\right\|^2 \\
        & \le \frac{2}{\mu^2} \left\|\Big(\loss_{\batch_t}(\bm\theta_t + \mu \mathbf{M}_t \mathbf{u}_t) - \loss_{\batch_t}(\bm\theta_t) - \mu \langle \nabla\loss_{\batch_t}(\bm\theta_t), \mathbf{M}_t \mathbf{u}_t\rangle\Big) \mathbf{M}_t\mathbf{u}_t\right\|^2 + 2\big\|\langle \nabla\loss_{\batch_t}(\bm\theta_t), \mathbf{M}_t \mathbf{u}_t\rangle \mathbf{M}_t\mathbf{u}_t\big\|^2 \\
        & \le \frac{\mu^2 L^2}{2} \|\mathbf{M}_t\mathbf{u}_t\|^6 + 2\big\|\langle \nabla\loss_{\batch_t}(\bm\theta_t), \mathbf{M}_t\mathbf{u}_t\rangle \mathbf{M}_t\mathbf{u}_t\big\|^2
    \end{align*}
    Note that the first term can be bounded as follows.
    \begin{align*}
        \mathbb{E}\big[\|\mathbf{M}_t\mathbf{u}_t\|^2\big] & = \mathbb{E}\big[\mathbf{u}_t^\top \mathbf{M}_t^2 \mathbf{u}_t\big] = \mathrm{Tr}(\mathbf{M}_t^2) \le s \\
        \mathbb{E}\big[\|\mathbf{M}_t\mathbf{u}_t\|^6\big]& = \mathbb{E}\big[(\mathbf{u}_t^\top \mathbf{M}_t^2 \mathbf{u}_t)^3\big] \\
        & \le 15\mathrm{Tr}(\mathbf{M}_t^2)^3 \\
        & \le 15s^3 
    \end{align*}
    under Lemma \ref{lemma:nth_moments} and $\mathrm{srank}(\mathbf{M}_t) \le s$. 
    Hence, it follows that 
    \begin{align*}
        \mathbb{E}_{\mathbf{u}_t}\big[\|\grad_t\|^2\big] & = \frac{\mu^2 L^2}{2} \mathbb{E}\big[\|\mathbf{M}_t \mathbf{u}_t\|^6\big] + 2\mathbb{E}\Big[\big\|\langle \nabla\loss_{\batch_t}(\bm\theta_t), \mathbf{M}_t \mathbf{u}_t\rangle \mathbf{M}_t\mathbf{u}_t\big\|^2\Big] \\
        & \le \frac{15 \mu^2 L^2 s^3}{2} + 2(s + 2) \|\nabla\loss_{\batch_t}(\bm\theta_t)\|^2
    \end{align*}
    where we use Lemma \ref{lemma:gaussian_norm} and Lemma \ref{lemma:gaussian_inner_prod}. Lastly, taking the expectations with respect to all randomness, we obtain
    \begin{align*}
        \mathbb{E}\big[\|\grad_t\|^2\big] & \le \frac{15 \mu^2 L^2 s^3}{2} + 2(s + 2) \mathbb{E}_{\mathcal{D}}\big[\|\nabla\loss_{\batch_t}(\bm\theta_t)\|^2\big] \\
        & \le \frac{15\mu^2 L^2 s^3}{2} + 4(s + 2) \left(\|\nabla\loss(\bm\theta_t)\|^2 + \frac{\sigma^2}{B}\right),
    \end{align*}
    using the condition \ref{con:conv_gradient}.
\end{proof}

From now, we provide some auxiliary lemmas under $\mathrm{srank}(\mathbf{M}_t) \le s$ and Assumption \ref{assumption:intdim}.
\begin{lemma}\label{lemma:intrinsic_distance}
    The smoothed loss $\mathcal{L}_\mu$ under Assumption \ref{assumption:intdim} satisfies
    \begin{align*}
        \|\nabla\loss_{\mu}(\bm\theta_t) - \nabla\loss(\bm\theta_t)\| & \le \frac{\sqrt{3}\mu L r \sqrt{d}}{2} + \frac{\mu L (d+4) \sqrt{d}}{2} \exp(-\frac{c_1}{2}d), \\
        \|\nabla\loss(\bm\theta_t)\|^2 & \le 3\mu^2 L^2 r^2 d + \mu^2 L^2 d(d+4)^2 \exp(-c_1 d) + 2\|\nabla\loss_\mu(\bm\theta_t)\|^2. 
    \end{align*}
\end{lemma}

\begin{proof}
    Let $\mathcal{E}$ be the event such that $\|\bm\theta_t + \mu \mathbf{u} - \bm\theta_t\| \le 2\alpha sG + 2\mu\sqrt{d}$. On the event $\mathcal{E}$, by Assumption \ref{assumption:intdim}, there exists some matrix $\mathbf{H}(\bm\theta_t)$ such that $\nabla^2 \loss(\bm\theta_t + \mu \mathbf{u}) \preceq \mathbf{H}(\bm\theta_t)$ with $\mathrm{intdim}(\mathbf{H}(\bm\theta_t)) \le r$. Thus, we have the following inequality
    \begin{align*}
        \loss(\bm\theta_t + \mu \mathbf{u}) & \le \loss(\bm\theta_t) + \mu \nabla \loss(\bm\theta_t)^\top \mathbf{u} + \frac{\mu^2}{2} \mathbf{u}^\top \mathbf{H}(\bm\theta_t) \mathbf{u} \\
        & = \loss(\bm\theta_t) + \mu \nabla \loss(\bm\theta_t)^\top \mathbf{u} + \frac{\mu^2}{2} \Big\langle \mathbf{H}(\bm\theta_t), \mathbf{u}\mathbf{u}^\top \Big\rangle 
    \end{align*}
    Likewise, on the event $\mathcal{E}^c$, it should hold that 
    \begin{align*}
        \loss(\bm\theta_t + \mu \mathbf{u}) \le \loss(\bm\theta_t) + \mu \nabla\loss(\bm\theta_t)^\top \mathbf{u} + \frac{\mu^2 L}{2} \|\mathbf{u}\|^2
    \end{align*}
    By the concentration inequality for (sub-)Gaussian random vector, we have $\mathbb{P}[\mathcal{E}^c] \le \exp(-c_1 d)$ for some universal constant $c_1 > 0$.
    \begin{align*}
        \|\nabla\loss_\mu(\bm\theta_t) - \nabla \loss(\bm\theta_t)\| & = \|\nabla \mathbb{E}\big[\loss(\bm\theta_t + \mu \mathbf{u})\big] - \nabla \loss(\bm\theta_t)\| \\
        & \le \Big\|\mathbb{E}_\mathbf{u}\Big[\frac{\loss(\bm\theta_t + \mu \mathbf{u})-\loss(\bm\theta_t)}{\mu} \mathbf{u}\Big] - \nabla\loss(\bm\theta_t)\Big\| \\
        & = \Big\|\mathbb{E}_{\mathbf{u}}\Big[\frac{\loss(\bm\theta_t + \mu \mathbf{u}) - \loss(\bm\theta_t) - \mu \langle \nabla\loss(\bm\theta_t), \mathbf{u}\rangle }{\mu} \mathbf{u}\Big]\Big\| \\
        & \le \mathbb{E}_{\mathbf{u}}\Big[\Big(\frac{\loss(\bm\theta_t + \mu \mathbf{u}) - \loss(\bm\theta_t) - \mu \langle \nabla\loss(\bm\theta_t), \mathbf{u}\rangle}{\mu}\Big)^2\Big]^{1/2} \mathbb{E}\big[\|\mathbf{u}\|^2\big]^{1/2}
    \end{align*}
    The first term could be bounded as 
    \begin{align*}
        \mathbb{E}_{\mathbf{u}}\Big[\Big(\frac{\loss(\bm\theta_t + \mu \mathbf{u}) - \loss(\bm\theta_t) - \mu \langle \nabla\loss(\bm\theta_t), \mathbf{u}\rangle}{\mu}\Big)^2\Big] & \le \frac{\mu^2}{4}\mathbb{E}_{\mathbf{u}}\Big[\big(\mathbf{u}^\top \mathbf{H}(\bm\theta_t) \mathbf{u}\big)^2\Big] + \exp(-c_1 d) \frac{\mu^2L^2}{4}\mathbb{E}_{\mathbf{u}}\big[\|\mathbf{u}\|^4\big] \\
        & \le \frac{3\mu^2 L^2r^2}{4} + \frac{\mu^2 L^2 (d+4)^2}{4} \exp(-c_1 d)
    \end{align*}
    where we use Lemma \ref{lemma:gaussian_norm} and Lemma \ref{lemma:nth_moments} as 
    \begin{align*}
        \mathbb{E}\Big[\big(\mathbf{u}^\top \mathbf{H}(\bm\theta_t) \mathbf{u}\big)^2\Big] & \le 3\mathrm{Tr}\big(\mathbf{H}(\bm\theta_t)\big)^2 \le 3L^2r^2 \\
        \mathbb{E}\big[\|\mathbf{u}\|^4\big] & \le (d+4)^2
    \end{align*}
    Therefore, it follows that 
    \begin{align*}
        \|\nabla\loss_\mu(\bm\theta_t) - \nabla\loss(\bm\theta_t)\| & \le \Big(\frac{\sqrt{3}\mu L r}{2} + \frac{\mu L (d+4)}{2} \exp(-\frac{c_1}{2}d)\Big)\sqrt{d} \\
        & = \frac{\sqrt{3}\mu L r \sqrt{d}}{2} + \frac{\mu L (d+4) \sqrt{d}}{2} \exp(-\frac{c_1}{2}d)
    \end{align*}
    Hence, we finally obtain 
    \begin{align*}
        \|\nabla\loss(\bm\theta_t)\|^2 & \le 2\|\nabla\loss_\mu(\bm\theta_t) - \nabla\loss(\bm\theta_t)\|^2 + 2\|\nabla\loss_\mu(\bm\theta_t)\|^2 \\
        & \le 3\mu^2 L^2 r^2 d + \mu^2 L^2 d(d+4)^2 \exp(-c_1 d) + 2|\nabla\loss_\mu(\bm\theta_t)\|^2  
    \end{align*}
\end{proof}
\begin{lemma}\label{lemma:descent}
    The following inequality holds
    \begin{align*}
        \loss_\mu(\bm\theta_{t+1}) &\le \loss_\mu(\bm\theta_t) - \alpha \Big\langle \nabla \loss_\mu(\bm\theta_t), \nabla \loss_{\mu, \mathbf{M}_t}(\bm\theta_t) \Big\rangle + \frac{\alpha^2}{2} \Big\langle \mathbb{E}\big[\grad_t \grad_t^\top\big], \mathbf{H}(\bm\theta_t) \Big\rangle \\
        & \qquad \qquad + \frac{\alpha^2}{2} \Big\langle \mathbb{E}\Big[\grad_t \grad_t^\top \!\cdot\! \mathbb{I}(\mathcal{E}^c)\Big], L\!\cdot\!\mathbf{I}_d - \mathbf{H}(\bm\theta_t) \Big\rangle,
    \end{align*}
    and the last term can be bounded by 
    \begin{align*}
        \frac{\alpha^2}{2} \Big\langle \mathbb{E}\Big[\grad_t \grad_t^\top \!\cdot\! \mathbb{I}(\mathcal{E}^c)\Big], L\!\cdot\! \mathbf{I}_d - \mathbf{H}(\bm\theta_t) \Big\rangle \le \alpha^2 L \Big(37\mu^2 L^2 s^3 + 42s\Big(\|\nabla\loss(\bm\theta_t)\|^2 + \frac{\sigma^2}{B}\Big)\Big)\exp\Big(-\frac{s}{2}\Big)
    \end{align*}
\end{lemma}

\begin{proof}
    Let $\mathcal{E}$ be the event that $\|\bm\theta_{t+1} - \bm\theta_t\| \le 2\alpha s G + 2\mu \sqrt{d} $. On $\mathcal{E}$, we have
    \begin{align*}
        \loss(\bm\theta_{t+1} + \mu \mathbf{u}) & \le \loss(\bm\theta_t + \mu \mathbf{u}) - \alpha\langle \nabla\loss(\bm\theta_t + \mu \mathbf{u}), \grad_t\rangle + \frac{1}{2} \Big\langle \grad_t \grad_t^\top, \mathbf{H}(\bm\theta_t) \Big\rangle
    \end{align*}
    On $\mathcal{E}^c$, we have for $L$-smooth loss,
    \begin{align*}
        \loss(\bm\theta_{t+1} + \mu \mathbf{u}) & \le \loss(\bm\theta_t + \mu \mathbf{u}) - \alpha \langle \nabla\loss(\bm\theta_t + \mu \mathbf{u}), \grad_t \rangle + \frac{L}{2} \|\grad_t\|^2 
    \end{align*}
    Combining two inequalities along the expectations with respect to $\mathbf{u}$ and $\mathbf{u}_t$, we have
    \begin{align*}
        \loss_\mu(\bm\theta_{t+1}) & \le \loss_\mu(\bm\theta_t) - \alpha \Big\langle \nabla\loss_\mu(\bm\theta_t), \nabla\loss_{\mu, \mathbf{M}_t}(\bm\theta_t) \Big\rangle + \frac{\alpha^2}{2}\Big\langle \mathbb{E}\big[\grad_t \grad_t^\top \!\cdot\! \mathbb{I}(\mathcal{E})\big], \mathbf{H}(\bm\theta_t) \Big\rangle \\
        & \qquad + \frac{\alpha^2 L}{2} \mathbb{E}\big[\|\grad_t\|^2 \!\cdot\! \mathbb{I}(\mathcal{E}^c)\big] \\
        & = \loss_\mu(\bm\theta_t) - \alpha \Big\langle \nabla \loss_\mu(\bm\theta_t), \nabla \loss_{\mu, \mathbf{M}_t}(\bm\theta_t) \Big\rangle + \frac{\alpha^2}{2} \Big\langle \mathbb{E}\big[\grad_t \grad_t^\top\big], \mathbf{H}(\bm\theta_t) \Big\rangle \\
        & \qquad + \frac{\alpha^2}{2} \Big\langle \mathbb{E}\Big[\grad_t \grad_t^\top \!\cdot\! \mathbb{I}(\mathcal{E}^c)\Big], L\!\cdot\!\mathbf{I}_d - \mathbf{H}(\bm\theta_t) \Big\rangle 
    \end{align*}
    The last term could be bounded as
    \begin{align*}
        \frac{\alpha^2}{2} \Big\langle \mathbb{E}\Big[\grad_t \grad_t^\top \cdot \mathbb{I}(\mathcal{E}^c)\Big], L\!\cdot\! \mathbf{I}_d - \mathbf{H}(\bm\theta_t) \Big\rangle & \le \alpha^2 L \mathbb{E}\Big[\|\grad_t\|^2 \cdot \mathbb{I}(\mathcal{E}^c)\Big] \\
        & = \alpha^2 L \mathbb{E}\Big[\|\grad_t\|^4\Big]^{1/2} \mathbb{P}[\mathcal{E}^c]^{1/2}.
    \end{align*}
    The first term could be computed as
    \begin{align*}
        \|\grad_t\|^4 & \le \left\|\frac{\loss_{\batch_t}(\bm\theta_t + \mu \mathbf{M}_t\mathbf{u}_t) - \loss_{\batch_t}(\bm\theta_t - \mu \mathbf{M}_t\mathbf{u}_t)}{2\mu} \mathbf{M}_t\mathbf{u}_t\right\|^4 \\
        & \le \frac{2}{\mu^4} \left\|\Big(\loss_{\batch_t}(\bm\theta_t + \mu \mathbf{M}_t \mathbf{u}_t) - \loss_{\batch_t}(\bm\theta_t) - \mu \langle \nabla\loss_{\batch_t}(\bm\theta_t), \mathbf{M}_t \mathbf{u}_t\rangle\Big) \mathbf{M}_t\mathbf{u}_t\right\|^4 + 4\big\|\langle \nabla\loss_{\batch_t}(\bm\theta_t), \mathbf{M}_t \mathbf{u}_t\rangle \mathbf{M}_t\mathbf{u}_t\big\|^4 \\
        & \le \frac{\mu^4 L^4}{8} \|\mathbf{M}_t\mathbf{u}_t\|^{12} + 4\Big\|\langle \nabla\loss_{\batch_t}(\bm\theta_t), \mathbf{M}_t\mathbf{u}_t\rangle \mathbf{M}_t\mathbf{u}_t\Big\|^4 
    \end{align*}
    The expectation of each term could be bounded as
    \begin{align*}
        \mathbb{E}\big[\|\mathbf{M}_t\mathbf{u}_t\|^{12}\big] & \le 10395s^6 \\
        \mathbb{E}\big[\|\langle\nabla \loss_{\batch_t}(\bm\theta_t), \mathbf{M}_t\mathbf{u}_t\rangle\mathbf{M}_t\mathbf{u}_t\|^4\big] & \le 105 s^2 \|\nabla\loss_{\batch_t}(\bm\theta_t)\|^4
    \end{align*}
    Thus, we have
    \begin{align*}
        \mathbb{E}\big[\|\grad_t\|^4\big] & \le \frac{10395\mu^4 L^4 s^6}{8} + 420s^2\|\nabla\loss_{\batch_t}(\bm\theta_t)\|^4 \\
        \mathbb{E}\big[\|\grad_t\|^4\big]^{1/2} & \le 37\mu^2 L^2 s^3 + 21s\|\nabla\loss_{\batch_t}(\bm\theta_t)\|^2 \\
        & \le 37\mu^2 L^2 s^3 + 42s(\|\nabla\loss(\bm\theta_t)\|^2 + \frac{\sigma^2}{B}) 
    \end{align*}
    
    Now, we compute the probability of the event $\mathcal{E}^c$. Note that on the event $\mathcal{E}^c$, we have
    \begin{align*}
        2\alpha s G + \mu \sqrt{d} \le \|\bm\theta_{t+1} - \bm\theta_t\| & = \alpha \|\grad_t\| \le \alpha G \|\mathbf{M}_t\mathbf{u}_t\|^2 
    \end{align*}
    Therefore, we should compute the probability such that
    \begin{align*}
        \|\mathbf{M}_t\mathbf{u}_t\|^2 \ge 2s + \frac{\mu \sqrt{d}}{\alpha G} 
    \end{align*}
    Using Hanson-Wright inequality~\citep{hanson1971bound,wright1973bound}, it is easy to check that $\mathbb{P}[\mathcal{E}^c] \le \exp(-s/2)$ holds. Combining all the items, we finally have
    \begin{align*}
        \frac{\alpha^2}{2} \Big\langle \mathbb{E}\Big[\grad_t \grad_t^\top \!\cdot\! \mathbb{I}(\mathcal{E}^c)\Big], L\!\cdot\! \mathbf{I}_d - \mathbf{H}(\bm\theta_t) \Big\rangle & \le \alpha^2 L \Big(37\mu^2 L^2 s^3 + 42s\Big(\|\nabla\loss(\bm\theta_t)\|^2 + \frac{\sigma^2}{B}\Big)\Big)\exp\Big(-\frac{s}{2}\Big)
    \end{align*}
\end{proof}

\begin{lemma}\label{lemma:second_order_term}
    The following inequality holds
    \begin{align*}
        \mathbb{E}\big[\grad_t^\top \mathbf{H}(\bm\theta_t) \grad_t\big] & \le 13\mu^2 L^3 s^2 \rho_t + 9L\rho_{\mathrm{max}} \Big(\|\nabla\loss(\bm\theta_t)\|^2 + \frac{\sigma^2}{B}\Big).
    \end{align*}
\end{lemma}

\begin{proof}
    By the definition of $\grad_t$, we have
    \begin{align*}
        \mathbb{E}\Big[\grad_t^\top \mathbf{H}(\bm\theta_t) \grad_t\Big] & = \frac{1}{4\mu^2}\mathbb{E}\Big[\Big(\loss_{\batch_t}(\bm\theta_t + \mu \mathbf{M}_t\mathbf{u}_t) - \loss_{\batch_t}(\bm\theta_t - \mu \mathbf{M}_t\mathbf{u}_t)\Big)^2 \mathbf{u}_t^\top \mathbf{M}_t\mathbf{H}(\bm\theta_t)\mathbf{M}_t\mathbf{u}_t\Big] \\
        & \le \frac{1}{4\mu^2} \mathbb{E}\Big[\Big(\loss_{\batch_t}(\bm\theta_t + \mu \mathbf{M}_t\mathbf{u}_t) - \loss_{\batch_t}(\bm\theta_t - \mu \mathbf{M}_t\mathbf{u}_t)\Big)^4\Big]^{1/2} \mathbb{E}\Big[\big(\mathbf{u}_t^\top \mathbf{M}_t\mathbf{H}(\bm\theta_t)\mathbf{M}_t\mathbf{u}_t\big)^2\Big]^{1/2} \\
        & \le \frac{\sqrt{3}L\rho}{4\mu^2} \mathbb{E}\Big[\Big(\loss_{\batch_t}(\bm\theta_t + \mu \mathbf{M}_t\mathbf{u}_t) - \loss_{\batch_t}(\bm\theta_t - \mu \mathbf{M}_t\mathbf{u}_t)\Big)^4\Big]^{1/2}
    \end{align*}
    The function difference could be handled as
    \begin{align*}
        \loss_{\batch_t}(\bm\theta_t + \mu \mathbf{M}_t\mathbf{u}_t) - \loss_{\batch_t}(\bm\theta_t - \mu \mathbf{M}_t\mathbf{u}_t) & \le \mu^2 L \|\mathbf{M}_t\mathbf{u}_t\|^2 + 2\mu \big\langle \nabla\loss_{\batch_t}(\bm\theta_t), \mathbf{M}_t\mathbf{u}_t \big\rangle 
    \end{align*}
    Therefore, the fourth exponent of this difference can be bounded by
    \begin{align*}
        \mathbb{E}\Big[\Big(\loss_{\batch_t}(\bm\theta_t + \mu \mathbf{M}_t\mathbf{u}_t) - \loss_{\batch_t}(\bm\theta_t - \mu \mathbf{M}_t\mathbf{u}_t)\Big)^4\Big] & \le \mathbb{E}\Big[\Big(\mu^2 L \|\mathbf{M}_t\mathbf{u}_t\|^2 + 2\mu \big\langle \nabla\loss_{\batch_t}(\bm\theta_t), \mathbf{M}_t\mathbf{u}_t \big\rangle \Big)^4\Big] \\
        & \le 8 \mathbb{E}\Big[\mu^8 L^4 \|\mathbf{M}_t\mathbf{u}_t\|^8 + 16\mu^4 \big\langle \nabla\loss_{\batch_t}(\bm\theta_t), \mathbf{M}_t\mathbf{u}_t \big\rangle^4\Big] \\
        & \le 8\mu^8 L^4 \mathbb{E}\big[\|\mathbf{M}_t\mathbf{u}_t\|^8\big] + 128\mu^4 \mathbb{E}\big[\langle \nabla\loss_{\batch_t}(\bm\theta_t), \mathbf{M}_t\mathbf{u}_t\rangle^4\big] \\
        & \le 840 \mu^8 L^4 s^4 + 384\mu^4 \|\nabla\loss_{\batch_t}(\bm\theta_t)\|^4
    \end{align*}
    Also, by the definition of subspace alignment in Definition \ref{def:subspace_alignment}, 
    we have
    \begin{align*}
        \mathrm{Tr}(\mathbf{M}_t \mathbf{H}(\bm\theta_t)\mathbf{M}_t) = L \rho_t
    \end{align*}
    Hence, we obtain
    \begin{align*}
        \mathbb{E}\big[\grad_t^\top \mathbf{H}(\bm\theta_t) \grad_t\big] & \le \frac{\sqrt{3}L\rho_t}{4\mu^2} \Big(\sqrt{840}\mu^4 L^2 s^2 + \sqrt{384}\mu^2 \|\nabla\loss_{\batch_t}(\bm\theta_t)\|^2\Big) \\
        & \le 13\mu^2 L^3 s^2 \rho_t + 9L\rho_{\mathrm{max}} \Big(\|\nabla\loss(\bm\theta_t)\|^2 + \frac{\sigma^2}{B}\Big)
    \end{align*}
\end{proof}

\subsection{Proof of Theorem \ref{thm:improved_convergence_zosgd}}

\begin{proof}
    By Lemma \ref{lemma:descent} and Lemma \ref{lemma:second_order_term}, it follows that
    \begin{align*}
        \loss_\mu(\bm\theta_{t+1}) & \le \loss_\mu(\bm\theta_t) - \frac{\alpha}{2} \|\nabla\loss_\mu(\bm\theta_t)\|^2 + \alpha \mu^2 L^2 (d+s) + \frac{\alpha^2}{2} \Big(13\mu^2 L^3 s^2 \rho_t + 9L\rho_{\mathrm{max}} \Big(\|\nabla\loss(\bm\theta_t)\|^2 + \frac{\sigma^2}{B}\Big)\Big) \\
        & \qquad + \alpha^2 L \Big(37\mu^2 L^2 s^3 + 42s(\|\nabla\loss(\bm\theta_t)\|^2 + \frac{\sigma^2}{B}) \Big)\exp(-s/2)
    \end{align*}
    By Lemma \ref{lemma:intrinsic_distance}, the gradient norm w.r.t. to the smoothed loss can be bounded by the gradient norm w.r.t. the original loss $\loss$ as
    \begin{align*}
        \frac{\alpha}{4}\|\nabla\loss(\bm\theta_t)\|^2 & \le \frac{3\alpha \mu^2 L^2 r^2d}{16} + \frac{\alpha}{2}\|\nabla\loss_\mu(\bm\theta_t)\|^2 \\
        & \le \frac{3\alpha \mu^2 L^2 r^2d}{16} + \loss_\mu(\bm\theta_t) - \loss_\mu(\bm\theta_{t+1}) + \alpha \mu^2 L^2 (d+s) + \frac{13\alpha^2 \mu^2 L^3 s^2 \rho_t }{2} + \frac{9\alpha^2 L \rho_{\mathrm{max}}}{2} \|\nabla\loss(\bm\theta_t)\|^2 \\
        & \qquad + \frac{9\alpha^2 L \rho_{\mathrm{max}} \sigma^2}{2B} + 37\alpha^2 \mu^2 L^3 s^3 \exp\big(-s/2\big) + 42\alpha^2 L s \exp\big(-s/2\big)\|\nabla\loss(\bm\theta_t)\|^2 \\
        & \qquad + \frac{42\alpha^2 Ls\sigma^2}{B} \exp\big(-s/2\big)
    \end{align*}
    By rearranging inequality, it should hold that
    \begin{align*}
        \Big(\frac{\alpha}{4} - \frac{9\alpha^2 L \rho_{\mathrm{max}}}{2} - 42\alpha^2 Ls \exp(-s/2)\Big) & \sum\limits_{t=1}^{T} \|\nabla\loss(\bm\theta_t)\|^2 \le \\
        &\frac{3\alpha \mu^2 L^2 r^2 d}{16} + \loss_\mu(\bm\theta_t) - \loss_\mu(\bm\theta_{t+1}) + \alpha \mu^2 L^2 (d+s) \\
        & \qquad + \frac{13\alpha^2 L^3 s^2 \rho_t}{2} + \frac{9\alpha^2 L\rho_{\mathrm{max}} \sigma^2}{2B} + 37\alpha^2 \mu^2 L^3 s^3 \exp(-s/2) \\
        & \qquad + \frac{42\alpha^2 Ls \sigma^2}{B} \exp(-s/2)
    \end{align*}
    Under the following parameter condition,
    \begin{align*}
        \alpha \le \frac{1}{36L\rho_{\mathrm{max}} + 336Ls\exp(-s/2)},
    \end{align*}
    \begin{align*}
        \frac{\alpha}{8} \|\nabla\loss(\bm\theta_t)\|^2 & \le \frac{3\alpha \mu^2 L^2 r^2 d}{16} + \loss_\mu(\bm\theta_t) - \loss_\mu(\bm\theta_{t+1}) + \alpha \mu^2 L^2 (d+s) + \frac{13\alpha^2 \mu^2 L^3 s^2 \rho_t}{2} + \frac{9\alpha^2 L \rho_{\mathrm{max}} \sigma^2}{2B} \\
        & \qquad + 37\alpha^2\mu^2L^3 s^3 \exp(-s/2) + \frac{42\alpha^2 L s \sigma^2}{B} \exp(-s/2)
    \end{align*}
    Therefore, we have
    \begin{align*}
        \|\nabla\loss(\bm\theta_t)\|^2 & \le \frac{3\mu^2 L^2 r^2 d}{2} + \frac{8(\loss_\mu(\bm\theta_t) - \loss_\mu(\bm\theta_{t+1}))}{\alpha} + 8\mu^2 L^2 (d+s) + 52\alpha \mu^2 L^3 s^2\rho_t + \frac{36\alpha L\rho_{\mathrm{max}} \sigma^2}{B} \\
        & \qquad + 296\alpha \mu^2 L^3 s^3 \exp(-s/2) + \frac{336\alpha Ls\sigma^2}{B}\exp(-s/2) \\
        & \le \frac{3\mu^2 L^2 r^2 d}{2} + \frac{8(\loss_\mu(\bm\theta_t) - \loss_\mu(\bm\theta_{t+1}))}{\alpha} + 8\mu^2 L^2 (d+s) + \frac{13\mu^2 L^2 s^2 \rho_t}{9} \\
        & \qquad + \frac{37\mu^2 L^2 s^2}{42} + \frac{\sigma^2}{B}
    \end{align*}
    Telescoping the inequality from $t = 1$ to $T$ with stepsize condition and $\mu = \mathcal{O}(\frac{1}{L\sqrt{\widebar{\rho} dT}})$, we obtain
    \begin{align*}
        \frac{1}{T} \sum\limits_{t=1}^{T} \|\nabla\loss(\bm\theta_t)\|^2 & \le \frac{3r^2}{2\widebar{\rho} T} +  \frac{8\Delta}{\alpha T} + \frac{8(d+s)}{\widebar{\rho} dT} + \frac{13s^2}{9 dT} + \frac{37s^2}{42\widebar{\rho} dT} + \frac{\sigma^2}{B} \\
        & \le \mathcal{O}\Big(\frac{1}{\widebar{\rho} T}\Big(r^2 + \frac{s^2}{d} + 1\Big)+ \frac{s^2}{dT} + \frac{\Delta}{\alpha T} + \frac{\sigma^2}{B}\Big) \\
        & = \mathcal{O}\Big(\frac{r^2}{\widebar{\rho}T} + \frac{s^2}{dT} + \frac{\Delta}{\alpha T} + \frac{\sigma^2}{B}\Big)
    \end{align*}
\end{proof}

\subsection{Proof of Proposition~\ref{prop:expected_rho}}
\begin{proof}
    For simplicity, we drop the timestep of $\mathbf{M}_t$ and $\mathbf{H}(\bm\theta_t)$ and denote them by simply $\mathbf{M}$ and $\mathbf{H}$ respectively.
    Since it follows that $\mathbb{E}[\mathbf{M}]=\frac{s}{d}\mathbf{I}_d (\neq \mathbf{0})$ for some $s>0$, we have $\mathbf{M}\neq\mathbf{0}$ almost surely. By using $\mathrm{Tr}(\mathbf{A}\mathbf{B}\mathbf{C})=\mathrm{Tr}(\mathbf{B}\mathbf{C}\mathbf{A})$, we have $\mathrm{Tr}(\mathbf{M}\mathbf{H}\mathbf{M})=\mathrm{Tr}(\mathbf{H}\mathbf{M}\mathbf{M})$. As $\mathbf{M}^2=\mathbf{M}$ holds a.s., we have $\mathrm{Tr}(\mathbf{H}\mathbf{M}\mathbf{M})=\mathrm{Tr}(\mathbf{H}\mathbf{M})$ a.s.. Then, it holds that $\rho=\frac{\mathrm{Tr}(\mathbf{M}\mathbf{H}\mathbf{M})}{\lambda_{\mathrm{max}}(\mathbf{H})}=\frac{\mathrm{Tr}(\mathbf{H}\mathbf{M})}{\lambda_{\mathrm{max}}(\mathbf{H})}$. By taking the expectation, we obtain
    \begin{align*}
        \mathbb{E}[\rho] & = \mathbb{E}\left[\frac{\mathrm{Tr}(\mathbf{H}\mathbf{M})}{\lambda_{\mathrm{max}}(\mathbf{H})}\right]\\
        & = \frac{1}{\lambda_{\mathrm{max}}(\mathbf{H})} \mathbb{E}\left[\mathrm{Tr}(\mathbf{H}\mathbf{M})\right] \quad\text{(since }\mathbf{H}\text{ is deterministic)}\\
        & = \frac{1}{\lambda_\mathrm{max}(\mathbf{H})}\mathrm{Tr}(\mathbb{E}\left[\mathbf{H}\mathbf{M}\right])\quad\text{(by the linearity of expectation)}\\
        & = \frac{1}{\lambda_\mathrm{max}(\mathbf{H})}\mathrm{Tr}(\mathbf{H}\mathbb{E}\left[\mathbf{M}\right])\quad\text{(again, since }\mathbf{H}\text{ is deterministic)} \\
        & =\frac{1}{\lambda_\mathrm{max}(\mathbf{H})}\mathrm{Tr}\left(\mathbf{H}\left(\frac{s}{d}\mathbf{I}_d\right)\right) =\frac{s}{d\lambda_\mathrm{max}(\mathbf{H})}\mathrm{Tr}(\mathbf{H}\mathbf{I}_d) = \frac{s\mathrm{Tr}(\mathbf{H})}{d\lambda_\mathrm{max}(\mathbf{H})}.
    \end{align*}
    This completes the proof. 
\end{proof}

\subsection{Proof that Definition~\ref{def:instantiation} satisfies the conditions in Proposition~\ref{prop:expected_rho}}
\begin{proof}
    We consider a case study.
    
    \textbf{(Case 1)} Low-rank Projection $\mathbf{M}=\mathbf{U}\mathbf{U}^\top$.
    
    First, we check the condition (i), orthogonality. The symmetry is confirmed by $\mathbf{M}^\top=(\mathbf{U}\mathbf{U}^\top)^\top=\mathbf{U}^\top(\mathbf{U}^\top)^\top=\mathbf{U}\mathbf{U}^\top$. The idempotency can be shown by $\mathbf{M}^2=(\mathbf{U}\mathbf{U}^\top)(\mathbf{U}\mathbf{U}^\top)=\mathbf{U}(\mathbf{U}^\top\mathbf{U})\mathbf{U}^\top=\mathbf{U}\mathbf{I}_s\mathbf{U}^\top=\mathbf{U}\mathbf{U}^\top$. Since $\mathbf{M}$ is symmetric and idempotent, it is an orthogonal projection matrix, which holds with probability $1$ by the definition of $\mathbf{U}$.
    
     Next, let's check the condition (ii), $\mathbb{E}[\mathbf{M}]=\frac{s}{d}\mathbf{I}_d$. By the definition of $\mathbf{M}$, the $s$-dimensional subspace spanned by column vectors of $\mathbf{U}$ is uniformly sampled from the Grassmannian manifold $\mathbf{Gr}(s,\mathbb{R}^d)$. Since $\mathbb{E}[\mathbf{M}]$ is a $d\times d$ matrix and the distribution of $\mathbf{M}$ is invariant to an orthogonal transformation (i.e., for any orthogonal matrix $\mathbf{P}\in\mathbb{R}^{d\times d}$, $\mathbf{M}$ and $\mathbf{P}\mathbf{M}\mathbf{P}^\top$ have the same distribution), its expectation $\mathbb{E}[\mathbf{M}]$ is also invariant to an orthogonal transformation. That means, 
     \begin{align}\label{eqn:ortho_inv}
         \text{for any orthogonal matrix }\mathbf{P}\in\mathbb{R}^{d\times d}, \mathbf{P}\mathbb{E}[\mathbf{M}]\mathbf{P}^\top=\mathbb{E}[\mathbf{P}\mathbf{M}\mathbf{P}^\top]=\mathbb{E}[\mathbf{M}]
     \end{align} holds. Since the unique matrix $\mathbf{M}$ satisfying~\ref{eqn:ortho_inv} is a scalar multiplier of identity matrix $\mathbf{I}_d$, there exists a $c\in\mathbb{R}$ such that $\mathbb{E}[\mathbf{M}]=c\mathbf{I}_d$.

     To find such a $c$, let's take a trace: $\mathrm{Tr}(\mathbb{E}[\mathbf{M}])=\mathbb{E}(\mathrm{Tr}(\mathbf{M}))$. Since $\mathbf{M}$ is a projection onto $s$-dimensional subspace of $\mathbb{R}^d$, we have $\mathrm{Tr}(\mathbf{M})=\mathrm{rank}(\mathbf{M})=s$. Hence, $\mathbb{E}[\mathrm{Tr}(\mathbf{M})]=\mathbb{E}[s]=s$. Since $\mathrm{Tr}(c\mathbf{I}_d)=cd$, by letting them equal, we obtain $cd=s$, $c=s/d$. Therefore, we can conclude that $\mathbb{E}[\mathbf{M}]=c\mathbf{I}_d=\frac{s}{d}\mathbf{I}_d$.

     \textbf{(Case 2)} Sparse Perturbation $\mathbf{M}=\diag(\mathbf{m})$, where $m_i\sim\mathrm{Bernoulli}(p=s/d)$. 
     
     Since $m_i$ is either 0 or 1, $\mathbf{M}$ is a diagonal matrix with diagonal entries that are either 0 or 1. Therefore, $\mathbf{M}$ is symmetric and idempotent (with probability $1$), which confirms the condition (i). By the definition of $\mathbf{M}$, $\mathbb{E}[\mathbf{M}]$ is a diagonal matrix with diagonal entries that are $p$, i.e., $\mathbf{M}=p\mathbf{I}_d$. Since $p=s/d$, we have $\mathbb{E}[\mathbf{M}]=\frac{s}{d}\mathbf{I}_d$.

     \textbf{(Case 3)} Block Sparse Perturbation $\mathbf{M}=\diag(\mathbf{m})$, where $m_i=1$ if $i\in B_j, j\sim\mathcal{U}\{1,N\}, |B_j|=s$. 
     
     Similar to the sparse perturbation, here $\mathbf{M}$ is a diagonal matrix with diagonal entries that are either 0 or 1. Thus, $\mathbf{M}$ satisfies the condition (i). 
     
     Next, let's start by letting $J~\sim\mathcal{U}\{1,N\}$ be the random variable representing the block index. Each diagonal entry $\mathbf{M}_{kk}$ is $1$ if the $k$ is in the selected block and otherwise, 0. Since the blocks are the partition of $\{1,\ldots,d\}$, every index $k$ is contained in exactly one block $B_{parent(k)}$. Since $\mathbf{M}_{kk}=1$ is equivalent to $J=parent(k)$ and $J\sim\mathcal{U}\{1,N\}$, $\mathbb{P}(J=parent(k))=\frac{1}{N}$. Hence, $\mathbb{E}[\mathbf{M}]=\frac{1}{N}\mathbf{I}_d$. Since $N=s/d$, we finally obtain $\mathbb{E}[\mathbf{M}]=\frac{s}{d}\mathbf{I}_d$.
\end{proof}

\subsection{Proof of Proposition~\ref{prop:prob_rho}}
To be self-contained, we first introduce the results from~\citep{AST_1988__157-158__273_0}.
\begin{lemma}[Example 2.1 in~\citep{AST_1988__157-158__273_0}]\label{lemma:levy}
    Let $\mathbb{S}^n$ be the Euclidean sphere equipped with the geodesic distance $\varphi$ and the rotation-invariant probability measure $\nu$. Then
    \begin{align*}
        \alpha(\mathbb{S}^{n+1};\varepsilon)\leq\sqrt{\pi/8}\exp(-n\varepsilon^2/2)\rightarrow 0\text{ for } n\rightarrow \infty
    \end{align*}
    for any fixed $\varepsilon>0$.
\end{lemma}
Here, $\alpha(X;\epsilon)=1-\inf\left\{\nu(A_\epsilon):A\text{ be a Borel Subset of } X, \nu(A)\geq \frac{1}{2}\right\}$ where $A_\varepsilon$ is the geodesic $\varepsilon$-neighourhood of $A$ on $X$, i.e., $A_\epsilon=\{x\in X:\varphi(x,A)\leq \epsilon\}$.

Here's another supporting lemma:
\begin{lemma}[Lipschitz constant of $f(\mathbf{U})=\Tr(\mathbf {U}^\top\mathbf{H}\mathbf{U})$]\label{lemma:Lipschitz}
    Let $\mathbf{H}\in\mathbb{R}^{d\times d}$ be a symmetric positive semi-definite with maximum eigenvalue $\lambda_{\max}(\mathbf{H})>0$. For $1\leq s\leq d$, let $f:V_s(\mathbb{R}^d)\rightarrow \mathbb{R}, f(\mathbf{U})=\Tr(\mathbf{U}^\top\mathbf{H}\mathbf{U})$. Then $f$ is
    \begin{align*}
        L=2\sqrt{s}\lambda_{\max}(\mathbf{H})\text{-Lipschitz on } (V_s(\mathbb{R}^d), \|\cdot\|_F)
    \end{align*}
    i.e.,
    \begin{align*}
        |f(\mathbf{U})-f(\mathbf{V})|\leq2\sqrt{s}\lambda_{\max}(\mathbf{H})\|\mathbf{U}-\mathbf{V}\|_F, \quad \forall \mathbf{U}, \mathbf{V}\in V_s(\mathbb{R}^d).
    \end{align*}
\end{lemma}

\begin{proof}
    Let $\mathbf{P}=\mathbf{U}\mathbf{U}^\top$ and $\mathbf{Q}=\mathbf{V}\mathbf{V}^\top$. Since $|f(\mathbf{U})-f(\mathbf{V})|=|\Tr(\mathbf{H}(\mathbf{P}-\mathbf{Q}))|$ and $|\Tr(\mathbf{A}\mathbf{B})|\leq\|\mathbf{A}\|_{2\rightarrow2}\|B\|_*$ for any $\mathbf{A}, \mathbf{B}$, we have
    $$|f(\mathbf{U})-f(\mathbf{V})|\leq\lambda_{\max}(\mathbf{H})\|\mathbf{P}-\mathbf{Q}\|_*.$$ Then, since $\mathbf{P}-\mathbf{Q}$ has rank at most $2s$, we know $$\|\mathbf{P}-\mathbf{Q}\|_*\leq \sqrt{2s}\|\mathbf{P}-\mathbf{Q}\|_F.$$ (From the fact that $\|\mathbf{A}\|_*\leq\sqrt{\mathrm{rank}(\mathbf{A})}\|\mathbf{A}\|_F$.)

    By writing $\mathbf{P}-\mathbf{Q}=(\mathbf{U}-\mathbf{V})\mathbf{U}^\top+\mathbf{V}(\mathbf{U}-\mathbf{V})^\top$, then apply $\|\mathbf{A}\mathbf{B}\|_F\leq\|\mathbf{A}\|_{2\rightarrow 2}\|\mathbf{B}\|_F$ and $\|\mathbf{U}\|_{2\rightarrow2}=\|\mathbf{V}\|_{2\rightarrow 2}=1$, we get $$\|\mathbf{P}-\mathbf{Q}\|_F\leq 2\|\mathbf{U}-\mathbf{V}\|_F.$$

    Combining all the inequalities, we obtain $$|f(\mathbf{U})-f(\mathbf{V})|\leq\lambda_{\max}(\mathbf{H})\sqrt{2s}(2\|\mathbf{U}-\mathbf{V}\|_F)=2\sqrt{s}\lambda_{\max}\|\mathbf{U}-\mathbf{V}\|_F.$$ Hence, $L=2\sqrt{s}\lambda_{\max}(\mathbf{H})$.
\end{proof}

\begin{proof}
    \textbf{(Case 1) Low-rank Projection.} Let $V_s(\mathbb{R}^d)=\{\mathbf{U}\in\mathbb{R}^{d\times s}:\mathbf{U}^\top\mathbf{U}=\mathbf{I}_s$ be the Stiefel manifold. If $\mathbf{U}$ is drawn uniformly from $V_s(\mathbb{R}^d)$ (with respect to Haar measure), then $\mathbf{U}\mathbf{U}^\top$ be an orthogonal projection matrix onto an $s$-dimensional subspace chosen uniformly at random from the Grasmannian manifold $\mathbf{Gr}(s,\mathbb{R}^d)$.
    
    By Lemma~\ref{lemma:levy}, if a function $F$ on $\mathbb{S}^d$ is $L$-Lipschitz, then for every $\Delta>0$, we have
    \begin{align*}
        \mathbb{P}(|F-\tilde{F}|\geq\Delta)\leq 2\alpha(\mathbb{S}^n;\Delta/L)
    \end{align*}
    where $\tilde{F}$ is the median of $F$. Incorporate with the definition of $\alpha$, we obtain
    \begin{align}\label{eqn:levy_variant}
        \mathbb{P}(|F-\tilde{F}|\geq\Delta)\leq 2\sqrt{\pi/8}\exp\big(-\frac{n\Delta^2}{2L^2}\big).
    \end{align}
    
    Since the Stiefel manifold $V_s(\mathbb{R}^d)\subset\sqrt{s}\mathbb{S}^{ds-1}$ and the Lipschitz constant $L$ for the function $F(\mathbf{U})=\mathrm{Tr}(\mathbf{U}^\top\mathbf{H}\mathbf{U})$ is $2\sqrt{s}\lambda_{\max}(\mathbf{H})$ (Lemma~\ref{lemma:Lipschitz}), by (\ref{eqn:levy_variant}), we obtain
    \begin{align*}
        \mathbb{P}(\rho\geq\hat\rho)&=\mathbb{P}(\rho-\mathbb{E}[\rho]\geq\Delta)\leq 2\sqrt{\pi/8}\exp\big(-\frac{(ds-1)(\lambda_{\max}\Delta)^2}{4s\lambda_{\max}^2}\big)\\&\leq 2\exp\big(-\frac{d\Delta^2}{8s\lambda_{\max}^2}\big)
    \end{align*}
    where in the last inequality we used $\sqrt{\pi/8}\leq 1$ and $ds-1\geq d$. Hence, by letting $c_1=\frac{1}{8\lambda_{\max}^2}$, this completes the proof.
    
    \textbf{(Case 2) Sparse Perturbation.} Given that $p=s/d$ and independent random variables $m_i\sim\mathrm{Bernoulli}(p)$, let $\rho-\mathbb{E}[\rho]=\sum_{i=1}^dX_i$ where $X_i\coloneqq\frac{\mathbf{H}_{ii}}{\lambda_{\mathrm{max}}(\mathbf{H})}(m_i-p)$. Then, by the definition of $X_i$, we have $\mathbb{E}[X_i]=0$ and $|X_i|\leq 1$ since $\mathbf{H}_{ii}<\lambda_{\mathrm{max}}(\mathbf{H})$. Also, we have
    \begin{align*}
        \mathrm{Var}(X_i)=p(1-p)\big(\frac{\mathbf{H}_{ii}}{\lambda_{\mathrm{max}}(\mathbf{H})}\big)^2, \sigma^2\coloneqq\mathrm{Var}\big(\sum_{i=1}^dX_i\big)=p(1-p)\sum_{i=1}^d\big(\frac{\mathbf{H}_{ii}}{\lambda_{\mathrm{max}}(\mathbf{H})}\big)^2.
    \end{align*}

    Since $\sum_{i=1}^d\mathbf{H}_{ii}^2\leq(\max_i \mathbf{H}_{ii})\sum_{i=1}^d\mathbf{H}_{ii}\leq\lambda_{\mathrm{max}}(\mathbf{H})\mathrm{Tr}(\mathbf{H})$, we have $$\sum_{i=1}^d\big(\frac{\mathbf{H}_{ii}}{\lambda_{\mathrm{max}}(\mathbf{H})}\big)^2\leq\frac{\lambda_{\mathrm{max}}(\mathbf{H})\mathrm{Tr}(\mathbf{H})}{\lambda_{\mathrm{max}}^2(\mathbf{H})}=\mathrm{intdim}(\mathbf{H})\leq r$$
    since $\mathrm{intdim(\mathbf{H})}\leq r$. Thus, $\sigma^2\leq p(1-p)r$.

    Then, by the Chebyshev's inequality, we obtain
    \begin{align*}
        \mathbb{P}(\rho\geq\hat\rho)&=\mathbb{P}(\rho-\mathbb{E}[\rho]\geq\Delta)=\mathbb{P}(\sum_{i=1}^d X_i\geq\Delta)\leq\mathbb{P}(|\sum_{i=1}^dX_i|\geq \Delta)\leq \frac{\sigma^2}{\Delta^2}\leq c_2\frac{sr}{d\Delta^2}
    \end{align*}
    where $c_2=1-p<1$, which completes the proof.
    
    \textbf{(Case 3) Block Sparse Perturbation.} Given that $\{1,\ldots,d\}$ is partitioned into $N=d/s$ disjoint blocks $B_1,\ldots,B_N$ of size $s$, choose a block index $J\sim\mathcal{U}\{1,\ldots,N\}$, we have
    \begin{align*}
        \rho=\frac{1}{\lambda_{\mathrm{max}}(\mathbf{H})}\sum_{i\in B_J}\mathbf{H}_{ii}
    \end{align*} by the definition of block sparse perturbation $\mathbf{M}$. By letting $h(J)=\frac{1}{\lambda_{\mathrm{max}}(\mathbf{H})}\sum_{i\in B_J}\mathbf{H}_{ii}$, we have
    \begin{align*}
        \mathbb{P}(\rho\geq\hat\rho)&=\mathbb{P}(h(J)\geq\hat\rho)=\sum_{j=1}^N\mathbb{P}(J=j)\mathbf{1}_{\{h(j)\geq\hat\rho\}}=\frac{1}{N}\sum_{j=1}^N\mathbf{1}_{\{h(j)\geq\hat\rho\}}\\&=\frac{1}{N}\big|\big\{j:\sum_{i\in B_j}\mathbf{H}_{ii}\geq\hat\rho\lambda_{\mathrm{max}}(\mathbf{H})\big\}\big|.
    \end{align*}
    since $J~\sim\mathcal{U}\{1,\ldots,N\}$. This completes the proof.
    
\end{proof}
\clearpage
\section{Proof of Generalization Error Bound}

In this section, we consider the following online update.
\begin{align*}
    \widehat\nabla \loss(\bm\theta_t; \mathbf{z}_{i_t}) & = \frac{\loss(\bm\theta_t + \mu \mathbf{M}_t \mathbf{u}_t; \mathbf{z}_{i_t}) - \loss(\bm\theta_t - \mu \mathbf{M}_t\mathbf{u}_t; \mathbf{z}_{i_t})}{2\mu} \mathbf{M}_t\mathbf{u}_t\\
    \bm\theta_{t+1} & = \bm\theta_t - \alpha_t \widehat\nabla \loss(\bm\theta_t; \mathbf{z}_{i_t})
\end{align*}

For generalization error bound, we employ the uniform stability framework. Note that we define some quantities
\begin{align*}
    \mathcal{E}_{t_0} & = \{\bm\theta_t = \bm\theta_{t_0}\}~ \text{(the event)} \\
    \delta_t & = \|\bm\theta_t - \bm\theta_t'\|, \quad \Delta_t = \mathbb{E}\big[\delta_t | \mathcal{E}_{t_0}\big]
\end{align*}

\subsection{Auxiliary Lemmas}

\begin{lemma}\label{lemma:diff_same_datapoint}
    For any $t > t_0$, the distance between the zeroth-order gradients, evaluated on the same datapoint $\mathbf{z}_{i_t}$ but on the different parameters $\bm\theta_t$ and $\bm\theta_t'$, is bounded as 
    \begin{align*}
        \mathbb{E}\Big[\Big\|\widehat\nabla \loss(\bm\theta_t; \mathbf{z}_{i_t}) - \widehat\nabla \loss(\bm\theta_t'; \mathbf{z}_{i_t})\Big\|\Big] & \le \sqrt{3s}\big(4L\rho + \mu Ls \exp(-s/2)\big) + L\sqrt{s} \delta_t
    \end{align*}
\end{lemma}

\begin{proof}
    For simplicity, we denote $\loss(\bm\theta_t; \mathbf{z}_{i_t}) \eqqcolon \loss_{i_t}(\bm\theta_t)$. By the definition of zeroth-order gradient estimator, we have
    \begin{align*}
        \widehat\nabla \loss_{i_t}(\bm\theta_t) & = \frac{\loss_{i_t}(\bm\theta_t + \mu \mathbf{M}_t \mathbf{u}_t) - \loss_{i_t}(\bm\theta_t - \mu \mathbf{M}_t\mathbf{u}_t)}{2\mu} \mathbf{M}_t\mathbf{u}_t \\
        & = \frac{\loss_{i_t}(\bm\theta_t + \mu \mathbf{M}_t\mathbf{u}_t) - \loss_{i_t}(\bm\theta_t) - \mu\langle \nabla\loss_{i_t}(\bm\theta_t), \mathbf{M}_t\mathbf{u}_t\rangle}{2\mu} \mathbf{M}_t\mathbf{u}_t \\
        & \qquad - \frac{\loss_{i_t}(\bm\theta_t - \mu \mathbf{M}_t\mathbf{u}_t) - \loss_{i_t}(\bm\theta_t) - \mu\langle \nabla\loss_{i_t}(\bm\theta_t), -\mathbf{M}_t\mathbf{u}_t\rangle}{2\mu} \mathbf{M}_t\mathbf{u}_t \\
        & \qquad + \langle \nabla\loss_{i_t}(\bm\theta_t), \mathbf{M}_t\mathbf{u}_t\rangle \mathbf{M}_t\mathbf{u}_t
    \end{align*}
    Therefore, the gradient difference is
    \begin{align*}
        \widehat\nabla\loss_{i_t}(\bm\theta_t) - \widehat\nabla\loss_{i_t}(\bm\theta_t') & = \frac{\loss_{i_t}(\bm\theta_t + \mu \mathbf{M}_t\mathbf{u}_t) - \loss_{i_t}(\bm\theta_t) - \mu\langle \nabla\loss_{i_t}(\bm\theta_t), \mathbf{M}_t\mathbf{u}_t\rangle}{2\mu} \mathbf{M}_t\mathbf{u}_t \\
        & \qquad - \frac{\loss_{i_t}(\bm\theta_t - \mu \mathbf{M}_t\mathbf{u}_t) - \loss_{i_t}(\bm\theta_t) - \mu\langle \nabla\loss_{i_t}(\bm\theta_t), -\mathbf{M}_t\mathbf{u}_t\rangle}{2\mu} \mathbf{M}_t\mathbf{u}_t \\
        & \qquad - \frac{\loss_{i_t}(\bm\theta_t' + \mu \mathbf{M}_t\mathbf{u}_t) - \loss_{i_t}(\bm\theta_t') - \mu\langle \nabla\loss_{i_t}(\bm\theta_t'), \mathbf{M}_t\mathbf{u}_t\rangle}{2\mu} \mathbf{M}_t\mathbf{u}_t \\
        & \qquad + \frac{\loss_{i_t}(\bm\theta_t' - \mu \mathbf{M}_t\mathbf{u}_t) - \loss_{i_t}(\bm\theta_t') - \mu\langle \nabla\loss_{i_t}(\bm\theta_t'), -\mathbf{M}_t\mathbf{u}_t\rangle}{2\mu} \mathbf{M}_t\mathbf{u}_t \\
        & \qquad + \langle \nabla\loss_{i_t}(\bm\theta_t) - \nabla \loss_{i_t}(\bm\theta_t'), \mathbf{M}_t\mathbf{u}_t \rangle \mathbf{M}_t\mathbf{u}_t
    \end{align*}
    Let $\mathcal{E}$ be the event that $\|\mu \mathbf{M}_t\mathbf{u}_t\| \le 2\alpha_t s G + 2\mu\sqrt{d}$. By Hanson-Wright inequality or Gaussian tail bound, it is easy to check $\mathbb{P}[\mathcal{E}^c] \le \exp(-c_3 d)$ for some universal constant $c_3 > 0$. On the event $\mathcal{E}$, we have
    \begin{align*}
        \Bigg|\frac{\loss_{i_t}(\bm\theta_t + \mu \mathbf{M}_t\mathbf{u}_t) - \loss_{i_t}(\bm\theta_t) - \mu\big\langle \nabla\loss_{i_t}(\bm\theta_t), \mathbf{M}_t\mathbf{u}_t\big\rangle}{2\mu}\Bigg| \le \frac{\mu}{4} \mathbf{u}_t^\top \mathbf{M}_t\mathbf{H}(\bm\theta_t) \mathbf{M}_t \mathbf{u}_t 
    \end{align*}

    On the event $\mathcal{E}^c$, with probability at most $\exp(-c_3 d)$, it follows that
    \begin{align*}
        \Bigg|\frac{\loss_{i_t}(\bm\theta_t + \mu \mathbf{M}_t\mathbf{u}_t) - \loss_{i_t}(\bm\theta_t) - \mu\big\langle \nabla\loss_{i_t}(\bm\theta_t), \mathbf{M}_t\mathbf{u}_t\big\rangle}{2\mu}\Bigg| \le \frac{\mu L}{4} \|\mathbf{M}_t\mathbf{u}_t\|^2
    \end{align*}
    Hence, we have (in fact, $c_3 = 1/4$ is enough)
    \begin{align*}
        \mathbb{E}\Bigg[\Bigg(\frac{\loss_{i_t}(\bm\theta_t + \mu \mathbf{M}_t \mathbf{u}_t) - \loss_{i_t}(\bm\theta_t) - \mu\big\langle \nabla\loss_{i_t}(\bm\theta_t), \mathbf{M}_t\mathbf{u}_t \big\rangle}{2\mu} \Bigg)^2\Bigg] & \le \frac{3\mu^2 L^2\rho_{\text{max}}^2}{16} + \frac{\mu^2 L^2 }{16}\mathbb{E}\big[\|\mathbf{M}_t \mathbf{u}_t\|^4\big] \exp(-s/4) \\ 
        & \le \frac{3\mu^2 L^2\rho_{\text{max}}^2}{16} + \frac{3\mu^2 L^2 s^2}{16} \exp(-s/4)
    \end{align*}
    Combining two inequalities, we obtain
    \begin{align*}
        & \mathbb{E}\Bigg[\Bigg\|\frac{\loss_{i_t}(\bm\theta_t + \mu \mathbf{M}_t\mathbf{u}_t) - \loss_{i_t}(\bm\theta_t) - \mu\langle \nabla\loss_{i_t}(\bm\theta_t), \mathbf{M}_t\mathbf{u}_t\rangle}{2\mu} \mathbf{M}_t\mathbf{u}_t\Bigg\|\Bigg] \\
        \le~ & \mathbb{E}\Bigg[\Bigg(\frac{\loss_{i_t}(\bm\theta_t + \mu \mathbf{M}_t\mathbf{u}_t) - \loss_{i_t}(\bm\theta_t) - \mu\langle \nabla\loss_{i_t}(\bm\theta_t), \mathbf{M}_t\mathbf{u}_t\rangle}{2\mu} \Bigg)^2\Bigg]^{1/2} \mathbb{E}\big[\|\mathbf{M}_t\mathbf{u}_t\|^2\big]^{1/2} \\
        \le~ & \Big(\frac{3\mu^2 L^2\rho_{\text{max}}^2}{16} + \frac{3\mu^2 L^2 s^2}{16}\exp(-s/4)\Big)^{1/2} s^{1/2} \\
        \le~ & \frac{\sqrt{3s}}{4} \Big(\mu L\rho_{\text{max}} + \mu Ls \exp(-s/2)\Big)
    \end{align*}
    Hence, the expected distance can be bounded as 
    \begin{align*}
        \mathbb{E}\Big[\|\widehat\nabla \loss_{i_t}(\bm\theta_t) - \widehat\nabla\loss_{i_t}(\bm\theta_t')\|\Big] & \le \sqrt{3s}\Big(\mu L\rho_{\text{max}} + \mu Ls \exp(-s/2)\Big) + \mathbb{E}\big[\|\langle \nabla\loss_{i_t}(\bm\theta_t) - \nabla\loss_{i_t}(\bm\theta_t'), \mathbf{M}_t \mathbf{u}_t\rangle \mathbf{M}\mathbf{u}_t\|] \\
        & \le \sqrt{3s}\Big(\mu L\rho_{\text{max}} + \mu Ls \exp(-s/2)\Big) + \|\nabla\loss_{i_t}(\bm\theta_t) - \nabla\loss_{i_t}(\bm\theta_t')\| \sqrt{s} \\
        & \le \sqrt{3s}\Big(\mu L\rho_{\text{max}} + \mu Ls \exp(-s/2)\Big) + L\sqrt{s}\|\bm\theta_t - \bm\theta_t'\| \\
        & = \sqrt{3s}\Big(\mu L\rho_{\text{max}} + \mu Ls \exp(-s/2)\Big) + L\sqrt{s} \delta_t
    \end{align*}
\end{proof}

\begin{lemma}
    The following recursive relation
    \begin{align*}
        \Delta_t & \le (1 + \alpha_t L \sqrt{s})\Delta_t + \alpha_t \sqrt{3s} (\mu L\rho_{\text{max}} + \mu L s\exp(-s/2)) + \alpha_t \frac{4G\sqrt{s}}{n}
    \end{align*}
\end{lemma}

\begin{proof}
    For any $t > t_0$, we have
    \begin{align*}
        \delta_{t+1} & = \big\|\bm\theta_t - \alpha_t \widehat\nabla\loss_{i_t}(\bm\theta_t) - \bm\theta_t' + \alpha_t \widehat\nabla\loss_{i_t'}(\bm\theta_t') \big\| \\
        & \le \|\bm\theta_t - \bm\theta_t'\| + \alpha_t \underbrace{\big\|\widehat\nabla\loss_{i_t}(\bm\theta_t) - \widehat\nabla\loss_{i_t'}(\bm\theta_t')\big\|}_{R_t} \\
        & = \delta_t  + \alpha_t R_t 
    \end{align*}
    We consider two cases for recursive relation for $R_t$ (in fact, for $\mathbb{E}[R_t]$).

    \textbf{Case 1: } With probability $1 - \dfrac{1}{n}$, we have $i_t = i_t'$ where the same datapoint is sampled at time $t$. By Lemma \ref{lemma:diff_same_datapoint}, we have
    \begin{align*}
        E[R_t] & = \mathbb{E}\big[\|\widehat\nabla\loss_{i_t}(\bm\theta_t) - \widehat\nabla\loss_{i_t}(\bm\theta_t')\|\big] \\
        & \le \sqrt{3s}\Big(\mu L\rho_{\text{max}} + \mu Ls \exp(-s/2)\Big) + L\sqrt{s} \Delta_t.
    \end{align*}
    
    \textbf{Case 2: } With probability $\dfrac{1}{n}$, we have $i_t \neq i_t'$. Therefore, we obtain
    \begin{align*}
        \mathbb{E}[R_t] & = \mathbb{E}\big[\|\widehat\nabla \loss_{i_t}(\bm\theta_t) - \widehat\nabla \loss_{i_t'}(\bm\theta_t') \|\big] \\
        & \le \mathbb{E}\big[\|\widehat\nabla \loss_{i_t}(\bm\theta_t)\|\big] + \mathbb{E}\big[\widehat\nabla\loss_{i_t'}(\bm\theta_t')\|\big] \\
        & \le \sqrt{3s} \Big(\mu L\rho_{\text{max}} + \mu Ls \exp(-s/2)\Big) + 4G\sqrt{s}
    \end{align*}
    From two cases, we can derive the recursive relation for $\Delta_t$ as 
    \begin{align*}
        \Delta_{t+1} & \le \Delta_t + \alpha_t\Big(1 - \frac{1}{n}\Big)\Big(\sqrt{3s}\Big(\mu L\rho_{\text{max}} + \mu Ls \exp(-s/2)\Big) + L\sqrt{s} \Delta_t\Big) \\
        & \qquad + \frac{\alpha_t}{n} \Big(\sqrt{3s} \Big(\mu L\rho_{\text{max}} + \mu Ls \exp(-s/2)\Big) + 4G\sqrt{s}\Big) \\
        & \le (1 + \alpha_t L \sqrt{s})\Delta_t + \alpha_t \sqrt{3s} (\mu L\rho_{\text{max}} + \mu L s\exp(-s/2)) + \alpha_t \frac{4G\sqrt{s}}{n}
    \end{align*}
\end{proof}

\subsection{Proof of Theorem \ref{thm:generalization_zosgd}}

\begin{proof}
    Under the following parameter condition
    \begin{align}\label{eqn:gen_param_cond}
        \mu \le \frac{1}{nL(\rho_{\text{max}} + s\exp(-s/2))}, \quad \alpha_t \le \frac{C}{t},
    \end{align}
    we have simpler recursive relation as 
    \begin{align*}
        \Delta_{t+1} & \le \Big(1 + \alpha_t L\sqrt{s}\Big)\Delta_t + \alpha_t\Big(\sqrt{3s}\mu L \rho_{\text{max}} + \sqrt{3s}\mu Ls\exp(-s/2) + \frac{4G\sqrt{s}}{n}\Big)
    \end{align*}
    Solving this inequality yields that
    \begin{align*}
        \Delta_T & \le \Big(\sqrt{3s}\mu L \rho_{\text{max}} + \sqrt{3s}\mu Ls\exp(-s/2) + \frac{4G\sqrt{s}}{n}\Big) \sum\limits_{t=t_0 + 1}^{T} \alpha_t \prod_{j=t+1}^{T} (1 + \alpha_j L\sqrt{s}) \\
        & \le C\Big(\sqrt{3s}\mu L \rho_{\text{max}} + \sqrt{3s}\mu Ls\exp(-s/2) + \frac{4G\sqrt{s}}{n}\Big) \sum\limits_{t=t_0+1}^{T} \frac{1}{t} \prod_{j=t+1}^{T} \Big(1 + \frac{CL\sqrt{s}}{j}\Big) \\
        & \le C\Big(\sqrt{3s}\mu L \rho_{\text{max}} + \sqrt{3s}\mu Ls\exp(-s/2) + \frac{4G\sqrt{s}}{n}\Big) \sum\limits_{t=t_0+1}^{T} \frac{1}{t} \prod_{j=t+1}^{T} \exp\Big(\frac{CL\sqrt{s}}{j}\Big) \\
        & \le C\Big(\sqrt{3s}\mu L \rho_{\text{max}} + \sqrt{3s}\mu Ls\exp(-s/2) + \frac{4G\sqrt{s}}{n}\Big) \sum\limits_{t=t_0+1}^{T} \frac{1}{t} \exp\Big(CL\sqrt{s}\sum\limits_{j=t+1}^{T} \frac{1}{j}\Big) \\
        & \le C\Big(\sqrt{3s}\mu L \rho_{\text{max}} + \sqrt{3s}\mu Ls\exp(-s/2) + \frac{4G\sqrt{s}}{n}\Big) \sum\limits_{t=t_0+1}^{T} \frac{1}{t} \exp\Big(CL\sqrt{s}\big(\log T - \log t\big)\Big) \\
        & \le C T^{CL\sqrt{s}} \Big(\sqrt{3s}\mu L \rho_{\text{max}} + \sqrt{3s}\mu Ls\exp(-s/2) + \frac{4G\sqrt{s}}{n}\Big) \sum\limits_{t=t_0+1}^{T} \frac{1}{t^{CL\sqrt{s}+1}} \\
        & \le \frac{1}{L\sqrt{s}} \Big(\sqrt{3s}\mu L \rho_{\text{max}} + \sqrt{3s}\mu Ls\exp(-s/2) + \frac{4G\sqrt{s}}{n}\Big) \Big(\Big(\frac{T}{t_0}\Big)^{CL\sqrt{s}} - 1\Big) \\
        & = \underbrace{\Big(\frac{\sqrt{3}+4G}{nL}\Big)}_{D} \Big[\Big(\frac{T}{t_0}\Big)^{CL\sqrt{s}} - 1\Big] \tag{Apply \eqref{eqn:gen_param_cond}}
    \end{align*}
    Let $q \coloneqq CL\sqrt{s}$ and the optimal value $t_0^*$ which minimizes the RHS of the following inequality
    \begin{align*}
        \mathbb{E}\big[\|\loss(\bm\theta_T; \mathbf{z}) - \loss(\bm\theta_T'; \mathbf{z})\|\big] & \le \frac{t_0}{n} \sup_{\bm\theta, \mathbf{z}} \loss(\bm\theta; \mathbf{z}) + G\Delta_T \\
        & \le \frac{t_0}{n} + GD \Big[\Big(\frac{T}{t_0}\Big)^q - 1\Big]
    \end{align*}
    is given by $t_0^* = \min\{(nqGD T^q)^{\frac{1}{1+q}}, T\}$. Under the smoothing parameter condition $\mu \le \frac{1}{nL\rho_{\text{max}}}$, we have
    \begin{align*}
        \epsilon_{\textnormal{gen}} & \le \frac{q+1}{nq} (nqGDT^q)^{\frac{1}{1+q}} - GD \\
        & = \frac{q+1}{nq}\Big[\frac{qG(\sqrt{3}+4G)}{L}\Big]^{\frac{1}{1+q}} T^{1 - \frac{1}{1+q}} - \frac{G(\sqrt{3} + 4G)}{nL} \\
        & = \frac{q+1}{nq^{\frac{q}{1+q}}}\Big(\frac{G(\sqrt{3} + 4G)}{L}\Big)^{\frac{1}{1+q}} T^{1-\frac{1}{1+q}} - \frac{G(\sqrt{3} + 4G)}{nL} \\
        & = \mathcal{O}\Bigg(\frac{T^{1-\frac{1}{1+q}}}{n}\Bigg),
    \end{align*}
    where $q = CL\sqrt{s}$. 
\end{proof}
\clearpage
\section{Detailed Description of MeZO-BCD}\label{app:mezo-bcd}

This section provides a comprehensive description of our proposed method, MeZO-BCD, including its core design choices, the rationale for each, empirical comparisons, and algorithmic details. We focus on two critical dimensions: (1) \emph{how to define blocks for block-wise updates}, and (2) \emph{how to select the active block at each iteration}. For both, we describe the main alternatives considered, summarize supporting experiments, and justify our final recommendations.

\subsection{Defining Blocks in MeZO-BCD}

We primarily consider decoder-only Transformer architectures, which underlie most modern large language models (LLMs)~\citep{zhang2022opt,touvron2023llama}. These consist of an embedding layer (mapping input tokens to vectors; (\texttt{embedding}), a sequence of $L$ decoder layers (each with self-attention, feed-forward networks, and normalization layers), and a language model head (\texttt{lm\_head}) projecting to token logits.

We evaluate three natural block partitioning strategies:
\begin{itemize}[leftmargin=4mm]
    \item \textbf{Layer-wise:} The \texttt{embedding}, the \texttt{lm\_head}, and each decoder layer are treated as individual blocks.
    \item \textbf{Linear-wise:} The \texttt{embedding}, the \texttt{lm\_head}, and each linear layer within every decoder block are treated as separate blocks.
    \item \textbf{Multi-layer:} The \texttt{embedding}, the \texttt{lm\_head}, and each block formed by grouping every two consecutive decoder layers together.
\end{itemize}

The layer-wise scheme aligns with functional modularity: decoder layers represent the principal processing units in LLMs, as recognized by both architectural conventions and empirical studies~\citep{ju2024large, langedijk2024decoderlens, skean2025layer, zhang2024investigating}. This definition maintains interpretability and matches the typical tensor organization of model parameters.  
The linear-wise strategy increases block granularity by splitting each decoder layer into its constituent linear submodules, potentially enabling finer updates but at the cost of increased fragmentation.  
The multi-layer scheme, motivated by the sequential nature of hidden state propagation in transformers, merges every two consecutive decoder layers into a single block, reducing the number of blocks and further coarsening the update granularity.

It is important to note that defining blocks as arbitrary sub-tensors within a weight matrix (\emph{e.g.}, slicing a linear layer's weight matrix) may appear to increase flexibility. However, it is impractical in modern deep learning frameworks since the entire tensor must still be loaded into memory for perturbation and update, negating any gain in memory transfer or efficiency.

\paragraph{Empirical Insights.}
To empirically assess the effect of block partitioning, we conduct experiments on OPT-1.3B across representative classification tasks (SST-2, RTE, and WIC). To isolate the effect of block definition, we fix the block selection strategy to cyclic random permutation (see Appendix~\ref{app:block_selection}), train for 20K steps, and evaluate every 4K steps, reporting the best test accuracy for each run. Learning rates are adapted to account for the varying block sizes, reflecting the general trend that finer-grained partitions require larger learning rates~\citep{kim-etal-2024-ra, pmlr-v235-zhang24ax, lialin2023scaling}. The search spaces used are: \textit{linear-wise}: $\{1 \times 10^{-5},\ 2 \times 10^{-5},\ 3 \times 10^{-5}\}$; \textit{layer-wise}: $\{5 \times 10^{-6},\ 1 \times 10^{-5}\}$; \textit{multi-layer}: $\{1 \times 10^{-6},\ 2 \times 10^{-6},\ 5 \times 10^{-6}\}$.

\begin{table}[!h]
    \centering
    \caption{Test accuracy of MeZO-BCD with different block partitioning schemes on OPT-1.3B. Results are the best accuracy across 20K steps, where test accuracies are evaluated every 4K steps.}\label{tab:block_def}
    \vspace{5mm}
    \begin{tabular}{l|ccc}
        \toprule
        \textbf{Block Definition} & \textbf{SST-2} & \textbf{RTE} & \textbf{WIC} \\
        \midrule 
        Linear-wise      & 88.1 & 57.4 & 58.2 \\
        Layer-wise       & 92.3 & 65.0 & 60.2 \\
        Multi-layer (2-layer) & 91.2 & 65.0 & 58.2 \\
        \bottomrule
    \end{tabular}
\end{table}

Table~\ref{tab:block_def} presents the results, showing that the linear-wise strategy consistently underperforms, even after hyperparameter tuning, whereas both layer-wise and multi-layer deliver similar results given appropriate learning rates. While smaller blocks reduce memory usage per step, excessively fine granularity can harm convergence. The \textit{\textbf{layer-wise}} scheme offers the best trade-off between empirical accuracy and practical overhead, and is therefore adopted as the default in MeZO-BCD.

\subsection{Block Selection Strategies in MeZO-BCD}\label{app:block_selection}

Our theoretical analysis in the main text assumes that the active block is chosen i.i.d. uniformly at random at each step. However, in practical training, this can cause short-term “block starvation” periods where some blocks are updated much more or less frequently than others, which may hinder convergence stability.

To mitigate this, while preserving the marginal selection probability of $1/N$ for each block, we evaluate four block selection strategies:
\begin{itemize}[leftmargin=4mm]
    \item \textbf{Ascending order}: $1 \rightarrow 2 \rightarrow \ldots \rightarrow N \rightarrow 1 \rightarrow \ldots$
    \item \textbf{Descending order}: $N \rightarrow N-1 \rightarrow \ldots \rightarrow 1 \rightarrow N \rightarrow \ldots$
    \item \textbf{Flip-flop order}: $1 \rightarrow 2 \rightarrow \ldots \rightarrow N \rightarrow N-1 \rightarrow \ldots \rightarrow 2 \rightarrow 1 \rightarrow 2 \rightarrow \ldots$
    \item \textbf{Cyclic random permutation}: At the start of every $N$ steps, a random permutation of the $N$ blocks is generated and followed for the subsequent $N$ updates.
\end{itemize}
Here, a lower block index corresponds to layers closer to the input, and a higher index to layers closer to the output. While all strategies yield the same marginal probability of updating each block, they differ in their local ordering and thus in their short-term update patterns.

\paragraph{Empirical Insights.} To isolate the effect of block selection strategy, we conduct experiments with OPT-1.3B on SST-2, RTE, and WIC, fixing the block definition to \textit{layer-wise} and searching over learning rates $\{5\times10^{-6}, 1\times10^{-5}\}$. The experimental protocol mirrors that in Table~\ref{tab:block_def}, ensuring comparability.

Table~\ref{tab:block_ordering} presents the results. The flip-flop and cyclic random permutation strategies exhibit robust and consistently strong performance across all tasks, whereas the ascending and descending orderings show larger variance and less favorable results, particularly on RTE and WIC. Based on these findings, we adopt \textit{\textbf{flip-flop}} and \textit{\textbf{cyclic random permutation}} as our default block selection strategies, treating the choice between them as a discrete design parameter subject to tuning.

\begin{table}[!h]
    \centering
    \caption{Test accuracy of MeZO-BCD with different block selection strategies on OPT-1.3B. Results are the best accuracy across 20K steps, where test accuracies are evaluated every 4K steps}
    \label{tab:block_ordering}
    \vspace{5mm}
    \begin{tabular}{l|ccc}
        \toprule
        \textbf{Block Selection} & \textbf{SST-2} & \textbf{RTE} & \textbf{WIC} \\
        \midrule 
        Ascending order & 91.9 & 61.7 & 58.0 \\
        Descending order & 92.8 & 59.6 & 58.5 \\
        Flip-flop order & 92.7 & 62.9 & 58.2 \\
        Cyclic random permutation & 92.3 & 65.0 & 60.2 \\
        \bottomrule
    \end{tabular}
\end{table}

\subsection{Theoretical Memory Consumption Analysis}

To rigorously compare the memory consumption efficiency of zeroth-order methods, we analyze the theoretical \textit{peak} memory requirements during training, as summarized in Table~\ref{tab:theoretical_memory}. Here, “peak” refers to the maximum instantaneous memory usage (\emph{e.g.}, VRAM) at any point during optimization, including temporary allocations for perturbations, but excluding activations or any other intermediate byproducts such as temporary forward buffers, dataloader prefetch, or I/O buffer.

Let the model consist of $L$ layers, with $\mathbf{W}_\ell \in \mathbb{R}^{m_\ell \times n_\ell}$ denoting the weight matrix of the $\ell$-th layer and $r$ the low-rank perturbation rank. For simplicity, we denote $|\mathbf{W}_\ell| \coloneqq m_\ell n_\ell $ by the total dimension of $\mathbf{W}_\ell$. Also, we refer to \textit{held perturbation} as a segment of memory that is not continuously resident in the system but is instead allocated transiently in response to a Gaussian perturbation.

\begin{table}[h]
    \centering
    \caption{Theoretical peak memory requirements (in number of parameters) during training, excluding activations or any other intermediate byproducts such as temporary forward buffers, dataloader prefetch, or I/O buffer. The \texttt{max} operator reflects instantaneous memory usage due to temporarily held perturbations.}\label{tab:theoretical_memory}
    \vspace{2mm}
    \begin{tabular}{l|l}
        \toprule
        \textbf{Method} & \textbf{Peak Memory Requirement} \\
        \midrule
        MeZO~\citep{malladi2023fine} & 
        $\underbrace{\sum_{\ell=1}^L |\mathbf{W}_\ell|}_{\text{model weights}} 
         + \underbrace{\max_{\ell} |\mathbf{W}_\ell|}_{\text{held perturbation}}$ \\
        SparseMeZO~\citep{liu2024sparse} & 
        $\underbrace{\sum_{\ell=1}^L |\mathbf{W}_\ell|}_{\text{model weights}}
         + \underbrace{\max_{\ell} |\mathbf{W}_\ell|}_{\text{held perturbation}}
         + \underbrace{\max_{\ell} |\mathbf{W}_\ell|}_{\text{sparse mask}}$ \\
        LoZO~\citep{chen2024enhancing} & 
        $\underbrace{\sum_{\ell=1}^L |\mathbf{W}_\ell|}_{\text{model weights}}
         + \underbrace{\max_{\ell} m_\ell r}_{\text{instant low-rank perturbation}}
         + \underbrace{\sum_{\ell=1}^L n_\ell r}_{\text{persistent low-rank factor}}$ \\
        MeZO-BCD (Ours) & 
        $\underbrace{\sum_{\ell=1}^L |\mathbf{W}_\ell|}_{\text{model weights}}
         + \underbrace{\max_{\ell \in B} |\mathbf{W}_\ell|}_{\text{held perturbation}}$ \\
        \bottomrule
    \end{tabular}
\end{table}

All methods must store the full set of model parameters ($\sum_{\ell=1}^L |\mathbf{W}_\ell|$) in memory throughout training. The additional \textit{held perturbation} term represents the transient memory usage needed to apply weight perturbations within each training step. For MeZO~\citep{malladi2023fine} and SparseMeZO~\citep{liu2024sparse}, this extra allocation never exceeds the size of a single layer’s weights at any moment, as perturbations are generated, used, and discarded immediately. SparseMeZO also incurs overhead from an additional mask of the same size. In contrast, LoZO~\citep{chen2024enhancing}’s low-rank scheme involves both instant memory for sampled low-rank factors and persistent memory for factors reused across steps. MeZO-BCD, on the other hand, requires perturbing only the currently active block, so its peak auxiliary memory is, by construction, at most the largest weight size within any block $B$.

Importantly, although MeZO-BCD perturbs only one block at a time, every block is eventually selected during training. Therefore, the peak memory usage still corresponds to the largest block, and the effective bound matches that of MeZO. As a result, the true peak is determined by the largest layer across all blocks, making the effective memory bound identical to that of MeZO. Therefore, the use of the $\max$ operator in the analysis reflects that, while only one block is perturbed at a time, the overall peak is dictated by the largest such allocation encountered over all steps. Consequently, MeZO-BCD (like MeZO) achieves the theoretical minimum peak memory for zeroth-order optimization, essentially matching inference-time requirements, whereas alternative methods may require substantially more memory. This guarantees that MeZO-BCD’s temporary perturbations never incur a memory usage exceeding the size of the largest individual block or layer.

\paragraph{Remark on the experimental results in Table~\ref{tab:time_mem_comparison}.} On SST-2, where activation buffers are relatively small, the reduction in auxiliary perturbation memory afforded by MeZO-BCD is large enough to lower the overall peak VRAM usage slightly, hence the small differences observed between methods on this dataset. By contrast, BoolQ and MultiRC use much longer input sequences, so their activation tensors dominate the GPU footprint. In those cases, the extra memory for perturbations or masks becomes negligible, and all methods converge to essentially the same peak memory.

Pseudocode for the full MeZO-BCD procedure is provided in Algorithm~\ref{alg:mezo_bcd}. We conclude this section with the following remark.

\paragraph{Why BCD is Particularly Effective in Zeroth-Order Optimization.} It is worth emphasizing why BCD, despite its classical status in first-order (FO) optimization, is particularly advantageous in the zeroth-order (ZO) context. In FO optimization, updating a single block still incurs the cost of computing and storing gradients for all downstream layers via backpropagation, resulting in substantial computational and memory overhead (especially for blocks closer to the input). In contrast, ZO optimization leverages only forward evaluations, so updating a block affects solely its own parameters with no added cost for the rest of the network. As a result, ZO-BCD achieves genuine reductions in memory transfer and computational overhead per step, avoiding the bottlenecks that limit the efficiency of BCD in first-order regimes. This distinction underlies the scalability and efficiency of MeZO-BCD for large-scale LLM fine-tuning.

\begin{algorithm}[!h]
\caption{MeZO-BCD}\label{alg:mezo_bcd} 
\centering
\begin{algorithmic}
    \REQUIRE block partitioning $\{B_i\}_{i=1}^{N}$, block ordering $\texttt{order}$, parameters $\mathbf{\theta}\in\mathbb{R}^d$ is partitioned into $[\theta_{B_1},\ldots,\theta_{B_N}]$, loss function $\mathcal{L}:\mathbb{R}^d\rightarrow\mathbb{R}$, step budget $T$, perturbation scale $\mu$, learning rate schedule $\{\eta_t\}_{t=1}^{T}$
    \item[]
    \FOR{$t=1,\ldots,T$}
        \STATE Sample minibatch $\mathcal{B}\subset \mathcal{D}$ and random seed $s$
        \STATE Update current active block index $j\leftarrow\mathrm{\texttt{UpdateBlockIdx}}(\texttt{order}, t, N)$
        \STATE $\mathbf{\theta}_{B_j}\leftarrow \mathrm{\texttt{PerturbParameters}}(\mathbf{\theta}_{B_j},\mu,s)$
        \STATE $\ell_+\leftarrow\mathcal{L}(\mathbf{\theta};\mathcal{B})$
        \STATE $\mathbf{\theta}_{B_j}\leftarrow \mathrm{\texttt{PerturbParameters}}(\mathbf{\theta}_{B_j},-2\mu,s)$
        \STATE $\ell_-\leftarrow\mathcal{L}(\mathbf{\theta};\mathcal{B})$
        \STATE $\mathbf{\theta}_{B_j}\leftarrow \mathrm{\texttt{PerturbParameters}}(\mathbf{\theta}_{B_j},\mu,s)$
        \item[]
        \STATE\texttt{projected\_grad} $\leftarrow (\ell_+-\ell_-)/2\mu$
        \STATE Reset random number generator with seed $s$
        \FOR{$\theta_i\in\mathbf{\theta}_{B_j}$}
            \STATE $u\sim\mathcal{N}(0,1)$
            \STATE $\theta_i\leftarrow\theta_i-\eta_t$ * \texttt{projected\_grad} * $u$
        \ENDFOR
    \ENDFOR
\end{algorithmic}
\end{algorithm}

\begin{figure}[htb]
\centering
\begin{minipage}[t]{0.48\textwidth}
    \begin{algorithm}[H]
    \caption{\texttt{PerturbParameters($\mathbf{\theta}, \mu, s$)}}\label{alg:perturb_params}
    \begin{algorithmic}
        \REQUIRE parameters $\mathbf{\theta}\in\mathbb{R}^n$, perturbation scale $\mu$, random seed $s$
        \item[]
        \STATE Reset random number generator with seed $s$
        \FOR{$\theta_i\in\mathbf{\theta}$}
            \STATE $u\sim\mathcal{N}(0,1)$
            \STATE $\theta_i\leftarrow \theta_i+\mu u$
        \ENDFOR
        \STATE \textbf{return} $\mathbf{\theta}$
    \end{algorithmic}
    \end{algorithm}
\end{minipage}
\hspace{0.02\textwidth}
\begin{minipage}[t]{0.48\textwidth}
    \begin{algorithm}[H]
    \caption{\texttt{UpdateBlockIdx(order, $t, N$)}}\label{alg:update_block_idx}
    \begin{algorithmic}
        \REQUIRE block ordering \texttt{order}, current time step $t$, number of blocks $N$
        \item[]
        \IF{\texttt{order} is ascending}
            \STATE $i\leftarrow ((t-1)\bmod N) +1$
        \ELSIF{\texttt{order} is descending}
            \STATE $i\leftarrow N - ((t-1) \bmod N)$
        \ELSIF{\texttt{order} is flip-flop}
            \STATE $i\leftarrow N-|((t-1)\bmod(2N-2))-(N-1)|$
        \ELSIF{\texttt{order} is random}
            \IF{$t\bmod N = 0$}
                \STATE Reset the random permutation $\pi\leftarrow \texttt{Shuffle}([N])$
            \ENDIF
            \STATE $i\leftarrow \pi[(t \bmod N)]$
        \ENDIF
        \STATE\textbf{return} $i$
    \end{algorithmic}
    \end{algorithm}
\end{minipage}
\end{figure}
\clearpage
\section{Experimental Details}

\subsection{Randomized Quadratic Minimization}\label{app:exp_detail_empirical}
For the experiments in Section~\ref{subsec:empirical_validation}, the perturbation scale is globally set to $\mu=10^{-4}$ unless stated otherwise. The Hessian-like matrix $\mathbf{H}$ is generated as follows. The matrix is constructed as a block-diagonal structure, where each block is a randomly generated positive semi-definite matrix. Specifically, the $d \times d$ matrix $\mathbf{H}$ is divided into $B$ non-overlapping blocks, each of size $\frac{d}{B} \times \frac{d}{B}$. For each block, a random $\frac{d}{B} \times \frac{d}{B}$ matrix is sampled, and an orthogonal matrix $\mathbf{Q}_i$ is obtained via QR decomposition. Given a predefined maximum eigenvalue $\lambda_{\max}^{(i)}$ for each block, the diagonal matrix $\mathbf{\Lambda}^{(i)}$ is constructed as:

\[
\mathbf{\Lambda}^{(i)} = \mathrm{diag}(\underbrace{\lambda_{\max}^{(i)}, \ldots, 0.1\times\lambda_{\max}^{(i)}}_{r}, \underbrace{0, \ldots, 0}_{\frac{d}{B} - r}),
\]

where eigenvalues decay linearly. Each block of the Hessian-like matrix is then formed as:

\[
\mathbf{H}_i = \mathbf{Q}_i \mathbf{\Lambda}^{(i)} \mathbf{Q}_i^\top.
\]

The final Hessian-like matrix $\mathbf{H}$ is assembled by placing these block matrices along the diagonal:

\[
\mathbf{H} =
\begin{bmatrix}
    \mathbf{H}_1 & 0 & \cdots & 0 \\
    0 & \mathbf{H}_2 & \cdots & 0 \\
    \vdots & \vdots & \ddots & \vdots \\
    0 & 0 & \cdots & \mathbf{H}_B
\end{bmatrix}.
\]

This process ensures that $\mathbf{H}$ maintains a block-diagonal structure while preserving the spectral properties within each block. The following function constructs $\mathbf{H}$ according to the above procedure:

\begin{python}[language=Python, caption=]
def generate_hessian_like_matrix(dim=256, rank=64, num_blocks=1, max_eigenvals=[10]):
    block_size = dim // num_blocks
    H = torch.zeros((dim, dim))

    for i in range(num_blocks):
        start = i * block_size
        end = start + block_size

        Q, _ = torch.linalg.qr(torch.randn(block_size, block_size))

        max_eigenval = max_eigenvals[i]
        eigenvalues = torch.cat([
            torch.linspace(max_eigenval, 0.1 * max_eigenval, steps=rank),
            torch.zeros(block_size - rank)
        ])
        S = torch.diag(eigenvalues)
        block = Q @ S @ Q.T
        
        H[start:end, start:end] = block

    return H
\end{python}

\subsubsection{Configuration for Figure~\ref{fig:curves_rho}}\label{app:exp_detail_fig1}
To examine the influence of the subspace alignment $\rho$, we use a dense Hessian matrix (i.e., $B=1$). The subspace alignment $\rho$ is controlled by introducing a parameter $\gamma \in [0,1]$, which determines the proportion of Hessian eigenvectors incorporated into the low-rank projection matrix $\mathbf{M}$. Specifically, $\lceil s\gamma\rceil$ eigenvectors (where $s = \mathrm{srank}(\mathbf{M})$) are randomly selected, while the remaining vectors are randomly generated and orthogonalized against the selected ones.

Formally, let the Hessian matrix $\mathbf{H}$ have rank $r$, and let $s$ denote the stable rank of $\mathbf{M}$. Denote the eigenvectors of $\mathbf{H}$ corresponding to nonzero eigenvalues as $\{\mathbf{q_1}, \mathbf{q_2}, \dots, \mathbf{q_r}\}$. To construct $\mathbf{M}$ with a controlled alignment $\rho$:
\begin{enumerate}[leftmargin=7mm]
    \item Randomly selected $s\gamma$ eigenvectors are set as columns of $\mathbf{M}_1$ (assuming $s\gamma$ is an integer for simplicity).
    \item A random $d \times (1-s\gamma)$ matrix $\mathbf{R}$ is generated.
    \item The QR decomposition is applied to $\mathbf{R} - \mathbf{M}_1\mathbf{M}_1^\top \mathbf{R}$ to obtain an orthogonal matrix $\mathbf{M}_2$ of size $d \times (1-s\gamma)$.
    \item The final projection matrix is constructed as $\mathbf{M} = [\mathbf{M}_1, \mathbf{M}_2]$.
\end{enumerate}

\paragraph{Experimental Setup for the Left Panel of Figure~\ref{fig:curves_rho}.}
For this experiment, we set $d=256$, the Hessian rank to 64, and the stable rank $s$ of $\mathbf{M}$ to 64. This ensures that when $\gamma=1.0$, the projection matrix fully captures the Hessian eigenvectors. We use a learning rate of $\eta = 10^{-3}$, a maximum of 1000 steps, and test five values of $\gamma$:  
\[
\gamma = \{0, 0.2, 0.4, 0.7, 1.0\}.
\]

\paragraph{Experimental Setup for the Middle Panel of Figure~\ref{fig:curves_rho}.}
The configuration is the same as in the left panel, except:
\begin{itemize}[leftmargin=7mm]
    \item The maximum number of iterations is increased to 10,000.
    \item The number of $\gamma$ values is expanded to 12:
      \[
      \gamma = \{0.0, 0.05, 0.1, 0.2, 0.3, 0.4, 0.5, 0.6, 0.7, 0.8, 0.9, 1.0\}.
      \]
    \item The target loss is defined as the minimum loss achieved during training under $\gamma=0.0$.
    \item The number of required iterations is computed as the minimum number of steps needed for the loss curve to reach this target value.
\end{itemize}

\subsubsection{Configuration for the Right Panel of Figure~\ref{fig:curves_rho}}\label{app:exp_detail_fig2}
To evaluate the subspace alignment $\rho$ across different methods, we conduct randomized quadratic minimization experiments using a block-diagonal Hessian matrix with heterogeneous eigenspectra. To ensure a broader range of $\mathrm{srank}(\mathbf{M})$ values, the problem dimension is set to $d=1024$, with the Hessian consisting of 16 diagonal blocks, each of rank 16. To introduce heterogeneity in the eigenspectrum across blocks, the eigenvalues in each block are set by randomly choosing a reference from $\{10, 40, 70, 100\}$ and sampling each eigenvalue as an integer within $\pm 2$ of it.

\paragraph{Measuring subspace alignment.}
Since the subspace alignment $\rho$ depends only on the projection matrix $\mathbf{M}$ and not on the parameter $\theta$, we measure $\rho$ without performing actual optimization. Instead, we randomly sample $\mathbf{M}$ 1000 times for each method and compute $\rho$ accordingly.

The projection matrices $\mathbf{M}$ are generated as follows:
\begin{itemize}[leftmargin=7mm]
    \item \textbf{Low-rank Projection}: The projection matrix is defined as $\mathbf{M} = \mathbf{U} \mathbf{U}^\top$, where $\mathbf{U}$ is obtained by performing QR decomposition on a randomly generated $d \times s$ matrix.
    \item \textbf{Sparse Perturbation}: A binary mask vector $\mathbf{m} = (m_1, \ldots, m_d)$ is created by randomly permuting $[d]$ and selecting the first $(1 - \frac{s}{d})$ fraction of indices.
    \item \textbf{Block Sparse Perturbation}: Given a predefined number of blocks $N$, each block is defined as:
    \[
    B_i = \left\{ (i-1) \times \frac{d}{N} + 1, \dots, i \times \frac{d}{N} \right\}.
    \]
    A block index $j$ is selected uniformly at random, and the mask vector is set such that $m_i = 1$ for $i \in B_j$, while all other entries are set to 0.
\end{itemize}

\paragraph{Experimental Setup.}
To examine the distribution of $\rho$ across different values of $\mathrm{srank}(M)$, we conduct experiments with:
\[
s \in \{16, 32, 64, 128, 256\}.
\]
For each $s$, the distribution of $\rho$ is computed over 1000 trials.

\subsection{Fine-tuning Large Language Models}\label{app:exp_detail_opt}
We now describe the experimental details for our large-scale fine-tuning experiments on LLMs. To ensure a fair comparison, we strictly adhere to the experimental setup of LoZO~\citep{chen2024enhancing}, maintaining consistency in the training environment.

\paragraph{Model and Datasets.}
For fine-tuning OPT-13B~\citep{zhang2022opt}, we conduct experiments on a diverse set of datasets, including SST-2~\citep{sst2}, RTE~\citep{dagan2005pascal, bar2006second, giampiccolo2007third, bentivogli2009fifth}, CB~\citep{de2019commitmentbank}, BoolQ~\citep{clark2019boolq}, WSC~\citep{levesque2012winograd}, WIC~\citep{pilehvar2018wic}, MultiRC~\citep{khashabi2018looking}, COPA~\citep{roemmele2011choice}, ReCoRD~\citep{zhang2018record}, SQuAD~\citep{rajpurkar2018know}, and DROP~\citep{dua2019drop}.  
For OPT-1.3B, we focus exclusively on SST-2 to evaluate convergence speed and optimization efficiency.

\paragraph{Licenses for Models and Datasets}
All models and datasets used in this work are publicly released and available for academic research. Specifically, we use the OPT-125M, OPT-1.3B, and OPT-13B language models released by Meta AI under a custom non-commercial research-only license.\footnote{\url{https://huggingface.co/facebook/opt-13b}} For all fine-tuning and evaluation tasks, we rely on standard NLP benchmarks: SST-2, RTE, CB, COPA, WSC, WIC, BoolQ, SQuAD, DROP, ReCoRD, and MultiRC. These datasets are part of the GLUE/SuperGLUE benchmark suite or other widely used QA corpora, and are distributed under permissive open-source or research-friendly licenses, including Apache 2.0, CC BY 4.0, CC BY-SA 4.0, and similar. We cite all relevant sources and have ensured that our usage adheres to each asset’s licensing terms. Where applicable, non-commercial restrictions have been strictly observed.

\paragraph{Prompts.}
In all our experiments, we use the same prompt templates as MeZO~\citep{malladi2023fine} (and thus LoZO~\citep{chen2024enhancing}) for every dataset, ensuring strict comparability and reproducibility. The prompt templates are provided in Table~\ref{tab:prompt}

\paragraph{Hyperparameters.}
Tables~\ref{tab:opt_hyper_13b} and~\ref{tab:opt_hyper_1.3b} outline the hyperparameter search spaces used for OPT-13B and OPT-1.3B, respectively, ensuring reproducibility.  
All experiments utilize a constant learning rate schedule, and training is conducted for 20K steps. The test accuracy is evaluated for every 4K steps, and we report the best one. The wall-clock time for OPT-1.3B experiments is measured on a single NVIDIA RTX 3090 GPU with an Intel Xeon Gold 5215 Processor.

\paragraph{Notes on the statistical significance.} Table~\ref{tab:opt-13b} reports results for a single random seed: $0$ (here, the random seed refers to the overall seed that controls the entire training process, including dataset sampling and super seeds for ZO perturbations, rather than just the per-iteration ZO seed). This choice was made to ensure a fair comparison with our main baseline, LoZO~\citep{chen2024enhancing}, which did not report any measures of statistical significance in their original paper. Analyzing their released code, we inferred that their published results were also obtained with a single run using random seed $0$. To maintain consistency and fairness in our comparisons, we therefore adopted the same experimental protocol for Table~\ref{tab:opt-13b}. While this approach enables a direct comparison, we acknowledge that statistical significance is essential for a robust evaluation. Accordingly, for completeness, we report additional results averaged over three random seeds in Section~\ref{app:stat_sig}, providing a more comprehensive assessment of performance variability.

\paragraph{Notes on the reproducibility of SparseMeZO~\citep{liu2024sparse}.}
We attempted to reproduce SparseMeZO~\citep{liu2024sparse} using our own implementation, as the official code is not publicly released. However, on fine-tuning OPT-13B for the RTE task, our best reproduction achieved only $57.0\%$ test accuracy, which is substantially lower than the $77.6\%$ reported in the original paper (Table 4 in~\citep{liu2024sparse}). Despite extensive tuning efforts, we were unable to close this gap. Due to this substantial gap and the unavailability of reference code for verification, we refrained from including our reproduction to avoid unfairly disadvantaging the original method. We therefore exclude SparseMeZO from our main experiments to preserve the integrity of cross-method comparisons.

\begin{table}[h]
\centering
\caption{Prompt templates and evaluation formats for all datasets. Prompt templates are identical to those used in MeZO~\citep{malladi2023fine}, ensuring strict consistency across methods.}\label{tab:prompt}
\vspace{4mm}
\resizebox{\textwidth}{!}{
\begin{tabular}{lll}
\toprule
Dataset & Task Type & Prompt Format \\
\midrule
SST-2 & Classification & \texttt{<text>} It was terrible/great \\
RTE & Classification & \texttt{<premise>} Does this mean that ``\texttt{<hypothesis>}'' is true? Yes or No? \\
CB & Classification & Suppose \texttt{<premise>} Can we infer that ``\texttt{<hypothesis>}''? Yes, No, or Maybe? \\
BoolQ & Classification & \texttt{<passage>} \texttt{<question>}? Yes/No \\
WSC & Classification & \texttt{<text>}\\
& & In the previous sentence, does the pronoun ``\texttt{<span2>}'' refer to \texttt{<span1>}? Yes or No? \\
WIC & Classification & Does the word ``\texttt{<word>}'' have the same meaning in these two sentences? Yes, No?\\
& & \texttt{<sent1>}\\
& & \texttt{<sent2>} \\
MultiRC & Classification & \texttt{<paragraph>}\\
& & Q: \texttt{<question>}\\
& & I found this answer ``\texttt{<answer>}''. Is that correct? Yes/No \\
COPA & Multiple Choice & \texttt{<premise>} so/because \texttt{<candidate>} \\
ReCoRD & Multiple Choice & \texttt{<passage>}\\
& & \texttt{<query>}.replace("@placeholder", \texttt{<candidate>})\\
SQuAD & QA & Title: \texttt{<title>}\\
& & Context: \texttt{<context>}\\
& & Q: \texttt{<question>}\\
& & A: \\
DROP & QA & Passage: \texttt{<context>}\\
& & Q: \texttt{<question>}\\
& & A: \\
\bottomrule
\end{tabular}
}
\label{tab:prompt-templates}
\end{table}

\begin{table*}[h]
    \centering
    \caption{The hyperparameter configurations utilized for OPT-13B experiments.}
    \begin{tabular}{lcc}
    \toprule
    Method & Hyperparameters & Values \\
    \midrule
    MeZO & Batch size & $16$ \\
    & Learning rate & $\{1\mathrm{e}{-6}, 1\mathrm{e}{-7} \}$\\
    & $\mu$ & $1\mathrm{e}{-3}$ \\
    \midrule
    LOZO & Batch size & $16$ \\
    & Learning rate & $\{1\mathrm{e}{-6}, 1\mathrm{e}{-7}\}$ \\
    & $\mu$ & $1\mathrm{e}{-3}$ \\
    & Rank ($r$) & $\{1,2,4\}$ \\
    & Interval ($\nu$) & $\{50\}$ \\
    \midrule
    MeZO-BCD & Batch size & $16$ \\
    & Learning rate & $\{1\mathrm{e}{-5}, 5\mathrm{e}{-6}\}$\\
    & $\mu$ & $1\mathrm{e}{-3}$ \\
    & Block Update Order & \{random, flip-flop\} \\
    \bottomrule
    \end{tabular}
    \label{tab:opt_hyper_13b}
\end{table*}

\begin{table*}[h]
    \centering
    \caption{The hyperparameter configurations utilized for OPT-1.3B experiments}
    \begin{tabular}{lcc}
    \toprule
    Method & Hyperparameters & Values \\
    \midrule
    MeZO & Batch size & $16$ \\
    & Learning rate & $\{1\mathrm{e}{-6}, 1\mathrm{e}{-7} \}$\\
    & $\mu$ & $1\mathrm{e}{-3}$ \\
    \midrule
    LOZO & Batch size & $16$ \\
    & Learning rate & $\{1\mathrm{e}{-6}, 1\mathrm{e}{-7}\}$ \\
    & $\mu$ & $1\mathrm{e}{-3}$ \\
    & Rank ($r$) & $\{1,2,4\}$ \\
    & Interval ($\nu$) & $\{50\}$ \\
    \midrule
    SparseMeZO & Batch size & $16$ \\
    & Learning rate & $\{1\mathrm{e}{-7}, 2\mathrm{e}{-7}, 5\mathrm{e}{-7} \}$ \\
    & $\mu$ & $1\mathrm{e}{-3}$ \\
    & Sparsity & $\{0.75, 0.8\}$ \\
    \midrule
    MeZO-BCD & Batch size & $16$ \\
    & Learning rate & $\{1\mathrm{e}{-5}, 5\mathrm{e}{-6}\}$\\
    & $\mu$ & $1\mathrm{e}{-3}$ \\
    & Block Update Order & \{random, flip-flop\} \\
    \bottomrule
    \end{tabular}
    \label{tab:opt_hyper_1.3b}
\end{table*}
\clearpage
\section{Supplementary Experimental Results}
\subsection{Supplementary Results on Empirical Study}\label{app:supp_empirical_study}

To complement the empirical observations from Figure~\ref{fig:curves_rho} and further validate our theoretical insights in Section~\ref{subsec:empirical_validation}, we present additional visualizations and convergence curves under varying subspace perturbations. These results reinforce the role of the subspace alignment $\rho$ and offer a deeper view into the distributional behavior of different subspace perturbation schemes.

\begin{figure}[ht]
    \centering
    \subfigure[Log-scaled boxen plot of distribution of $\rho$]{\includegraphics[width=0.48\textwidth]{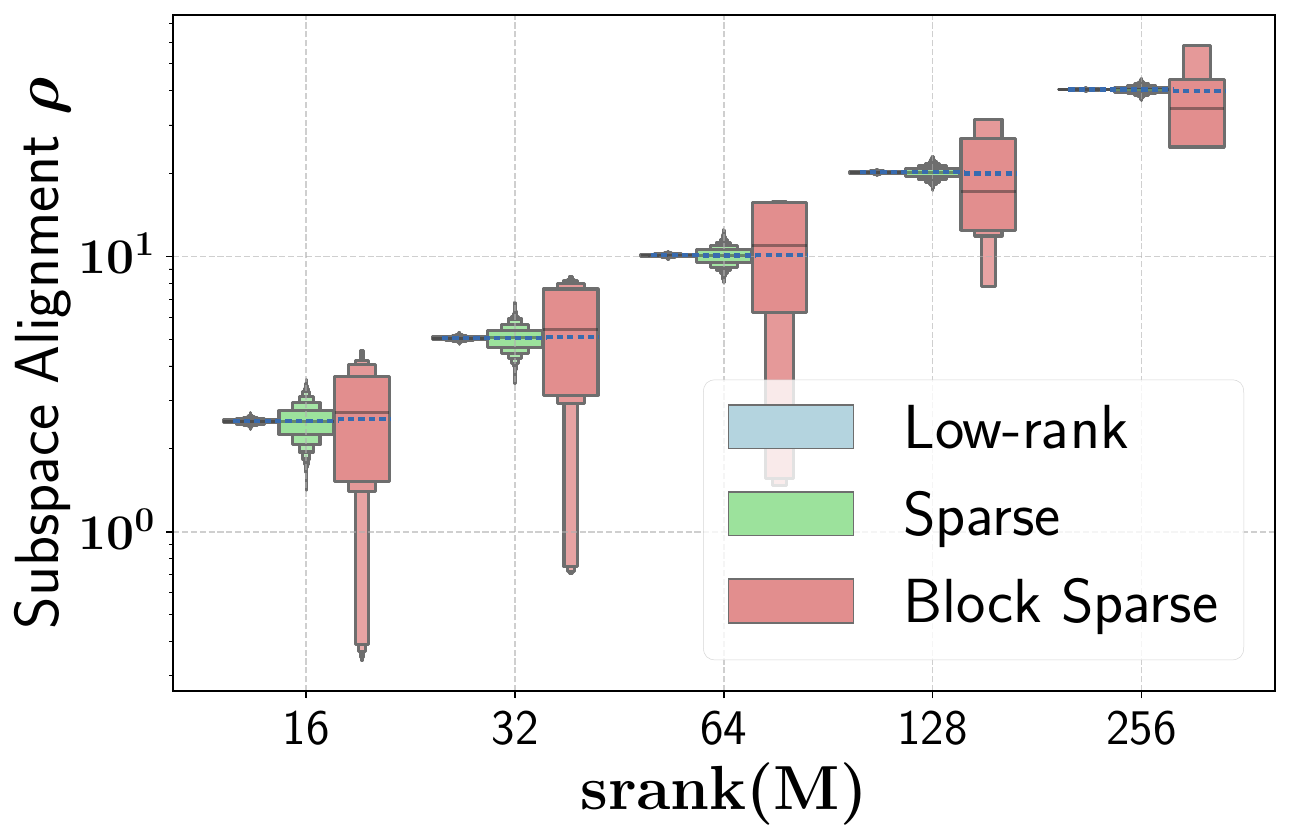}}
    \subfigure[Strip plot of distribution of $\rho$]{\includegraphics[width=0.48\textwidth]{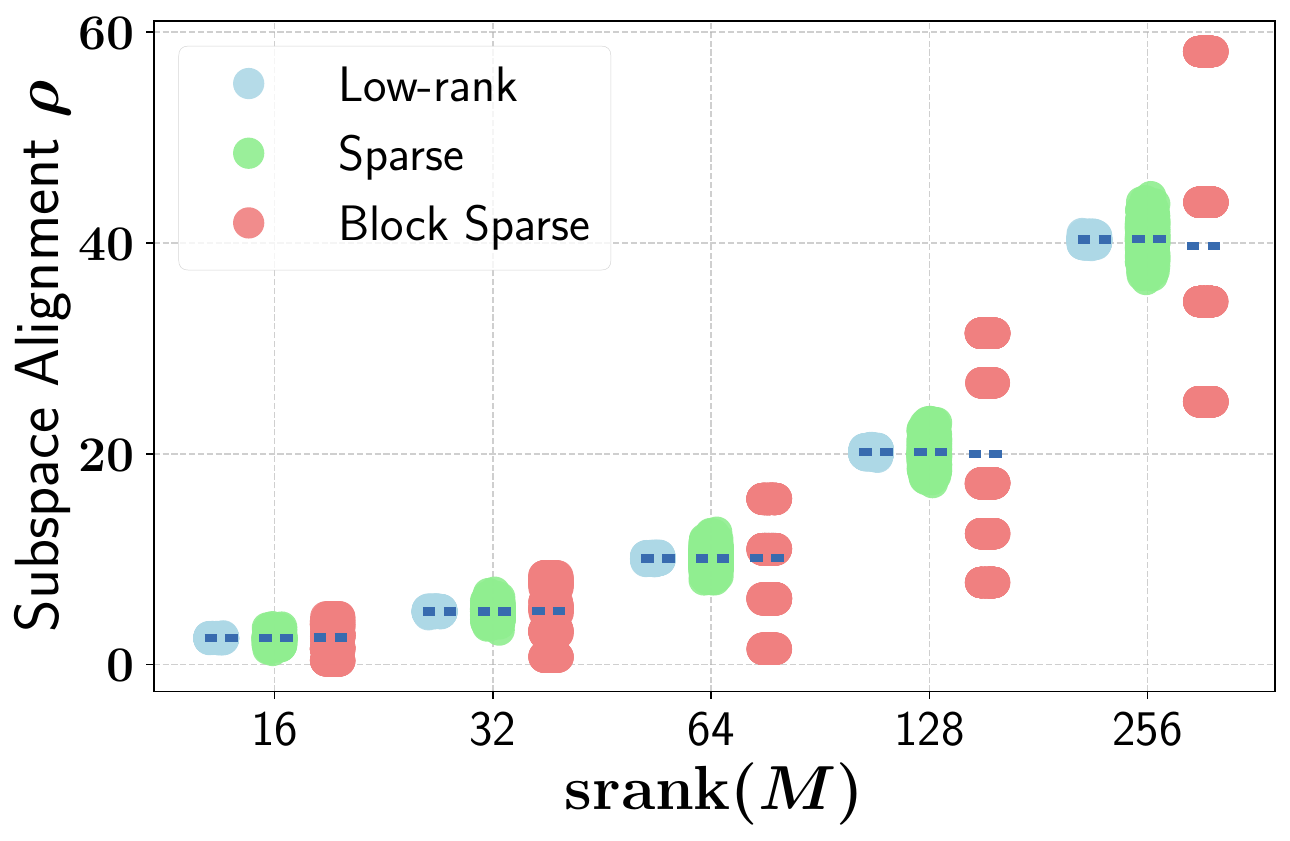}}
    \caption{Comparative visualization of the distribution of subspace alignment $\rho$ under varying $\mathrm{srank}(\mathbf{M})$ across three perturbation types: low-rank projection, sparse perturbation, and block sparse perturbation. (a) The log-scaled boxen plot visualizes the lower tail behavior of $\rho$. (b) The strip plot highlights the discrete nature of block sparse perturbations.}
    \label{fig:sup_fig1}
\end{figure}

Figure~\ref{fig:sup_fig1} provides alternative visualizations of the subspace alignment $\rho$ for different subspace perturbation schemes. While the mean alignment $\bar\rho$ remains similar across all cases (consistent with Proposition~\ref{prop:expected_rho}), the distributional patterns differ significantly. In particular, block sparse perturbation displays a discrete and more dispersed distribution compared to the tightly concentrated low-rank projection and sparse perturbation, aligning with the predictions of Proposition~\ref{prop:prob_rho}.

\begin{figure}[!h]
    \centering
    \vskip -10pt
    \subfigure[$\mathrm{srank}(M)=16$]{\includegraphics[width=0.32\textwidth]{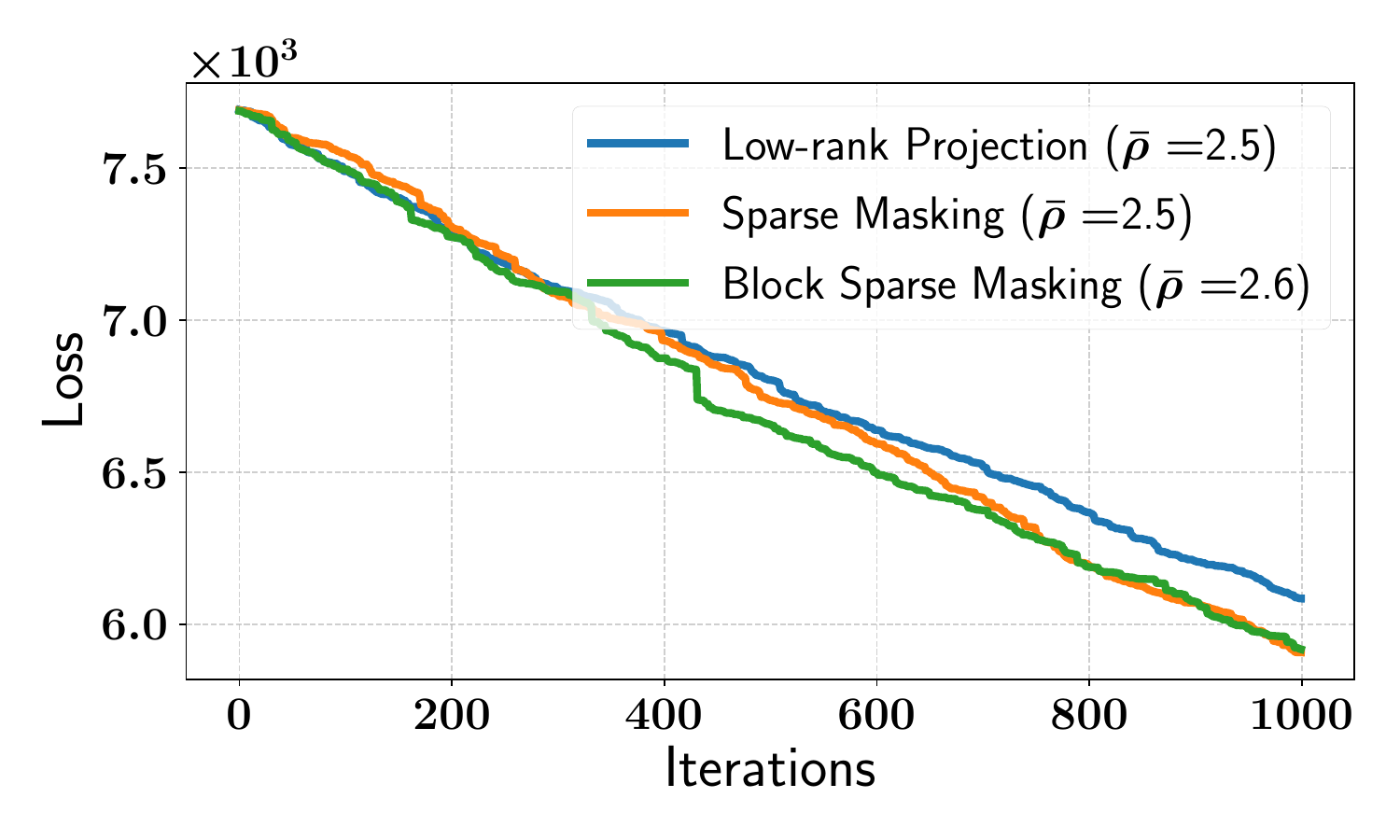}}
    \subfigure[$\mathrm{srank}(M)=32$]{\includegraphics[width=0.32\textwidth]{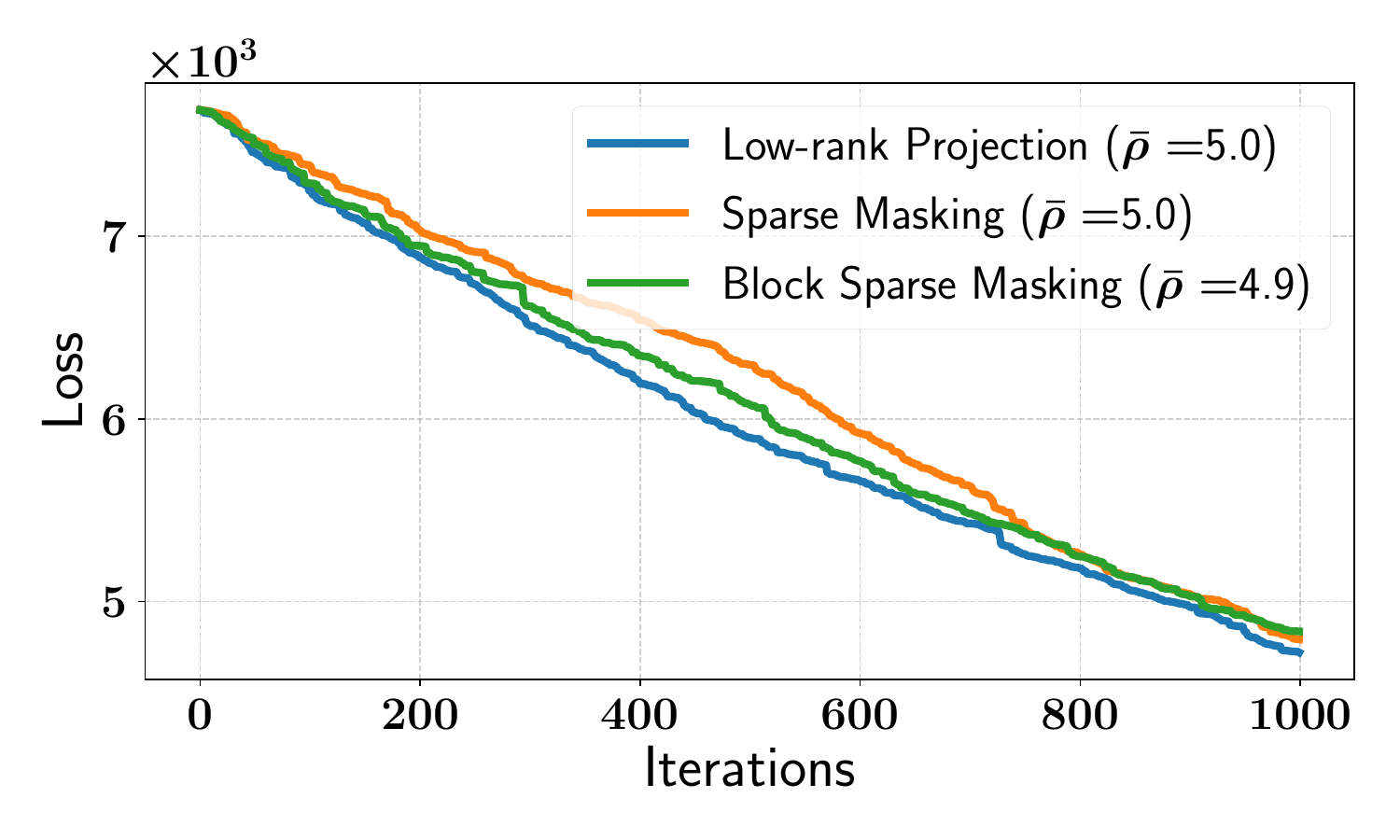}}
    \subfigure[$\mathrm{srank}(M)=64$]{\includegraphics[width=0.32\textwidth]{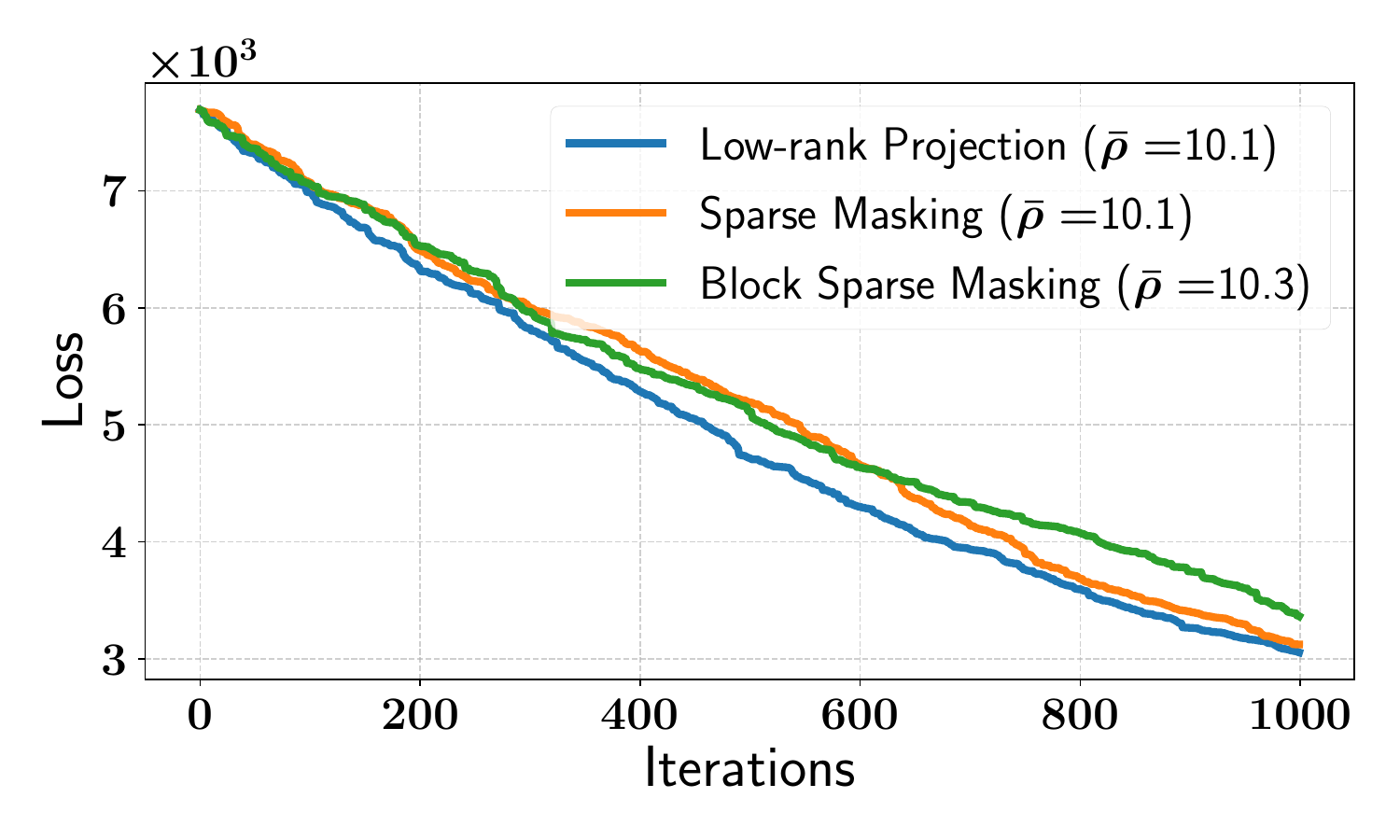}}
    \subfigure[$\mathrm{srank}(M)=128$]{\includegraphics[width=0.32\textwidth]{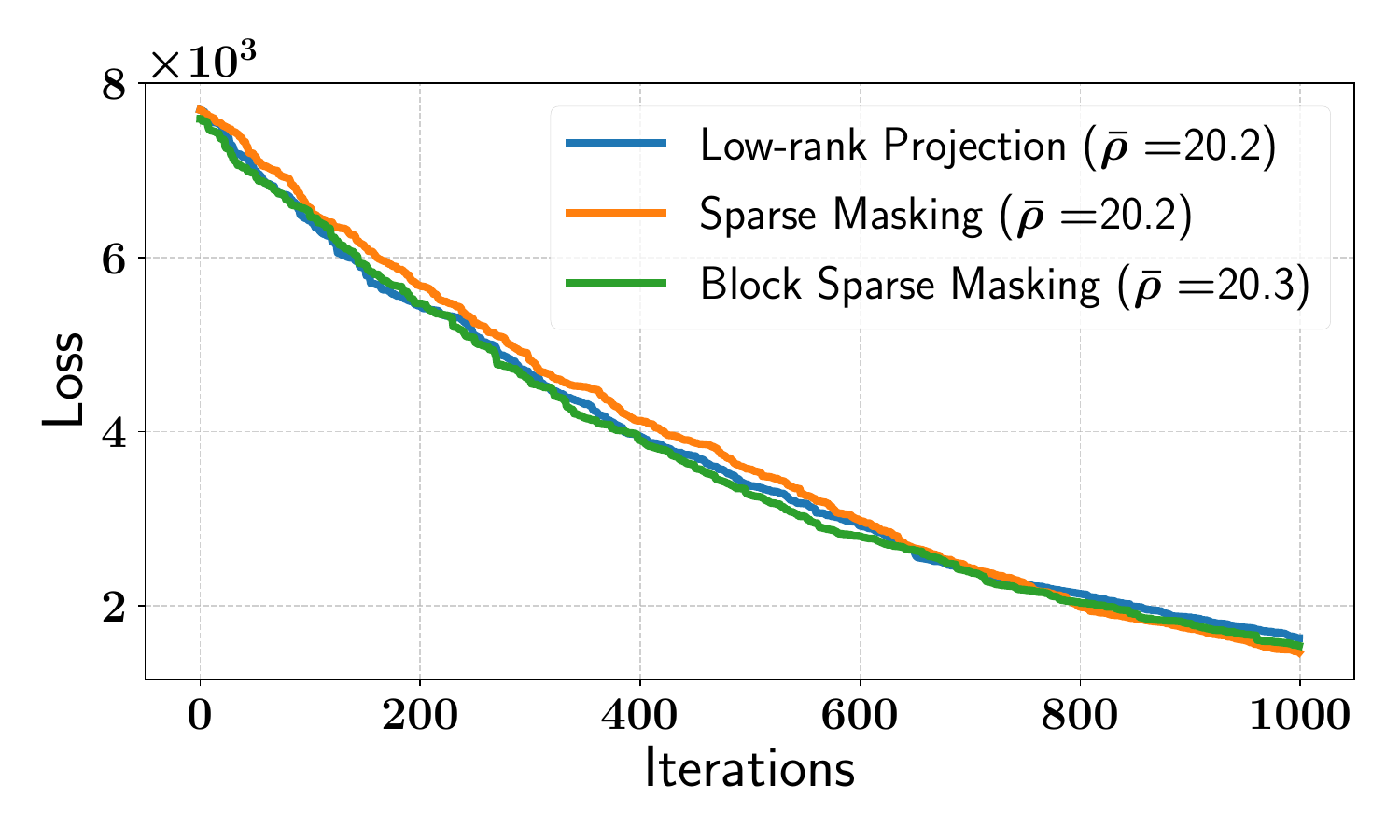}}
    \subfigure[$\mathrm{srank}(M)=256$]{\includegraphics[width=0.32\textwidth]{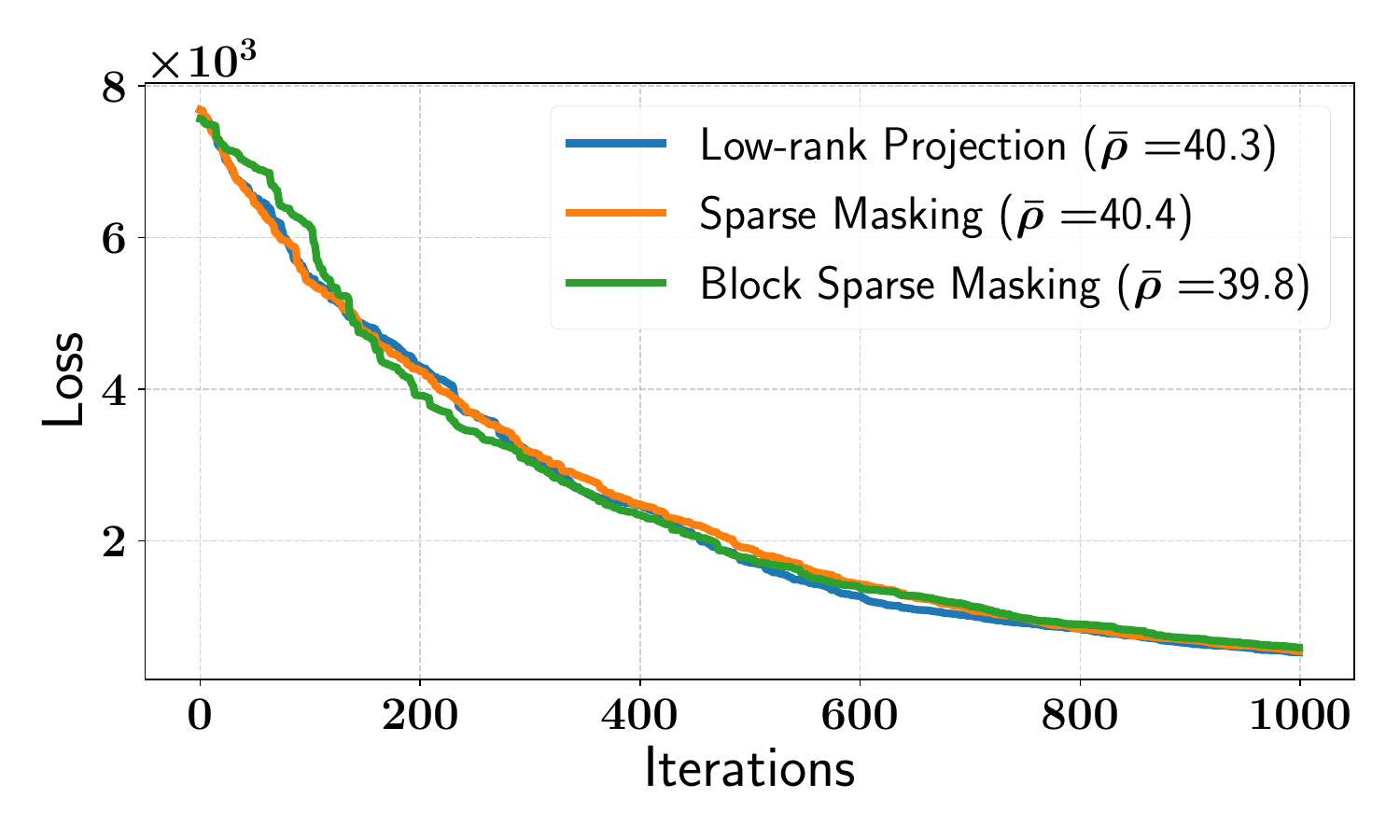}}
    \caption{Training loss curves for varying $\mathrm{srank}(\mathbf{M})$ across low-rank, sparse, and block sparse perturbation methods. Each plot corresponds to a fixed $\mathrm{srank}(\mathbf{M})$. The reported $\bar\rho$ values indicate the mean subspace alignment for each method. Learning rate $\eta = 10^{-4}$ is used throughout.}
    \label{fig:sup_fig2}
\end{figure}

Figure~\ref{fig:sup_fig2} shows that all three subspace perturbation schemes lead to comparable convergence rates under random subspace selection. This confirms that in the absence of adaptivity (in subspace selection), the convergence is largely governed by the mean subspace alignment $\bar\rho$, rather than the structural form of the perturbation itself.

\subsection{On the Statistical Significance of LLM Fine-tuning}
\label{app:stat_sig}

As mentioned in Appendix~\ref{app:exp_detail_opt}, our main results in Table~\ref{tab:opt-13b} were reported using a single random seed (seed $0$) to ensure fair comparison with the primary baseline, LoZO~\citep{chen2024enhancing}, whose results were also obtained using the same seed. To provide a more comprehensive evaluation of performance variability and robustness, we conduct additional experiments using three random seeds: $\{0, 42, 100\}$.

We emphasize that performance considerations did not influence the choice of seeds but instead reflect common practice in the literature. Due to the high computational cost of fine-tuning OPT-13B, we restrict the evaluation to four representative datasets (SST-2, RTE, WIC, and SQuAD) covering both classification and generation tasks.

\begin{table*}[!h]
\centering
\caption{
Mean ($\pm$ standard deviation) performance of fine-tuning OPT-13B on 1000 examples, averaged over three seeds \{$0$, $42$, $100$\}. Tasks include classification (SST-2, RTE, WIC) and generation (SQuAD).}
\label{tab:opt-13b-add}
\resizebox{0.8\linewidth}{!}{
    \begin{tabular}{lcccc}
    \toprule
    Task      & \textbf{SST-2} & \textbf{RTE} & \textbf{WIC} & \textbf{SQuAD} \\
    Task type & \multicolumn{3}{c}{------------------ classification ------------------} & generation \\
    \midrule
    MeZO      & $91.14\scriptstyle\pm0.30$ & $64.58\scriptstyle\pm3.71$ & $58.93\scriptstyle\pm0.90$ & $82.18\scriptstyle\pm0.47$ \\
    LoZO      & $92.72\scriptstyle\pm0.89$ & $66.73\scriptstyle\pm3.28$ & $59.51\scriptstyle\pm2.53$ & $\mathbf{84.55}\scriptstyle\pm1.11$ \\
    MeZO-BCD  & $\mathbf{92.81}\scriptstyle\pm0.53$ & $\mathbf{69.56}\scriptstyle\pm0.45$ & $\mathbf{60.92}\scriptstyle\pm0.99$ & $84.31\scriptstyle\pm0.53$ \\
    \bottomrule
    \end{tabular}
    }
\end{table*}

Overall, the results in Table~\ref{tab:opt-13b-add} suggest that MeZO-BCD achieves consistently strong performance across tasks. In particular, it outperforms or matches LoZO on most benchmarks despite using a slightly smaller tuning grid (Table~\ref{tab:opt_hyper_13b}). Although LoZO performs competitively on SST-2 and SQuAD, its relatively high variance and occasional degradation (e.g., on RTE and WIC) highlight the importance of tuning stability. These findings reinforce our main message: \emph{MeZO-BCD delivers competitive fine-tuning performance while offering favorable robustness and lower tuning sensitivity.}

\begin{figure}[ht]
    \centering
    \subfigure[Mean loss curve of OPT-13B on SST-2]{\includegraphics[width=0.45\textwidth]{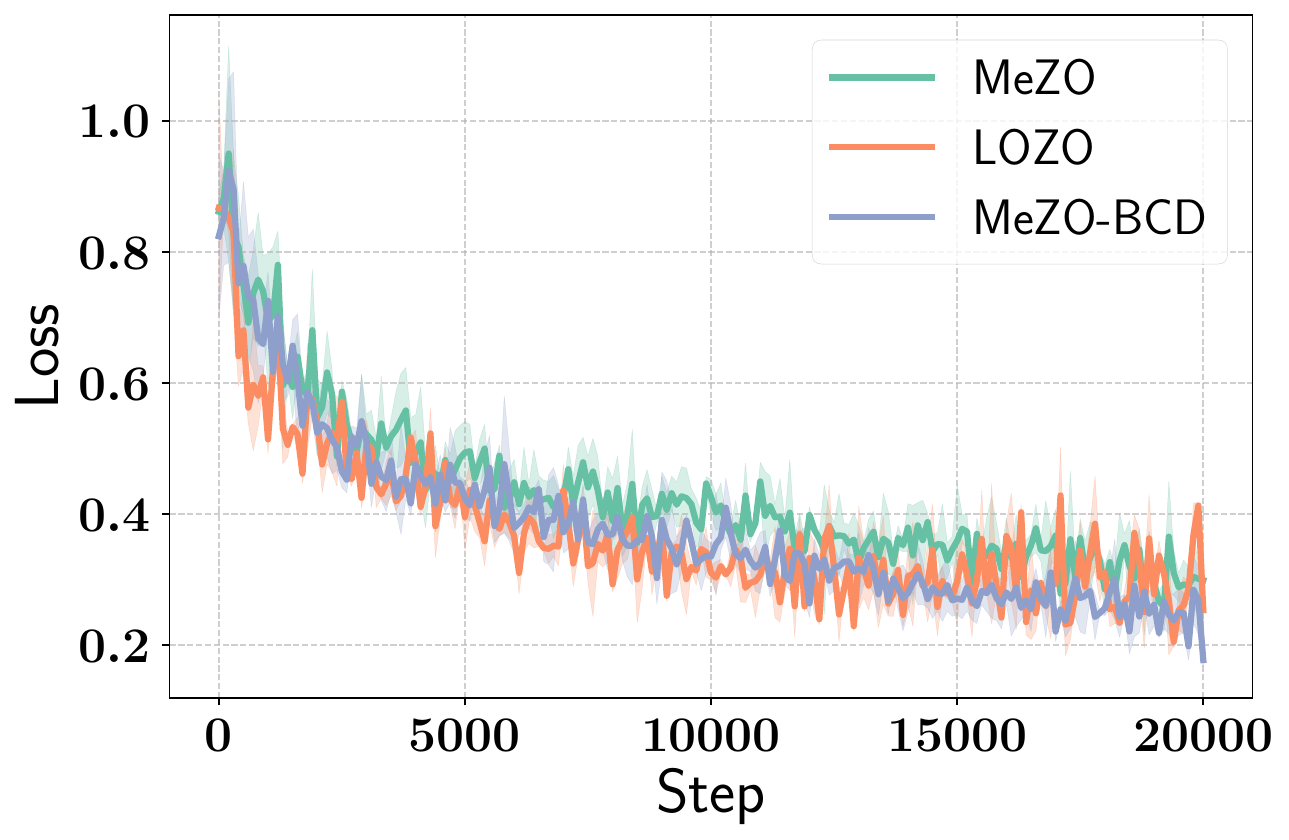}}
    \subfigure[Mean loss curve of OPT-13B on SQuAD]{\includegraphics[width=0.45\textwidth]{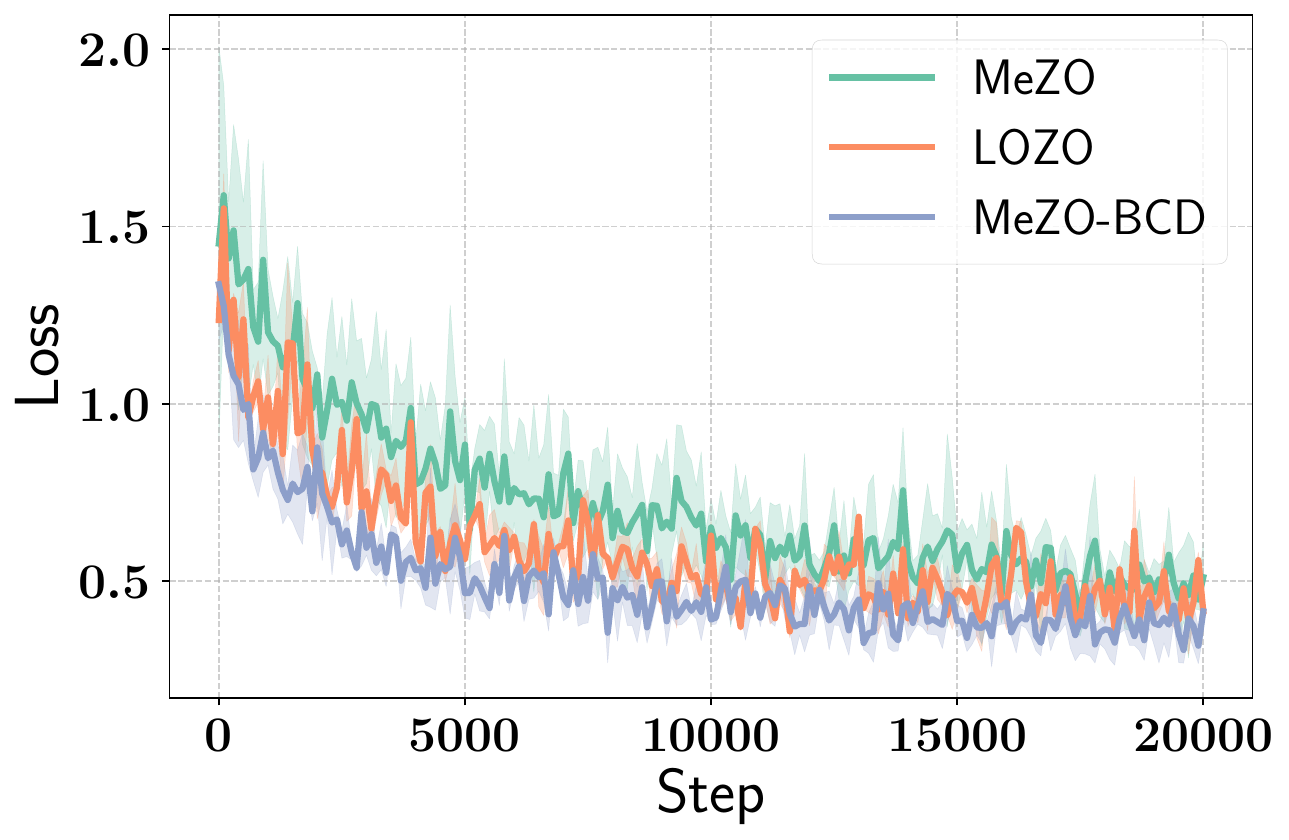}}
    \caption{Training loss curves averaged over three seeds \{$0$, $42$, $100$\} for OPT-13B on SST-2 and SQuAD. The shaded region indicates the standard deviation.}
    \label{fig:loss_opt_13b}
\end{figure}

To further support the claim that MeZO-BCD retains comparable iteration complexity to other ZO methods, we present additional training loss curves averaged over the three random seeds (Figure~\ref{fig:loss_opt_13b}). These results corroborate the claim in Section~\ref{subsec:llm_exp} that MeZO-BCD converges at a similar rate to other subspace methods, further confirming that it maintains comparable iteration complexity while reducing overall wall-clock time.

\clearpage
\section{Preliminary Exploration on Future Directions}\label{app:future_directions}

We supplement our main discussion with preliminary explorations of two promising extensions of MeZO-BCD: (1) adaptive block selection for dynamically prioritizing parameter updates, and (2) integration with stateful optimizers such as Adam. For each, we provide the motivation, methodological formulation, and initial empirical findings. Comprehensive evaluation and further algorithmic refinement remain as future work.

\subsection{MeZO-BCD with Adaptive Block Selection}

The convergence rate of zeroth-order optimization under subspace perturbation is governed by the subspace alignment $\rho_t$ (see Theorem~\ref{thm:improved_convergence_zosgd}). Our theoretical analysis (Propositions~\ref{prop:expected_rho},~\ref{prop:prob_rho}) shows that while various subspace schemes (e.g., low-rank, sparse, block sparse) share the same expected alignment, their probability distributions differ significantly. In particular, block sparse perturbation admits \emph{“good”} blocks with substantially higher alignment, suggesting that adaptively prioritizing such blocks can yield practical convergence gains beyond expectation-based analysis.

In practice, however, directly computing $\rho$ for each block is infeasible, especially in large-scale or black-box settings. As a practical alternative, we consider using the norm of the projected zeroth-order gradient for each block as a proxy for its alignment and importance. This quantity is both easy to compute and naturally highlights blocks that are either far from stationary or particularly sensitive to updates, making it a convenient and effective surrogate for guiding adaptive block selection.

\paragraph{Method.}
Let the model parameters be partitioned into $N$ disjoint blocks $\bm\theta = [\bm\theta^{(1)}, \dots, \bm\theta^{(N)}]$.  
At each iteration $t$, for each block $j \in \{1,\dots,N\}$, define the projected gradient norm
\[
    z_t^{(j)} := \frac{\mathcal{L}(\bm\theta_t + \mu \cdot \mathbf{e}_{B_j}(\mathbf{u}_t^{(j)}); \mathcal{B}_t) - \mathcal{L}(\bm\theta_t - \mu \cdot \mathbf{e}_{B_j}(\mathbf{u}_t^{(j)}); \mathcal{B}_t)}{2\mu}\in\mathbb{R},
\]
where $\mathbf{u}_t^{(j)} \sim \mathcal{N}(0, I_{|B_j|})$ and $\mathbf{e}_{B_j}$ embeds the block-local vector into the full parameter space. We use the $z_t^{(j)}$ as a block importance metric.

To avoid excessive computation, all $z_t^{(j)}$ are not evaluated at every iteration. Instead, during a warmup period $T_{\mathrm{warm}}$, blocks are updated in random cycle order to collect statistics. For each block, we maintain an exponential moving average (EMA) of the projected gradient norm:
\[
    \bar{z}_t^{(j)} = \alpha \cdot z_t^{(j)} + (1-\alpha)\bar{z}_{t-1}^{(j)},
\]
with $\bar{z}_0^{(j)} = 0$ and $\alpha \in (0,1]$. After warmup, we convert EMA values into a sampling probability via softmax:
\[
    p_t^{(j)} = \frac{\exp(\bar{z}_t^{(j)}/\tau)}{\sum_{k=1}^N \exp(\bar{z}_t^{(k)}/\tau)},
\]
where $\tau$ is a softmax temperature. At each iteration, block $j_t$ is sampled from this distribution and updated using the corresponding ZO projected gradient. Also, $\bar z_t^{(j)}$ is updated. Pseudocode for the full procedure is provided in Algorithm~\ref{alg:mezo_bcd_adaptive}.

\paragraph{Preliminary Experiments and Analysis.}
To assess the effectiveness of our adaptive block selection strategy, we conduct a preliminary fine-tuning experiment using OPT-1.3B on the SST-2 dataset. We use a learning rate of $\eta=5\times 10^{-6}$ and run experiments over three random seeds: $\{0, 42, 100\}$. We run each experiment for 20K steps, and evaluate the test accuracy for every 4K steps. The adaptive selection hyperparameters are set as follows: $\alpha=0.1$, $\tau=1$, and warmup period $T_{\text{warm}}=10N=270$ (since $N=27$ for OPT-1.3B).

\begin{figure}[!h]
    \centering
    \includegraphics[width=0.6\linewidth]{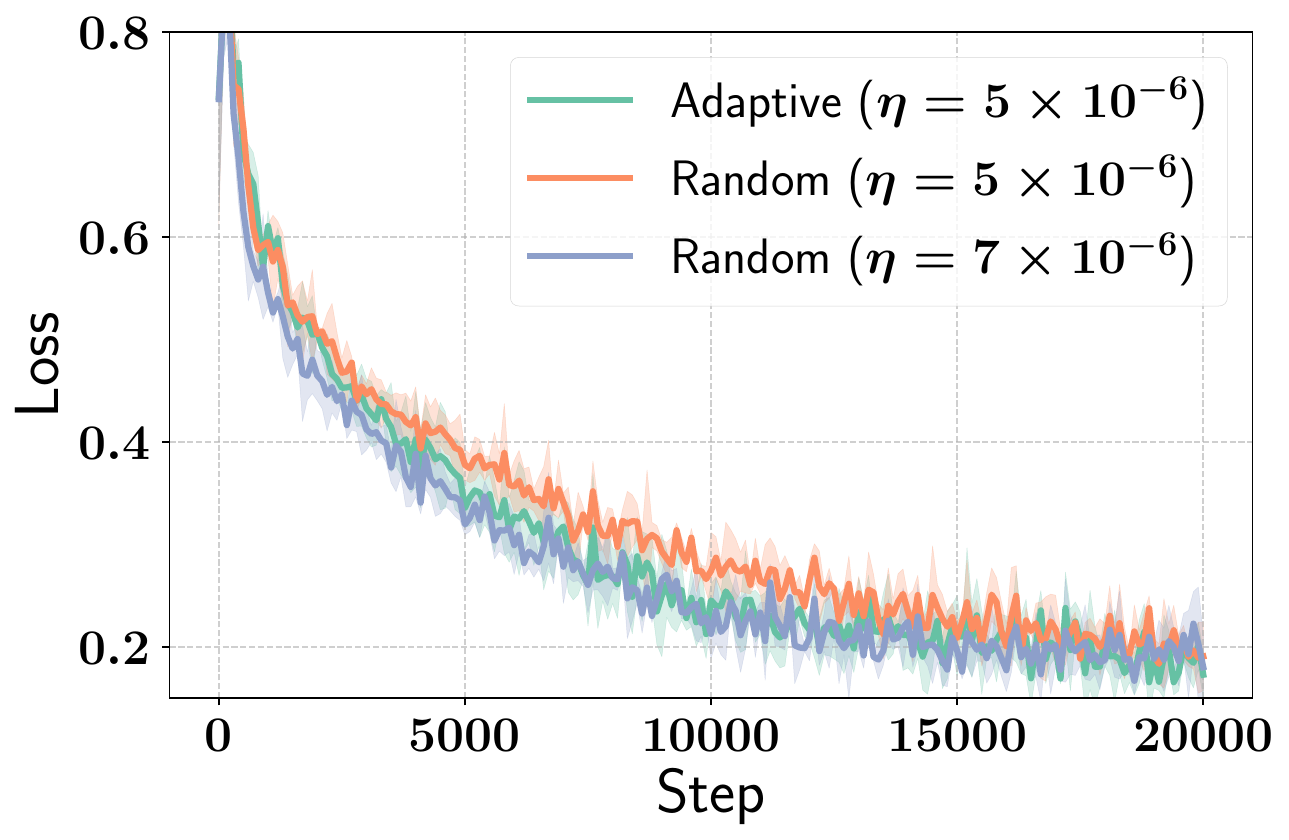}
    \caption{Loss curves for MeZO-BCD under adaptive vs random block selection strategies on OPT-1.3B, SST-2. Each curve represents the average $\pm$ one standard deviation over three seeds.}\label{fig:adaptive_loss_curve}
\end{figure}

Figure~\ref{fig:adaptive_loss_curve} presents the training loss curves of MeZO-BCD under three configurations: random (cyclic permutation) with $\eta=5\times10^{-6}$, random with $\eta=7\times10^{-6}$, and adaptive selection with $\eta=5\times10^{-6}$. The results show that, at the same learning rate, the adaptive strategy leads to faster convergence compared to random selection. However, this advantage is largely offset when a slightly higher learning rate is used for the random baseline. This suggests that the effect of current adaptive selection is not substantially different from that of simple learning rate tuning.

\renewcommand{\arraystretch}{1.15}
\begin{table}[!h]
    \centering
    \footnotesize
    \caption{Comparison of MeZO-BCD test accuracy (\%) under different block selection strategies and learning rates of OPT-1.3B fine-tuning on SST-2. Each value is mean $\pm$ std over 3 runs.}
    \label{tab:adaptive}
    \vspace{1em}
    \resizebox{0.7\linewidth}{!}{
    \begin{tabular}{lccc}
        \toprule
        \multirow{2}{*}{\textbf{Block Selection}} & \multicolumn{2}{c}{\textbf{Random}} & \textbf{Adaptive} \\
        \cmidrule(lr){2-3} \cmidrule(lr){4-4}
        & $\eta=5\times 10^{-6}$ & $\eta=7\times 10^{-6}$ & $\eta=5\times 10^{-6}$ \\
        \midrule
        MeZO-BCD & $91.32{\scriptstyle\pm1.10}$ & $91.28{\scriptstyle\pm0.83}$ & $91.17{\scriptstyle\pm0.41}$ \\
        \bottomrule
    \end{tabular}
    }
\end{table}
\renewcommand{\arraystretch}{1.0}

Table~\ref{tab:adaptive} summarizes the best test accuracies under each setting. All methods achieve statistically comparable performance, with no significant differences across selection strategies or learning rates. These observations imply that the current form of adaptive block selection, though conceptually appealing, yields limited improvement in either convergence speed or fine-tuning performance. Therefore, further investigation is necessary to realize its full potential.

\begin{figure}[!h]
    \centering
    \includegraphics[width=0.48\linewidth]{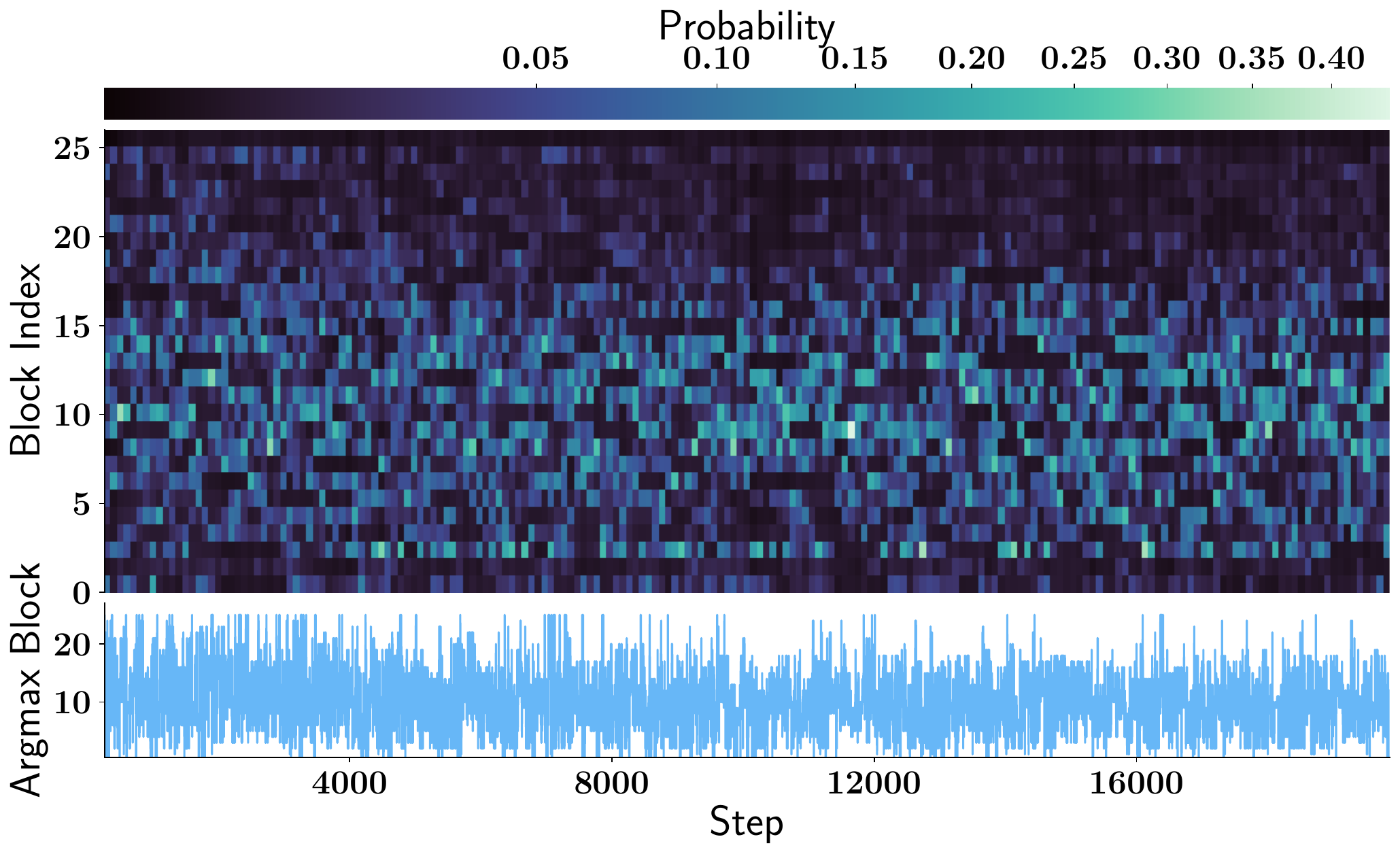}
    \includegraphics[width=0.48\linewidth]{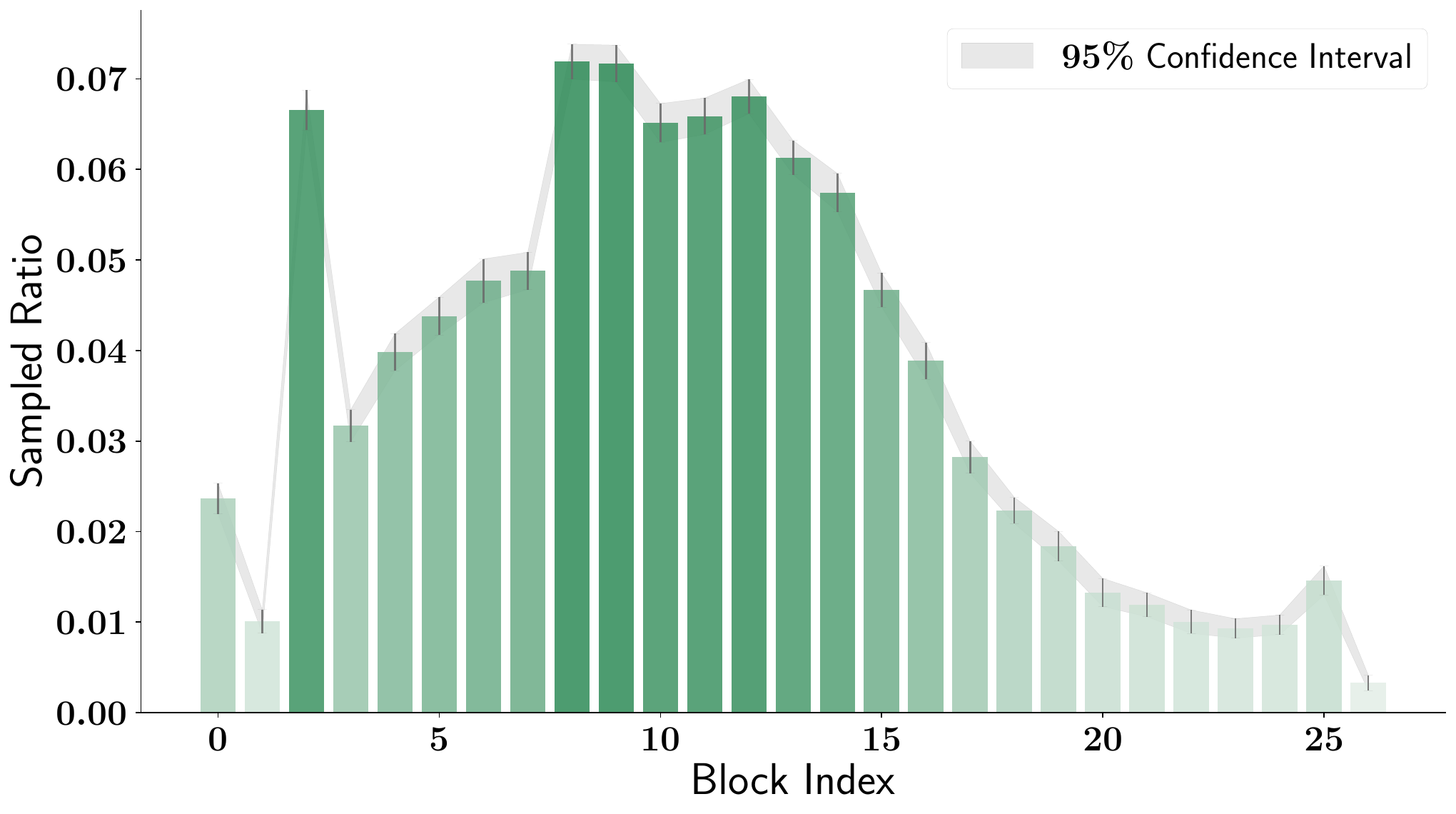}
    \caption{Statistics of adaptive block selection. \emph{Left:} The heatmap shows the $p_t^{(j)}$ for each block as training progresses. For heatmap visualization, the probability distribution is averaged every 100 steps. The line plot below the heatmap shows the block index with the highest selection probability (\emph{i.e.}, the $\mathrm{argmax}$) at each step. \emph{Right:} The sampled ratio of block indices over all training steps, simulated by drawing samples from the per-step probabilities at each step and aggregating across 1000 trials. The shaded region indicates the 95\% confidence interval.}
    \label{fig:adaptive_block_selection}
\end{figure}

To verify that the adaptive mechanism behaves as intended, we analyze its sampling behavior in more detail. Figure~\ref{fig:adaptive_block_selection} provides a diagnostic visualization. The left panel shows how the per-block sampling probabilities $p_t^{(j)}$ evolve over time, averaged every 100 steps. The line plot below tracks the block index with the highest selection probability at each step. We observe that the distribution varies throughout training, indicating that no single block or small set of blocks dominates selection persistently. This confirms that the selection mechanism, driven by the projected gradient norm and modulated through softmax, effectively balances exploration and exploitation by prioritizing informative blocks without collapsing to a single choice.

The right panel shows the overall sampling ratio of each block, obtained by sampling 1000 trajectories from the per-step distributions $\{p_t^{(j)}\}_{j=1}^N$. We find that most blocks are selected with reasonable frequency, although some (e.g., blocks 1 and 26) appear less frequently. These results validate that the mechanism functions as designed and achieves nontrivial adaptivity.

Taken together, this analysis demonstrates that our proposed adaptive selection mechanism operates correctly and provides a sound basis for further research. While its current form does not deliver substantial performance gains over random selection, we believe that it provides a solid foundation for exploring more sophisticated adaptive scheduling strategies.

\begin{algorithm}[!h]
\caption{MeZO-BCD with Adaptive Block Selection}\label{alg:mezo_bcd_adaptive} 
\centering
\begin{algorithmic}
    \REQUIRE block partitioning $\{B_i\}_{i=1}^{N}$, parameters $\mathbf{\theta}\in\mathbb{R}^d$ is partitioned into $[\theta_{B_1},\ldots,\theta_{B_N}]$, loss function $\mathcal{L}:\mathbb{R}^d\rightarrow\mathbb{R}$, step budget $T$, perturbation scale $\mu$, learning rate schedule $\{\eta_t\}_{t=1}^{T}$, warm-up budget $T_{\mathrm{warm}}=10N <T$, decay rate $\alpha=0.1$, temperature $\tau=1$ 
    \item[]
    \STATE \textbf{Initialize: } $\bar z_0^{(i)}\leftarrow 0$ for all $i=1,\ldots, N$
    \FOR{$t=1,\ldots,T$}
        \STATE Sample minibatch $\mathcal{B}\subset \mathcal{D}$ and random seed $s$\item[]
        \IF{$t\leq T_{\mathrm{warm}}$}
            \STATE Update current active block index $j\leftarrow\texttt{UpdateBlockIdx}(\text{ascending}, t, N)$ (Algorithm~\ref{alg:update_block_idx})
        \ELSE
            \STATE Compute probabilities $p_t^{(i)}= \frac{\exp(\bar{z}_t^{(i)}/\tau)}{\sum_{k=1}^N \exp(\bar{z}_t^{(k)}/\tau)}$ for all $i=1,\ldots, N$
            \STATE Update current active block index $j\sim\mathrm{Categorical}(p_t^{(1)},\ldots, p_t^{(N)})$
        \ENDIF
        \item[]
        \STATE $\mathbf{\theta}_{B_j}\leftarrow \mathrm{\texttt{PerturbParameters}}(\mathbf{\theta}_{B_j},\mu,s)$ (Algorithm~\ref{alg:perturb_params})
        \STATE $\ell_+\leftarrow\mathcal{L}(\mathbf{\theta};\mathcal{B})$
        \STATE $\mathbf{\theta}_{B_j}\leftarrow \mathrm{\texttt{PerturbParameters}}(\mathbf{\theta}_{B_j},-2\mu,s)$
        \STATE $\ell_-\leftarrow\mathcal{L}(\mathbf{\theta};\mathcal{B})$
        \STATE $\mathbf{\theta}_{B_j}\leftarrow \mathrm{\texttt{PerturbParameters}}(\mathbf{\theta}_{B_j},\mu,s)$
        \item[]
        \STATE\texttt{projected\_grad} $\leftarrow (\ell_+-\ell_-)/2\mu$
        \item[]
        \FOR{$i=1,\ldots,N$}
            \IF{$i=j$}
                \STATE $\bar z_t^{(j)}\leftarrow \alpha\cdot\texttt{projected\_grad}+(1-\alpha)\bar z_{t-1}^{(j)}$
            \ELSE
                \STATE $\bar z_t^{(j)}\leftarrow \bar z_{t-1}^{(j)}$
            \ENDIF
        \ENDFOR
        \item[]
        \STATE Reset random number generator with seed $s$
        \FOR{$\theta_i\in\mathbf{\theta}_{B_j}$}
            \STATE $u\sim\mathcal{N}(0,1)$
            \STATE $\theta_i\leftarrow\theta_i-\eta_t$ * \texttt{projected\_grad} * $u$
        \ENDFOR
    \ENDFOR
\end{algorithmic}
\end{algorithm}

\subsection{MeZO-BCD with Adam}

Stateful optimizers such as Adam~\citep{KingBa15} are widely used in first-order optimization, but their application in zeroth-order methods has been limited by high memory overhead and only marginal gains~\citep{malladi2023fine,zo-bench}. MeZO-BCD provides a practical alternative: since only a single block is perturbed and updated at each step, optimizer states can be stored and updated locally for the active block, eliminating the need for full-model state storage. This enables scalable use of stateful optimizers in zeroth-order settings.

To fully leverage Adam’s adaptivity, it is important that the optimizer state for each block is updated across multiple consecutive steps. If the active block were switched every step, the optimizer state would not accumulate meaningful historical information, and Adam would degenerate into vanilla SGD. Thus, we introduce an interval parameter $\nu$, ensuring that each selected block is optimized for $\nu$ consecutive steps before switching. This design allows the optimizer state to capture meaningful trends within each block, unlocking the potential benefits of stateful optimization in the MeZO-BCD framework. Note that this design closely resembles the BAdam~\citep{luo2024badam}.

\paragraph{Method.}
For the active block $j$, and at interval step $t'=1,\dots,\nu$, let $g_{t'}$ be the block-local ZO gradient at step $t'$. Adam’s state (momentum $m_{t'}$ and variance $v_{t'}$) is updated and used according to:
\[
    m_{t'} = \beta_1 m_{t'-1} + (1-\beta_1)g_{t'}, \quad v_{t'} = \beta_2 v_{t'-1} + (1-\beta_2)g_{t'}^2
\]
\[
    \widehat{m}_{t'} = m_{t'}/(1-\beta_1^{t'}), \quad \widehat{v}_{t'} = v_{t'}/(1-\beta_2^{t'})
\]
\[
    \theta^{(j)}_{t'+1} = \theta^{(j)}_{t'} - \eta \frac{\widehat{m}_{t'}}{\sqrt{\widehat{v}_{t'}}+\varepsilon}
\]
where $\beta_1,\beta_2$ are the Adam decay rates. After $\nu$ iterations, the optimizer state for block $j$ is discarded, and the next block is selected. Pseudocode for the full procedure is provided in Algorithm~\ref{alg:mezo_bcd_adam}.

\paragraph{Preliminary Experiments and Analysis.}

To empirically assess the effectiveness of MeZO-BCD when integrated with Adam, we conduct preliminary experiments. Experiments are carried out using OPT-1.3B across five representative tasks: SST-2, RTE, WIC, BoolQ, and SQuAD. Following the same protocol as in prior sections, models are trained for 20K steps with evaluation every 4K steps. The block selection strategy is fixed to cyclic random permutation, and we set the consecutive update interval to $\nu=50$ to allow the Adam state to accumulate meaningfully within each block.

We perform hyperparameter tuning over two learning rates for each method. For MeZO-BCD, we consider $\eta \in \{5 \times 10^{-6}, 1 \times 10^{-5}\}$, while for MeZO-BCD (+Adam), we explore $\eta \in \{5 \times 10^{-5}, 1 \times 10^{-4}\}$. We found that the Adam variant tends to require a higher learning rate for stable convergence. While not rigorously analyzed here, one possible reason is the normalization effect of Adam, which scales updates by the inverse square root of the accumulated variance. All experiments are repeated over three random seeds $\{0, 42, 100\}$.

\renewcommand{\arraystretch}{1.15}
\begin{table*}[!h]
\centering
\caption{
Mean ($\pm$ standard deviation) test accuracy of fine-tuning OPT-1.3B on 1000 examples across five tasks, averaged over three seeds \{0, 42, 100\}. MeZO-BCD is compared against its variant with Adam.}\label{tab:adam}
\resizebox{0.85\linewidth}{!}{
    \begin{tabular}{lccccc}
    \toprule
    Task      & \textbf{SST-2} & \textbf{RTE} & \textbf{WIC} & \textbf{BoolQ}& \textbf{SQuAD} \\
    Task type & \multicolumn{4}{c}{------------------------ classification ------------------------} & generation \\
    \midrule
    MeZO-BCD & $91.32\scriptstyle\pm1.10$ & $63.78\scriptstyle\pm2.73$ & $59.98\scriptstyle\pm0.24$ & $63.53\scriptstyle\pm0.86$ & $76.58\scriptstyle\pm0.58$\\
    MeZO-BCD (+Adam)  & $91.40\scriptstyle\pm0.69$ & $62.82\scriptstyle\pm4.00$ & $59.51\scriptstyle\pm0.33$ & $62.83\scriptstyle\pm1.07$ & $76.28\scriptstyle\pm0.92$\\
    \bottomrule
    \end{tabular}
    }
\end{table*}
\renewcommand{\arraystretch}{1.0}

\begin{figure}[ht]
    \centering
    \subfigure[Mean loss curve of OPT-1.3B on SST-2]{\includegraphics[width=0.45\textwidth]{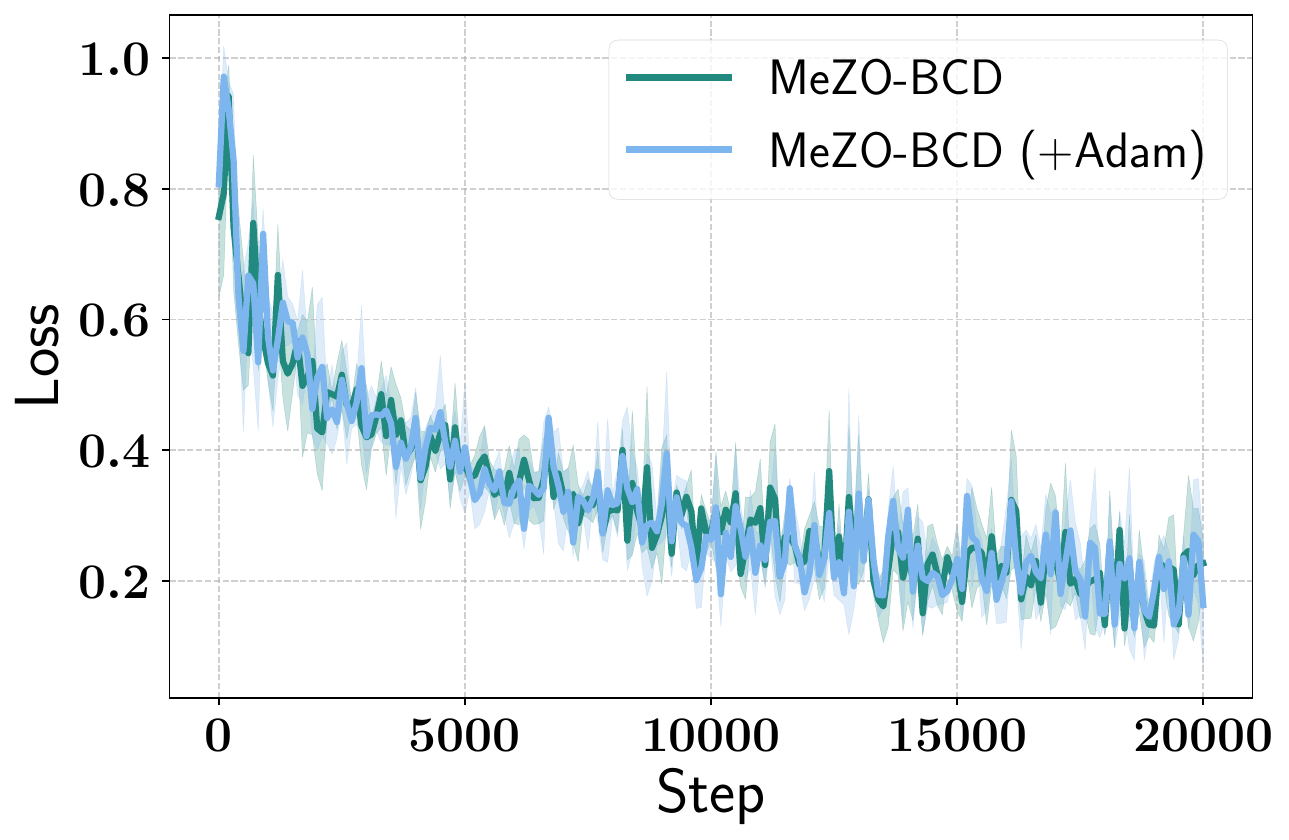}}
    \subfigure[Mean loss curve of OPT-1.3B on SQuAD]{\includegraphics[width=0.45\textwidth]{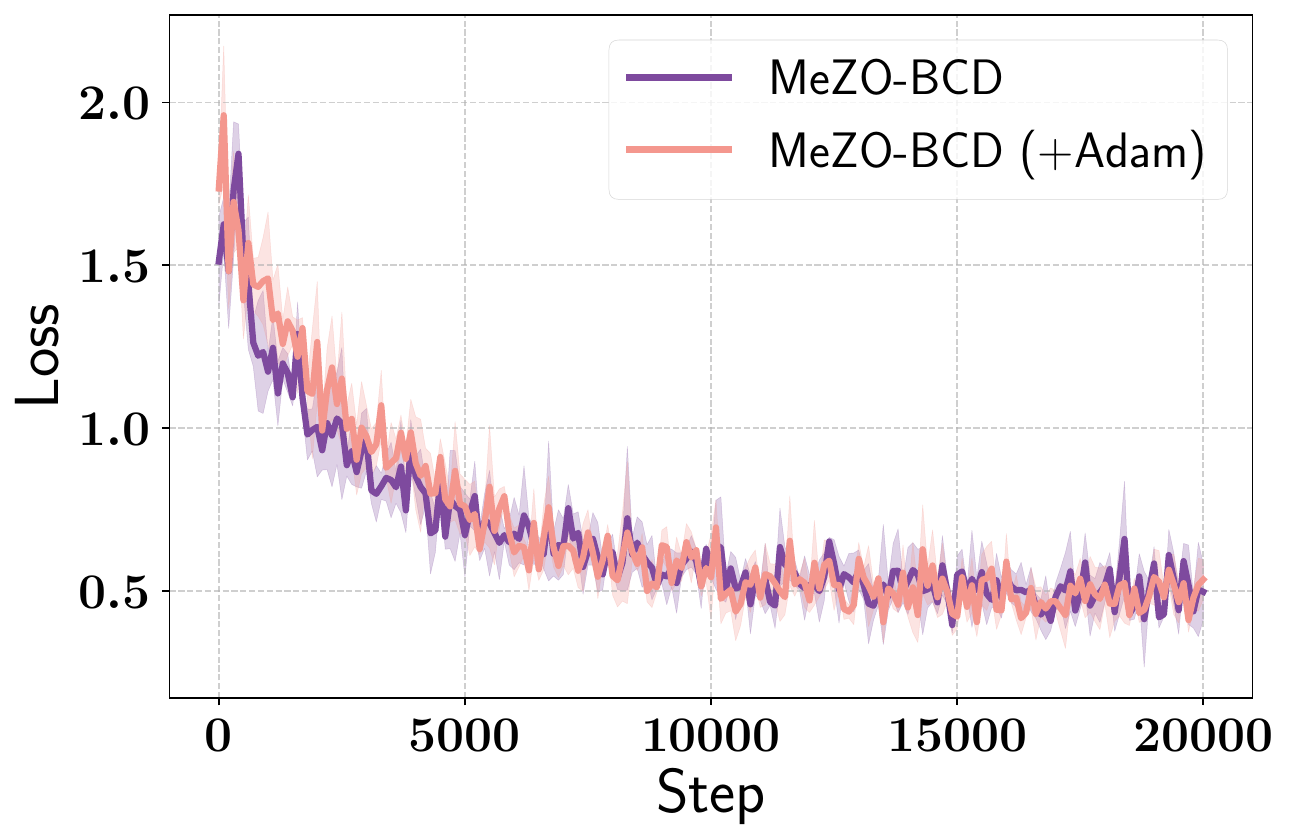}}
    \caption{Training loss curves for MeZO-BCD and MeZO-BCD (+Adam) on SST-2 and SQuAD. Each curve shows the average over three seeds \{0, 42, 100\}, with shaded regions denoting standard deviation. Both methods display similar convergence trajectories over 20K training steps.}
    \label{fig:loss_adam}
\end{figure}

Table~\ref{tab:adam} reports the best test accuracy across all tasks, and Figure~\ref{fig:loss_adam} shows the corresponding loss trajectories on SST-2 and SQuAD. In line with earlier studies~\citep{malladi2023fine,zo-bench}, we observe that integrating Adam does not yield meaningful improvements in convergence speed or fine-tuning accuracy. While this outcome is expected, a more meaningful insight emerges when we examine the memory requirements and iteration time shown in Figure~\ref{fig:adam_mem_time}.

\begin{figure}
    \centering
    \includegraphics[width=0.9\linewidth]{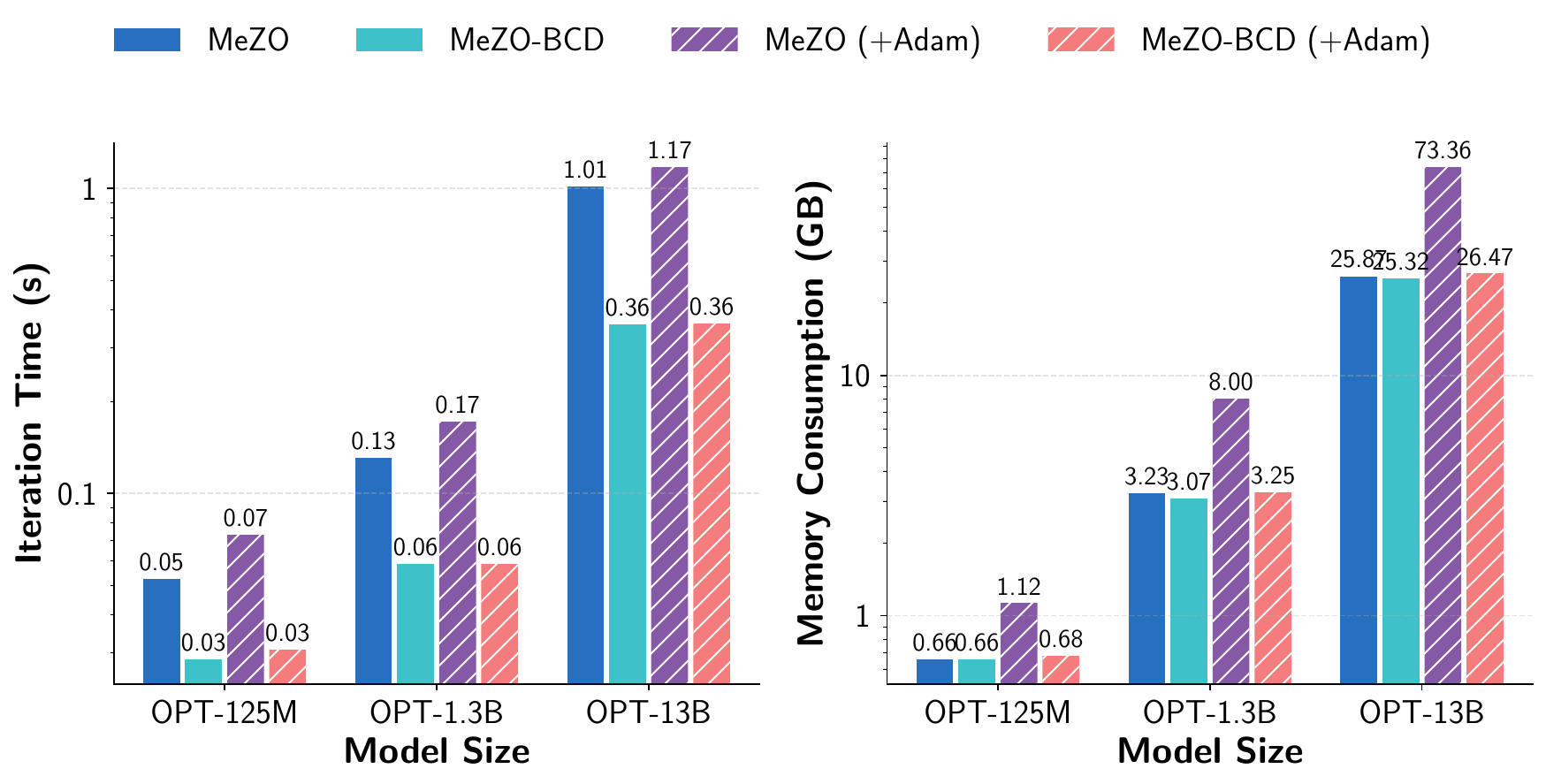}
    \caption{Comparison of per-step iteration time (\emph{left}) and peak memory consumption (\emph{right}) for MeZO and MeZO-BCD with and without Adam, across model sizes for fine-tuning on SST-2 (averaged over 100 training steps). MeZO-BCD (+Adam) achieves significantly lower memory usage compared to MeZO (+Adam), while maintaining comparable iteration time to MeZO-BCD. All results are measured using a single A100 80GB PCIe.}\label{fig:adam_mem_time}
\end{figure}

Most notably, MeZO-BCD (+Adam) achieves a drastic reduction in memory consumption compared to MeZO (+Adam). For instance, on OPT-13B, MeZO (+Adam) requires $\mathbf{73.4} $GB of memory due to full-model optimizer states, whereas MeZO-BCD (+Adam) only requires $\mathbf{26.5}$GB, almost identical to the non-Adam version of MeZO-BCD ($25.3$GB). This result is significant: it demonstrates that MeZO-BCD enables the use of stateful optimizers in large-scale zeroth-order settings where such methods were previously impractical due to prohibitive memory costs.

Equally important, this memory saving is achieved without sacrificing wall-clock efficiency. As shown in the left panel of Figure~\ref{fig:adam_mem_time}, MeZO-BCD (+Adam) maintains iteration time nearly identical to MeZO-BCD, while MeZO (+Adam) exhibits increased latency caused by additional memory transfer overhead involved by optimizer states. Thus, MeZO-BCD not only makes the use of Adam feasible from a memory standpoint but also preserves the low iteration cost critical to large-scale deployment.

These results collectively position MeZO-BCD as a scalable and practical foundation for extending zeroth-order optimization with adaptive and momentum-based techniques. Although Adam does not yield significant gains in our setting, the block-wise structure of MeZO-BCD makes stateful optimization both memory-efficient and computationally viable, thus opening the door to more sophisticated strategies that were previously infeasible at scale.

\begin{algorithm}[!h]
\caption{MeZO-BCD with Adam}\label{alg:mezo_bcd_adam} 
\centering
\begin{algorithmic}
    \REQUIRE block partitioning $\{B_i\}_{i=1}^{N}$, block ordering $\texttt{order}$, parameters $\mathbf{\theta}\in\mathbb{R}^d$ is partitioned into $[\mathbf{\theta}_{B_1},\ldots,\mathbf{\theta}_{B_N}]$, loss function $\mathcal{L}:\mathbb{R}^d\rightarrow\mathbb{R}$, step budget $T$, perturbation scale $\mu$, learning rate schedule $\{\eta_t\}_{t=1}^{T}$, block update interval $\nu$, $(\beta_1, \beta_2, \varepsilon)=(0.9, 0.999, 10^{-8})$
    \item[]

    \FOR{$t_b=1,\ldots,\lceil T/\nu\rceil$}
        \STATE Update current active block index $j\leftarrow\mathrm{\texttt{UpdateBlockIdx}}(\texttt{order}, t_b, N)$ (Algorithm~\ref{alg:update_block_idx})
        \STATE $\mathbf{m}_0\leftarrow \mathbf{0}$, $\mathbf{v}_0\leftarrow \mathbf{0}$
        \item[]
        \FOR{$t=1,\ldots,\nu$}
            \STATE Sample minibatch $\mathcal{B}\subset \mathcal{D}$ and random seed $s$
            \STATE $\mathbf{\theta}_{B_j}\leftarrow \mathrm{\texttt{PerturbParameters}}(\mathbf{\theta}_{B_j},\mu,s)$
            \STATE $\ell_+\leftarrow\mathcal{L}(\mathbf{\theta};\mathcal{B})$
            \STATE $\mathbf{\theta}_{B_j}\leftarrow \mathrm{\texttt{PerturbParameters}}(\mathbf{\theta}_{B_j},-2\mu,s)$
            \STATE $\ell_-\leftarrow\mathcal{L}(\mathbf{\theta};\mathcal{B})$
            \STATE $\mathbf{\theta}_{B_j}\leftarrow \mathrm{\texttt{PerturbParameters}}(\mathbf{\theta}_{B_j},\mu,s)$
            \item[]
            \STATE Reset random number generator with seed $s$
            \STATE $\mathbf{u}\sim\mathcal{N}(\mathbf{0},\mathbf{I}_{|B_j|})$
            \STATE $\widehat{\mathbf{g}}_t\leftarrow ((\ell_+-\ell_-)/2\mu)\mathbf{u}$
            \STATE $\mathbf{m}_t\leftarrow \beta_1\mathbf{m}_{t-1}+(1-\beta_1)\widehat{\mathbf{g}}_{t}$
            \STATE $\mathbf{v}_t\leftarrow \beta_2\mathbf{v}_{t-1}+(1-\beta_2)\widehat{\mathbf{g}}_{t}^{\odot 2}$
            \STATE $\widehat{\mathbf{m}}_t, \widehat{\mathbf{v}}_t\leftarrow \frac{\mathbf{m}_t}{1-\beta_1^t}, \frac{\mathbf{v}_t}{1-\beta_2^t}$
            \STATE $\mathbf{\theta}_{B_j}\leftarrow \mathbf{\theta}_{B_j}-\eta_{\nu (t_b-1)+t}\frac{\widehat{\mathbf{m}}_t}{\sqrt{\widehat{\mathbf{v}}}_t+\varepsilon}$
        \ENDFOR
    \ENDFOR

    \item[]
    \STATE \textbf{Subroutine:} \texttt{PerturbParameters}$(\mathbf{\theta}, \mu, s)$
    \STATE \hspace{1em} Reset random number generator with seed $s$
    \STATE \hspace{1em} $\mathbf{u} \sim \mathcal{N}(\mathbf{0}, \mathbf{I}_n)$
    \STATE \hspace{1em} $\mathbf{\theta} \leftarrow \mathbf{\theta} + \mu \mathbf{u}$
    \STATE \hspace{1em} \textbf{return} $\mathbf{\theta}$
    \item[]
\end{algorithmic}
\end{algorithm}

\end{document}